\pdfoutput=1
\documentclass{article}

\usepackage{jmlr2e}
\usepackage{booktabs}
\usepackage{multirow}
\usepackage{amsmath}
\usepackage{subcaption}
\usepackage{bm}

\newtheorem{prob}{Problem}[section]

\newtheorem{assum}{Assumption}[section]

\usepackage{algorithm}
\usepackage{algorithmic}

\usepackage{lastpage}
\jmlrheading{25}{2024}{1-\pageref{LastPage}}{7/22; Revised
9/24}{12/24}{22-0815}{Naoki Sato, Koshiro Izumi, and Hideaki Iiduka}
\ShortHeadings{Scaled Conjugate Gradient Method for Nonconvex Optimization in DNN}{Sato, Izumi, and Iiduka}

\firstpageno{1}

\begin{document}

\title{Scaled Conjugate Gradient Method for Nonconvex Optimization in Deep Neural Networks}

\author{\name Naoki Sato \email naoki310303@gmail.com \\
 \addr Computer Science Course, \\
 Graduate School of Science and Technology,\\
 Meiji University, Kanagawa 214-8571, Japan 
 \AND
\name Koshiro Izumi \email koshi-izumi@secom.co.jp \\
 \addr Visual Solutions Department 1, \\
 Technology Development Division,\\
 SECOM Co.,Ltd., Tokyo 181-8528, Japan
 \AND
 \name Hideaki Iiduka \email iiduka@cs.meiji.ac.jp \\
 \addr Department of Computer Science,\\
 Meiji University, Kanagawa 214-8571, Japan}

\editor{Prateek Jain}

\maketitle

\begin{abstract}
A scaled conjugate gradient method that accelerates existing adaptive methods utilizing stochastic gradients is proposed for solving nonconvex optimization problems with deep neural networks. It is shown theoretically that, whether with constant or diminishing learning rates, the proposed method can obtain a stationary point of the problem. Additionally, its rate of convergence with diminishing learning rates is verified to be superior to that of the conjugate gradient method. The proposed method is shown to minimize training loss functions faster than the existing adaptive methods in practical applications of image and text classification. Furthermore, in the training of generative adversarial networks, one version of the proposed method achieved the lowest Fr\'echet inception distance score among those of the adaptive methods.
\end{abstract}


\section{Introduction}
\label{sec:1}
Nonconvex optimization is needed for neural networks and learning systems. Various methods that use deep neural networks have been reported. In particular, stochastic gradient descent (SGD) \citep{robb1951} is a simple deep-learning optimizer that in theory can be applied to nonconvex optimization in deep neural networks \citep{feh2020,sca2020,chen2020,loizou2021}.

Momentum methods \citep{polyak1964,nes1983} and adaptive methods are powerful deep-learning optimizers for accelerating SGD. The adaptive methods include, for example, adaptive gradient (Adagrad) \citep{adagrad}, root mean square propagation (RMSprop) \citep{rmsprop}, adaptive moment estimation (Adam) \citep{adam}, adaptive mean square gradient (AMSGrad) \citep{reddi2018}, and Adam with decoupled weight decay (AdamW) \citep{loshchilov2018decoupled}. In \citep{chen2019}, generalized Adam (GAdam), which includes Adagrad, RMSprop, and AMSGrad, was proposed for nonconvex optimization. Adapting step sizes based on belief in observed gradients (AdaBelief) \citep{adab} is a recent deep-learning optimizer for nonconvex optimization. Convergence rates for SGD \citep{sca2020} and the adaptive methods \citep{chen2019,adab} reported in recent papers are summarized in Table \ref{table:1}. Note that ``LR'' in Table \ref{table:1} means learning rate.

\begin{table}[htbp]
\caption{Comparison of convergence rates for deep-learning optimizers used in nonconvex optimization}\label{table:1}
\centering
\begin{tabular}{l|cc}
\bottomrule
\multirow{2}{*}{Algorithm} & \multicolumn{2}{c}{Nonconvex optimization} \\
& Constant LR 
& Diminishing LR \\
\hline
SGD \citep{sca2020}
& $\mathcal{O}\left(\frac{1}{N} \right) + C$ 
& $\mathcal{O}\left(\frac{1}{\sqrt{N}} \right)$ \\
GAdam \citep{chen2019} 
& --------- 
& $\mathcal{O}\left(\frac{\log N}{\sqrt{N}} \right)$ \\
AdaBelief \citep{adab} 
& --------- 
& $\mathcal{O}\left(\frac{\log N}{\sqrt{N}} \right)$ \\
CG \citep{electronics9111809} 
& --------- 
& $\mathcal{O}\left(\sqrt{\frac{\log N}{N}} \right)$ \\
SCG (Proposed)
& $\mathcal{O}\left(\frac{1}{N} \right) + C_1 \alpha + C_2 \beta + C_3 \delta$ 
& $\mathcal{O}\left(\frac{1}{\sqrt{N}} \right)$ \\
\toprule 
\end{tabular}
\end{table}

The constant learning rate rule requires that $\alpha_n = \alpha > 0$, $\beta_n = \beta \geq 0$, and $\delta_n = \delta \geq 0$, whereas the diminishing learning rate rule requires that $\alpha_n=\mathcal{O}(1/\sqrt{n})$, $\beta_n = \beta^n$, and $\delta_n = \delta^n$. $C$, $C_1$, $C_2$, and $C_3$ are positive constants independent of the learning rate and number of iterations $N$. See Section \ref{sec:1_1} for an explanation of the mathematical notation. 
The proposed method (SCG) can cover all the algorithms in Table \ref{table:1} (see Section \ref{subsec:1.3} for details). Hence, it enables algorithms, such as GAdam, AdaBelief, and the CG method, with diminishing learning rates achieve an $\mathcal{O}(1/\sqrt{N})$ convergence rate (Theorem \ref{thm:2}), which is an improvement on the results in \citep{chen2019,adab,electronics9111809}. 

\subsection{Motivation}
\label{subsec:1.1}
\subsubsection{Scaled conjugate gradient (SCG) method}
\label{subsec:1.1.1}
The above-mentioned adaptive methods use the stochastic gradient of an observed loss function at each iteration. Meanwhile, conjugate gradient (CG) methods can be applied to large-scale nonconvex optimization (see Chapter 5.2 in \citep{noce}) by using the following conjugate gradient direction, which is defined from the current gradient $\mathsf{G}(\bm{x}_n, \xi_n)$ and the past direction $\bm{\mathsf{G}}_{n-1}$:
\begin{align}\label{cg}
\bm{\mathsf{G}}_n := \mathsf{G}(\bm{x}_n, \xi_n) - \delta_n \bm{\mathsf{G}}_{n-1},
\end{align}
where $\delta_n > 0$ is a parameter. A CG method making use of this \citep{electronics9111809} was presented for solving nonconvex optimization problems with deep neural networks. Table \ref{table:1} shows that CG has a better convergence rate than those of GAdam and AdaBelief. 

There are useful methods for general nonconvex optimization based on CG methods, such as {\em scaled CG (SCG) } and three-term CG \citep{MOLLER1993525,Zhang:2006vd,doi:10.1080/01630560701749524,nec2007,1547-5816_2013_3_595,naru2014}. The SCG method generates the following search direction using a {\em scaled} current gradient and the past direction:
\begin{align}\label{scg}
\bm{\mathsf{G}}_n := (1 + \gamma_n) \mathsf{G}(\bm{x}_n, \xi_n) - \delta_n \bm{\mathsf{G}}_{n-1},
\end{align}
where $\gamma_n \geq 0$ is a parameter. The SCG direction \eqref{scg} with $\gamma_n = 0$ coincides with the CG direction \eqref{cg}. Adding a scaled parameter $\gamma_n$ allows us to accelerate CG methods for nonconvex optimization \citep{1547-5816_2013_3_595,naru2014}. Accordingly, the first motivation of this paper is to show whether the SCG method for solving nonconvex optimization problems with deep neural networks, Algorithm \ref{algo:1} below, accelerates powerful deep-learning optimizers, such as RMSprop, Adagrad, Adam, and AdamW.

\begin{algorithm}
\caption{Scaled conjugate gradient method}\label{algo:1}
\begin{algorithmic}[1]
\REQUIRE{$(\alpha_n)_{n\in\mathbb{N}} \subset (0,1), (\beta_n)_{n\in\mathbb{N}} \subset [0,1), (\gamma_n)_{n\in\mathbb{N}} \subset [0,+\infty), (\delta_n)_{n\in\mathbb{N}} \subset [0,1/2], \zeta \in [0,1)$}
\STATE{$n \gets 0, \bm{x}_0,\bm{\mathsf{G}}_{-1}, \bm{m}_{-1} \in \mathbb{R}^d, \mathsf{H}_0 \in \mathbb{S}_{++}^d \cap \mathbb{D}^d$}
\LOOP
	\STATE{$\bm{\mathsf{G}}_n := \left(1 + \gamma_n \right) \mathsf{G}(\bm{x}_n, \xi_n) - \delta_n \bm{\mathsf{G}}_{n-1}$}
	\STATE{$\bm{m}_n := \beta_n \bm{m}_{n-1} + (1 - \beta_n) \bm{\mathsf{G}}_n$}
	\STATE{$\hat{\bm{m}}_n = (1-\zeta^{n+1})^{-1} \bm{m}_n$}
	\STATE{$\mathsf{H}_n \in \mathbb{S}_{++}^d \cap \mathbb{D}^d$ (see \eqref{adam} and \eqref{amsg} for examples of $\mathsf{H}_n$)}
	\STATE{Find $\bm{\mathsf{d}}_n \in \mathbb{R}^d$ that solves $\mathsf{H}_n \bm{\mathsf{d}} = -\hat{\bm{m}}_n$}
	\STATE{$\bm{x}_{n+1} := \bm{x}_n + \alpha_n \bm{\mathsf{d}}_n$}
	\STATE{$n \gets n + 1 $}
 \ENDLOOP
 \end{algorithmic}
\end{algorithm}

The existing SCG methods in \citep{MOLLER1993525,nec2007,1547-5816_2013_3_595,naru2014} use the full gradient $\nabla f$ of a loss function $f$ and, in particular, the SCG method in \citep{MOLLER1993525} can solve minimization problems in neural networks. The novelty of Algorithm \ref{algo:1} compared with the existing SCG methods is to use the stochastic gradient $\mathsf{G}(\bm{x}_n, \xi_n)$ to make Algorithm \ref{algo:1} implementable for deep neural networks. 

The second motivation related to the results in Table \ref{table:1} is to clarify whether the SCG method (Algorithm \ref{algo:1}) can in theory be applied to nonconvex optimization in deep neural networks. A particularly interesting concern is whether, under the diminishing learning rate rule, the SCG method has a better convergence rate than the CG method \citep{electronics9111809}. Moreover, we would like to clarify whether, under the constant learning rate rule, the SCG method can in theory approximate a desirable solution to a nonconvex optimization problem. 

The third motivation is to provide evidence that the SCG method performs better than the existing methods, such as RMSprop, Adagrad, Adam, AMSGrad, and AdamW, on image and text classification tasks. This is because the usefulness of the SCG method should be verified from the viewpoint of not only theory but also practice. 

\subsection{Contribution}
\label{subsec:1.2}
The present work makes theoretical and practical contributions. The theoretical contribution is to show that the SCG method, whether using constant or diminishing learning rates, can find a stationary point of an optimization problem (Theorems \ref{thm:1} and \ref{thm:2}). Using constant learning rates allows the SCG method to have approximately an $\mathcal{O}(1/N)$ convergence rate, where $N$ denotes the number of iterations (see Table \ref{table:1}). For diminishing learning rates, it is shown that the rate of convergence of the SCG method is $\mathcal{O}(1/\sqrt{N})$, which is better than the $\mathcal{O}(\sqrt{\log N/N})$ rate of convergence of the CG method \citep{electronics9111809} (see Table \ref{table:1}). 

The practical contribution is to provide numerical examples in which the SCG method performs better than the existing adaptive methods at image classification, text classification, and image generation (Section \ref{sec:4}). In particular, it is shown that, for training ResNet-18 on the CIFAR-100 and CIFAR-10 datasets, the SCG method with constant learning rates minimize the training loss functions faster than other methods. For the text classification task, the SCG method and RMSProp minimize the training loss functions faster than other methods. Since the SCG method uses the scaled conjugate gradient direction to accelerate the existing adaptive methods, it is considered that the method used in the experiments could optimize the training loss functions. In addition, among the adaptive methods, the SCG method can achieve lowest FID score in training several GANs.

\subsection{Related work} 
\label{subsec:1.3}
Algorithm \ref{algo:1} when $\gamma_n = 0$, i.e., the non-scaled CG method, coincides with the CG method \citep{electronics9111809} using $\bm{\mathsf{G}}_n := \mathsf{G}(\bm{x}_n, \xi_n) - \delta_n \bm{\mathsf{G}}_{n-1}$. The CG method with diminishing learning rates has $\mathcal{O}(\sqrt{\log N/N})$ convergence. Meanwhile, a convergence rate analysis of the CG method with constant learning rates has not yet been performed (see the ``CG" column of Table \ref{table:1}). However, SCG with constant learning rates has $\mathcal{O}(1/N) + C_1 \alpha + C_2 \beta + C_3 \delta$ convergence (see the ``SCG" column of Table \ref{table:1}). Since SCG is CG when $\gamma_n = 0$, the results of the present study also indicate that CG with constant learning rates has $\mathcal{O}(1/N) + C_1 \alpha + C_2 \beta + C_3 \delta$ convergence.

Algorithm \ref{algo:1} when $\gamma_n = 0$ and $\delta_n = 0$ is a unification \citep{iiduka2021} of the adaptive methods using $\bm{\mathsf{G}}_n := \mathsf{G}(\bm{x}_n, \xi_n)$, such as GAdam \citep{chen2019} and AdaBelief \citep{adab} (see \citep{iiduka2021} for examples of $\mathsf{H}_n$ in step 6 of Algorithm \ref{algo:1}). Our results show that the adaptive methods (i.e., SCG with $\gamma_n = 0$ and $\delta_n = 0$) using constant learning rates have $\mathcal{O}(1/N) + C_1 \alpha + C_2 \beta$ convergence, which is the same as in \citep{iiduka2021}. Moreover, the adaptive methods using diminishing learning rates have $\mathcal{O}(1/\sqrt{N})$ convergence, the same as in \citep{iiduka2021}. Therefore, our results are generalizations of the ones in \citep{iiduka2021}.

The remainder of the paper is as follows. Section \ref{sec:1_1} provides the mathematical preliminaries. Section \ref{sec:3} presents convergence analyses of Algorithm \ref{algo:1}. Section \ref{sec:4} provides numerical performance comparisons of the proposed method with the existing adaptive methods. Section \ref{sec:5} concludes the paper with a brief summary.

\section{Mathematical Preliminaries}
\label{sec:1_1}
\subsection{Notation and definitions}
$\mathbb{N}$ is the set of non-negative integers. For $n \in \mathbb{N} \backslash \{0\}$, define $[n]:=\{1,2,\ldots,n\}$. The $d$-dimensional Euclidean space $\mathbb{R}^d$ has an inner product $\langle \cdot, \cdot \rangle$ inducing the norm $\| \cdot \|$. $\mathbb{S}^d$ denotes the set of $d \times d$ symmetric matrices: $\mathbb{S}^d = \{ M \in \mathbb{R}^{d \times d} \colon M = M^\top \}$, where $^\top$ indicates the transpose operation. The set of $d \times d$ symmetric positive-definite matrices is denoted as $\mathbb{S}_{++}^d = \{ M \in \mathbb{S}^{d} \colon M \succ O \}$, and the set of $d \times d$ diagonal matrices is denoted as $\mathbb{D}^d = \{ M \in \mathbb{R}^{d \times d} \colon M = \mathsf{diag}(x_i), x_i \in \mathbb{R} \text{ } (i\in [d]) \}$. 

\subsection{Mathematical modeling in deep learning and assumptions}
Given a parameter $\bm{x} \in \mathbb{R}^d$ and given a data point $z$ in a data domain $Z$, a machine learning model provides a prediction whose quality is measured by a differentiable nonconvex loss function $f(\bm{x};z)$. We aim to minimize the expected loss defined for all $\bm{x} \in \mathbb{R}^d$ by
\begin{align}\label{expected}
f(\bm{x}) = \mathbb{E}_{z \sim \mathcal{D}} 
[f(\bm{x};z) ]
= \mathbb{E}[ f_{\xi} (\bm{x}) ],
\end{align}
where $\mathcal{D}$ is a probability distribution over $Z$, $\xi$ denotes a random variable with distribution function $P$, and $\mathbb{E}[\cdot]$ denotes the expectation taken with respect to $\xi$. A particularly interesting example of \eqref{expected} is the empirical average loss defined for all $\bm{x} \in \mathbb{R}^d$ by 
\begin{align}\label{empirical}
f(\bm{x}; S) = \frac{1}{T} \sum_{i\in [T]} f(\bm{x};z_i)
= \frac{1}{T} \sum_{i\in [T]} f_i(\bm{x}),
\end{align}
where $S = (z_1, z_2, \ldots, z_T)$ denotes the training set and $f_i (\cdot) := f(\cdot;z_i)$ denotes the loss function corresponding to the $i$-th training data $z_i$. 

This paper considers optimization problems under the following assumptions. 

\begin{assum}\label{assum:0}
\text{ } 

\begin{enumerate} 
\item[{\em (A1)}] $f_i \colon \mathbb{R}^d \to \mathbb{R}$ ($i \in [T]$) is continuously differentiable and $f \colon \mathbb{R}^d \to \mathbb{R}$ is defined for all $\bm{x}\in \mathbb{R}^d$ by 
\[
f(\bm{x}) := 
\frac{1}{T} \sum_{i=1}^T f_i (\bm{x}), 
\]
where $T$ denotes the number of samples.
\item[{\em (A2)}] For each iteration $n$, the optimizers estimate the full gradient vector $\nabla f(\bm{x}_n)$ as the stochastic gradient vector $\mathsf{G}(\bm{x}_n, \xi_n)$ such that, for all $\bm{x} \in \mathbb{R}^d$, $\mathbb{E}[\mathsf{G}(\bm{x}, \xi_n)] = \nabla f(\bm{x})$.
\item[{\em (A3)}] There exists a positive number $M$ such that, for all $\bm{x}\in \mathbb{R}^d$, $\mathbb{E}[\|\mathsf{G}(\bm{x}, \xi_n))\|^2] \leq M^2$.
\end{enumerate}
\end{assum}

Assumption \ref{assum:0} (A1) is a standard one for nonconvex optimization in deep neural networks (see, e.g., \citep[(2)]{chen2019}). Assumption \ref{assum:0} (A2) is needed for the optimizers to work (see, e.g., \citep[Section 2]{chen2019}), and Assumption \ref{assum:0} (A3) is used in the analysis of optimizers (see, e.g., \citep[A2]{chen2019}).

\subsection{Nonconvex optimization problem in deep neural networks}
\label{sec:2}
This paper deals with the following stationary point problem \citep{chen2019,adab,iiduka2021} for nonconvex optimization to minimize $f$ defined in (A1).

\begin{prob}\label{prob:1}
Under Assumption \ref{assum:0}, we would like to find a stationary point $\bm{x}^\star$ of a nonconvex optimization problem that minimizes $f$ over $\mathbb{R}^d$, i.e.,
\begin{align*}
\bm{x}^\star \in X^\star := 
\left\{ \bm{x}^\star \in \mathbb{R}^d \colon 
\nabla f (\bm{x}^\star) = \bm{0}
\right\}
=
\left\{ \bm{x}^\star \in \mathbb{R}^d \colon 
\langle \bm{x}^\star - \bm{x}, \nabla f (\bm{x}^\star) \rangle \leq 0 \text{ } \left(\bm{x} \in \mathbb{R}^d \right)
\right\}.
\end{align*}
\end{prob}
The performance measure of the optimizers for Problem \ref{prob:1} is as follows:
\begin{align}\label{pm}
\mathbb{E}\left[\langle \bm{x}_n - \bm{x}, \nabla f (\bm{x}_n) \rangle \right],
\end{align}
where $\bm{x} \in \mathbb{R}^d$ and $(\bm{x}_n)_{n\in\mathbb{N}}$ is the sequence generated by the optimizer.

Let us discuss the relationship between the performance measure \eqref{pm} and the squared $\ell_2$ norm of the gradient $\nabla f$, which is the traditional performance measure. Assume that $(\bm{x}_n)_{n\in\mathbb{N}}$ is bounded (see (A6)). First, we consider the case where there exists $n_0 \in \mathbb{N}$ such that, for all $n\in \mathbb{N}$ and all $\bm{x}\in \mathbb{R}^d$, $n \geq n_0$ implies that $\mathbb{E}[\langle \bm{x}_n - \bm{x}, \nabla f (\bm{x}_n) \rangle ] < 0$. The boundedness condition of $(\bm{x}_n)_{n\in\mathbb{N}}$ ensures that there exists a subsequence $(\bm{x}_{n_i})_{i\in\mathbb{N}}$ of $(\bm{x}_n)_{n\in\mathbb{N}}$ such that $(\bm{x}_{n_i})_{i\in\mathbb{N}}$ converges to a point $\bar{\bm{x}} \in \mathbb{R}^d$. Hence, the continuity of $\nabla f$ (see (A1)) guarantees that, for all $\bm{x} \in \mathbb{R}^d$,
\begin{align*}
\mathbb{E}\left[\langle \bar{\bm{x}} - \bm{x}, \nabla f (\bar{\bm{x}}) \rangle \right]
\leq 0,
\end{align*}
which, together with $\bm{x} := \bar{\bm{x}} - \nabla f (\bar{\bm{x}})$, implies that 
\begin{align*}
\mathbb{E}\left[\| \nabla f (\bar{\bm{x}}) \|^2 \right]
= 0.
\end{align*}
Accordingly, the convergent point $\bar{\bm{x}}$ of $(\bm{x}_{n_i})_{i\in\mathbb{N}}$ is a stationary point of $f$. Next, we consider the case where, for all $n_0 \in \mathbb{N}$, there exist $n\in \mathbb{N}$ and $\bm{x}\in \mathbb{R}^d$ such that $n \geq n_0$ and $\mathbb{E}[\langle \bm{x}_n - \bm{x}, \nabla f (\bm{x}_n) \rangle ] \geq 0$. Then, there exists a subsequence $(\bm{x}_{n_j})_{j\in\mathbb{N}}$ of $(\bm{x}_n)_{n\in\mathbb{N}}$ such that, for all $j\in \mathbb{N}$, $0 \leq \mathbb{E}[\langle \bm{x}_{n_j} - \bm{x}, \nabla f (\bm{x}_{n_j}) \rangle ]$. The boundedness condition of $(\bm{x}_n)_{n\in\mathbb{N}}$ ensures that there exists a subsequence $(\bm{x}_{n_{j_k}})_{k\in\mathbb{N}}$ of $(\bm{x}_{n_j})_{j\in\mathbb{N}}$ such that $(\bm{x}_{n_{j_k}})_{k\in\mathbb{N}}$ converges to a point $\hat{\bm{x}} \in \mathbb{R}^d$. Hence, the continuity of $\nabla f$ (see (A1)) guarantees that
\begin{align*}
0 \leq 
\lim_{k \to + \infty} \mathbb{E}\left[\langle \bm{x}_{n_{j_k}} - \bm{x}, \nabla f (\bm{x}_{n_{j_k}}) \rangle \right]
\leq
\mathbb{E}\left[\langle \hat{\bm{x}} - \bm{x}, \nabla f (\hat{\bm{x}}) \rangle \right] 
\leq \epsilon,
\end{align*}
where $\epsilon > 0$ is the precision. The definition of the inner product implies that, if $\epsilon$ is small enough, then $\|\nabla f (\hat{\bm{x}})\|$ will be small.

Let $\bm{x} := \bm{x}^\star$ be a stationary point of $f$ and consider the following performance measure: 
\begin{align}\label{pm_1}
\mathbb{E}\left[\langle \bm{x}_n - \bm{x}^\star, \nabla f (\bm{x}_n) \rangle \right],
\end{align}
which is a restricted measure of \eqref{pm}. When the optimizer can find $\bm{x}^\star$, the inner product of $\bm{x}_n - \bm{x}^\star$ and $\nabla f (\bm{x}_n)$ is positive. If the upper bound of \eqref{pm_1} is small enough, then the definition of the inner product implies that the optimizer has fast convergence. Hence, the performance measure \eqref{pm} including \eqref{pm_1} will be adequate for investigating the performances of deep-learning optimizers. 

Let us consider the case where $f$ is convex \citep{adam,reddi2018}. In this case, Problem \ref{prob:1} involves doing the following:
\begin{align}\label{convex}
\text{Find a point } \bm{x}^\star \text{ such that } f(\bm{x}^\star) \leq f(\bm{x}) \text{ for all } \bm{x} \in \mathbb{R}^d.
\end{align}
The performance measure of Problem \eqref{convex} is the regret, defined as 
\begin{align*}
R(T) := \sum_{i=1}^T f_i (\bm{x}_i) - T f^\star,
\end{align*}
where $f^\star$ denotes the optimal value of Problem \eqref{convex} and $(\bm{x}_i)_{i=1}^T$ is the sequence generated by the optimizer. Let $\bm{x}^\star$ be a solution of Problem \eqref{convex}. When the upper bound of the performance measure \eqref{pm} is $\epsilon > 0$, the convexity of $f$ ensures that 
\begin{align*}
\mathbb{E}\left[ f(\bm{x}_n) - f^\star \right]
\leq
\mathbb{E}\left[\langle \bm{x}_n - \bm{x}^\star, \nabla f (\bm{x}_n) \rangle \right]
\leq 
\epsilon.
\end{align*}
Accordingly, the performance measure \eqref{pm} leads to the one $\mathbb{E}[ f(\bm{x}_n) - f^\star ]$ used for convex optimization.

\section{Convergence Analyses of Algorithm \ref{algo:1}}
\label{sec:3}
\subsection{Outline of our results}
Let us outline our results (see also Table \ref{table:1}). First, we show that the SCG method (Algorithm \ref{algo:1}) with constant learning rates ensures that there exist positive constants $C_i$ ($i=1,2,3$) such that 
\begin{align*}
\frac{1}{n} \sum_{k=1}^n \mathbb{E}\left[\langle \bm{x}_k - \bm{x}, \nabla f (\bm{x}_k) \rangle \right]
\leq 
\mathcal{O}\left(\frac{1}{n} \right)
+
C_1 \alpha
+
C_2 \beta
+
C_3 \delta
\end{align*}
(see Theorem \ref{thm:1} and Table \ref{table:1}). This implies that, if $f$ is convex, we have 
\begin{align*}
\frac{R(T)}{T}
\leq 
\mathcal{O}\left(\frac{1}{T} \right)
+
C_1 \alpha
+
C_2 \beta
+
C_3 \delta
\end{align*}
(see Proposition \ref{prop:1}). Next, we show that the SCG method (Algorithm \ref{algo:1}) with diminishing learning rates, such as $\alpha_n = \mathcal{O}(1/\sqrt{n})$, gives 
\begin{align*}
\frac{1}{n} \sum_{k=1}^n \mathbb{E}\left[\langle \bm{x}_k - \bm{x}, \nabla f (\bm{x}_k) \rangle \right]
\leq
\mathcal{O}\left(\frac{1}{\sqrt{n}} \right)
\end{align*}
(see Theorem \ref{thm:2} and Table \ref{table:1}). This implies that, if $f$ is convex, 
\begin{align*}
\frac{R(T)}{T}
\leq 
\mathcal{O}\left(\frac{1}{\sqrt{T}} \right)
\end{align*} 
(see Proposition \ref{prop:2}).

\subsection{Our detailed results}
In order to analyze Algorithm \ref{algo:1}, we will assume the following.

\begin{assum}\label{assum:1}
For the sequence $(\mathsf{H}_n)_{n\in\mathbb{N}} \subset \mathbb{S}_{++}^d \cap \mathbb{D}^d$ in Algorithm \ref{algo:1} defined by $\mathsf{H}_n := \mathsf{diag}(h_{n,i})$, the following conditions are satisfied:
\begin{enumerate}
\item[{\em (A4)}] $h_{n+1,i} \geq h_{n,i}$ for all $n\in\mathbb{N}$ and all $i\in [d]$;
\item[{\em (A5)}] For all $i\in [d]$, a positive number $B_i$ exists such that $\sup \{ {\mathbb{E}}[h_{n,i}] \colon n \in \mathbb{N} \} \leq B_i$.
\end{enumerate}
Moreover, 
\begin{enumerate}
\item[{\em (A6)}] $D := \max_{i\in [d]} \sup \{ (x_{n,i} - x_i)^2 \colon n \in \mathbb{N} \} < + \infty$ for $\bm{x} = (x_i) \in \mathbb{R}^d$.
\end{enumerate}
\end{assum}

Section \ref{sec:4} provides examples of $\mathsf{H}_n$ satisfying Assumption \ref{assum:1}\ (A4) and (A5) (see also \citep{iiduka2021}). Assumption \ref{assum:1}\ (A6) was used in the analysis of the existing adaptive methods (see \citep{adam,reddi2018,iiduka2021}).

The following is a convergence analysis of Algorithm \ref{algo:1} with constant learning rates. The proof of Theorem \ref{thm:1} is given in Appendix \ref{appen:1}.

\begin{theorem}\label{thm:1}
Suppose that Assumptions \ref{assum:0} and \ref{assum:1} hold and let $\alpha_n := \alpha$, $\beta_n := \beta$, $\gamma_n := \gamma$, and $\delta_n := \delta$. Then, Algorithm \ref{algo:1} is such that, for all $\bm{x} \in \mathbb{R}^d$,
\begin{align*}
\liminf_{n\to +\infty} 
\mathbb{E}\left[ \left\langle \bm{x}_n - \bm{x}, \nabla f (\bm{x}_n) \right\rangle \right]
\leq 
\frac{\tilde{B}^2 \tilde{M}^2}{2\tilde{b} \tilde{\gamma}\tilde{\zeta}^2} \alpha 
+
\frac{\tilde{M}\sqrt{Dd}}{\tilde{b}\tilde{\gamma}\tilde{\zeta}} \beta 
+
\frac{4 \hat{M} \sqrt{Dd}}{\tilde{\gamma}\tilde{\zeta}} \delta,
\end{align*}
where $\tilde{b} := 1 - \beta$, $\tilde{\gamma} := 1 + \gamma$, $\tilde{\zeta} := 1 - \zeta$, $\tilde{M}^2 := \max \{ \|\bm{m}_{-1}\|^2, M^2 \}$, $\tilde{B} := \sup\{ {\max_{i\in [d]} h_{n,i}^{-1/2}} \colon n\in\mathbb{N}\} < + \infty$, and $\hat{M} := \max \{M^2, \|\bm{\mathsf{G}}_{-1} \|^2 \}$. Moreover, for all $\bm{x} \in \mathbb{R}^d$ and all $n\in\mathbb{N}$,
\begin{align*}
\frac{1}{n} \sum_{k=1}^n \mathbb{E}\left[ \left\langle \bm{x}_k - \bm{x}, \nabla f (\bm{x}_k) \right\rangle \right]
\leq 
\frac{D \sum_{i=1}^d 
B_{i}}{2 \tilde{b} \alpha n}
+
\frac{\tilde{B}^2 \tilde{M}^2}{2 \tilde{b}\tilde{\zeta}^2} \alpha
+
\frac{\tilde{M}\sqrt{Dd}}{\tilde{b}} \beta
+
\frac{4 \hat{M}\sqrt{Dd}}{\tilde{\gamma}} \delta.
\end{align*} 
\end{theorem}

Theorem \ref{thm:1} implies the following proposition.

\begin{proposition}\label{prop:1}
Under the assumptions in Theorem \ref{thm:1} and the convexity of $f_i$ ($i\in [T]$), we have 
\begin{align*}
\liminf_{n\to +\infty} 
\mathbb{E}\left[ f (\bm{x}_n) - f^\star \right]
\leq 
\frac{\tilde{B}^2 \tilde{M}^2}{2\tilde{b} \tilde{\gamma}\tilde{\zeta}^2} \alpha 
+
\frac{\tilde{M}\sqrt{Dd}}{\tilde{b}\tilde{\gamma}\tilde{\zeta}} \beta 
+
\frac{4 \hat{M} \sqrt{Dd}}{\tilde{\gamma}\tilde{\zeta}} \delta,
\end{align*}
where the positive constants are defined as in Theorem \ref{thm:1}. Moreover, for all $n\in\mathbb{N}$,
\begin{align*}
&\min_{k\in [n]} \mathbb{E}\left[ f (\bm{x}_k) - f^\star \right]
\leq 
\frac{D \sum_{i=1}^d 
B_{i}}{2 \tilde{b} \alpha n}
+
\frac{\tilde{B}^2 \tilde{M}^2}{2 \tilde{b}\tilde{\zeta}^2} \alpha
+
\frac{\tilde{M}\sqrt{Dd}}{\tilde{b}} \beta
+
\frac{4 \hat{M}\sqrt{Dd}}{\tilde{\gamma}} \delta,\\
&\frac{R(T)}{T}
\leq 
\frac{D \sum_{i=1}^d 
B_{i}}{2 \tilde{b} \alpha T}
+
\frac{\tilde{B}^2 \tilde{M}^2}{2 \tilde{b}\tilde{\zeta}^2} \alpha
+
\frac{\tilde{M}\sqrt{Dd}}{\tilde{b}} \beta
+
\frac{4 \hat{M}\sqrt{Dd}}{\tilde{\gamma}} \delta.
\end{align*} 
Additionally, if we define $\tilde{\bm{x}}_n$ by $\tilde{\bm{x}}_n := (1/n) \sum_{k=1}^n \bm{x}_k$, then
\begin{align*}
\mathbb{E}\left[ f (\tilde{\bm{x}}_n) - f^\star \right]
\leq 
\frac{D \sum_{i=1}^d 
B_{i}}{2 \tilde{b} \alpha n}
+
\frac{\tilde{B}^2 \tilde{M}^2}{2 \tilde{b}\tilde{\zeta}^2} \alpha
+
\frac{\tilde{M}\sqrt{Dd}}{\tilde{b}} \beta
+
\frac{4 \hat{M}\sqrt{Dd}}{\tilde{\gamma}} \delta.
\end{align*}
\end{proposition}

The following is a convergence analysis of Algorithm \ref{algo:1} with diminishing learning rates. The proof of Theorem \ref{thm:2} is given in Appendix \ref{appen:1}.

\begin{theorem}\label{thm:2}
Suppose that Assumptions \ref{assum:0} and \ref{assum:1} hold and let $(\alpha_n)_{n\in\mathbb{N}}$, $(\beta_n)_{n\in\mathbb{N}}$, $(\gamma_n)_{n\in\mathbb{N}}$, and $(\delta_n)_{n\in\mathbb{N}}$ satisfy 

{\em (C1)} $\sum_{n=0}^{+\infty} \alpha_n = + \infty$, $\sum_{n=0}^{+\infty} \alpha_n^2 < + \infty$, $\sum_{n=0}^{+\infty} \alpha_n \beta_n < + \infty$, $\sum_{n=0}^{+\infty} \alpha_n \gamma_n < + \infty$, and $\sum_{n=0}^{+\infty} \alpha_n \delta_n < + \infty$. 

Then, Algorithm \ref{algo:1} satisfies, for all $\bm{x} \in \mathbb{R}^d$,
\begin{align*}
&\liminf_{n\to +\infty} 
\mathbb{E}\left[ \left\langle \bm{x}_n - \bm{x}, \nabla f (\bm{x}_n) \right\rangle \right] \leq 0.
\end{align*}
Let $(\alpha_n)_{n\in\mathbb{N}}$, $(\beta_n)_{n\in\mathbb{N}}$, $(\gamma_n)_{n\in\mathbb{N}}$, and $(\delta_n)_{n\in\mathbb{N}}$ satisfy 

{\em (C2)} $\lim_{n \to +\infty} 1/(n \alpha_n) = 0$, $\lim_{n \to +\infty} (1/n) \sum_{k=0}^n \alpha_k = 0$, $\lim_{n \to +\infty} (1/n) \sum_{k=0}^n \beta_k = 0$, $\lim_{n \to +\infty} (1/n) \sum_{k=0}^n \delta_k = 0$, and $\gamma_{n+1} \leq \gamma_n$ ($n\in \mathbb{N}$). 

Then, for all $\bm{x} \in \mathbb{R}^d$,
\begin{align*}
&\limsup_{n \to +\infty} \frac{1}{n} \sum_{k=1}^n \mathbb{E}\left[ \left\langle \bm{x}_k - \bm{x}, \nabla f (\bm{x}_k) \right\rangle \right]
\leq 0.
\end{align*}
In particular, for the case of $\alpha_n := \mathcal{O}(1/n^\eta)$ ($\eta \in (0,1)$), $\beta_n := \beta^n$, and $\delta_n := \delta^n$, Algorithm \ref{algo:1} exhibits convergence such that, for all $\bm{x} \in \mathbb{R}^d$ and all $n\in\mathbb{N}$,
\begin{align*}
&\frac{1}{n} \sum_{k=1}^n \mathbb{E}\left[ \left\langle \bm{x}_k - \bm{x}, \nabla f (\bm{x}_k) \right\rangle \right]
\leq
\frac{D \sum_{i=1}^d 
B_{i}}{2 \tilde{b} n^{1 - \eta}}
+
\frac{\tilde{B}^2 \tilde{M}^2}{2 \tilde{b}\tilde{\zeta}^2 (1 - \eta) n^{\eta}}
+
\frac{\beta \tilde{M}\sqrt{Dd}}{\tilde{b}(1-\beta) n}
+
\frac{4 \delta \hat{M}\sqrt{Dd}}{(1-\delta) n}.
\end{align*}
\end{theorem}

Theorem \ref{thm:2} implies the following proposition.

\begin{proposition}\label{prop:2}
Under the assumptions in Theorem \ref{thm:2} and the convexity of $f_i$ ($i\in [T]$), Algorithm \ref{algo:1} with {\em (C1)} satisfies 
\begin{align*}
\liminf_{n\to +\infty} 
\mathbb{E}\left[ f (\bm{x}_n) - f^\star \right] = 0.
\end{align*}
Moreover, under {\em (C2)}, any accumulation point of $(\tilde{\bm{x}}_n)_{n\in\mathbb{N}}$ defined by $\tilde{\bm{x}}_n := (1/n) \sum_{k=1}^n \bm{x}_k$ almost surely belongs to the solution set of Problem \eqref{convex}. In particular, in the case of $\alpha_n := \mathcal{O}(1/n^\eta)$ ($\eta \in (0,1)$), $\beta_n := \beta^n$, and $\delta_n := \delta^n$, Algorithm \ref{algo:1} exhibits convergence such that, for all $n\in\mathbb{N}$,
\begin{align*}
&\min_{k\in [n]} \mathbb{E}\left[ f (\bm{x}_k) - f^\star \right]
\leq
\frac{D \sum_{i=1}^d 
B_{i}}{2 \tilde{b} n^{1 - \eta}}
+
\frac{\tilde{B}^2 \tilde{M}^2}{2 \tilde{b}\tilde{\zeta}^2 (1 - \eta) n^{\eta}}
+
\frac{\beta \tilde{M}\sqrt{Dd}}{\tilde{b}(1-\beta) n}
+
\frac{4 \delta \hat{M}\sqrt{Dd}}{(1-\delta) n},\\
&\frac{R(T)}{T}
\leq
\frac{D \sum_{i=1}^d 
B_{i}}{2 \tilde{b} T^{1 - \eta}}
+
\frac{\tilde{B}^2 \tilde{M}^2}{2 \tilde{b}\tilde{\zeta}^2 (1 - \eta) T^{\eta}}
+
\frac{\beta \tilde{M}\sqrt{Dd}}{\tilde{b}(1-\beta) T}
+
\frac{4 \delta \hat{M}\sqrt{Dd}}{(1-\delta) T},\\
&\mathbb{E}\left[ f (\tilde{\bm{x}}_n) - f^\star \right]
\leq
\frac{D \sum_{i=1}^d B_{i}}{2 \tilde{b} n^{1 - \eta}}
+
\frac{\tilde{B}^2 \tilde{M}^2}{2 \tilde{b}\tilde{\zeta}^2 (1 - \eta) n^{\eta}}
+
\frac{\beta \tilde{M}\sqrt{Dd}}{\tilde{b}(1-\beta) n}
+
\frac{4 \delta \hat{M}\sqrt{Dd}}{(1-\delta) n}.
\end{align*}
\end{proposition}

\section{Numerical Experiments}
\label{sec:4}
\subsection{Algorithms}
We define $\mathsf{H}^{\mathrm{Adam}}_n$ by 
\begin{align}\label{adam}
\begin{split}
&\bm{v}_n := \theta \bm{v}_{n-1} + (1-\theta) \bm{\mathsf{G}}_{n} \odot \bm{\mathsf{G}}_{n},\\
&\bar{\bm{v}}_n := (1 - \theta^{n+1})^{-1} \bm{v}_n,\\
&\hat{\bm{v}}_n := \max \{ \hat{v}_{n-1,i}, \bar{v}_{n,i} \},\\
&\mathsf{H}^{\mathrm{Adam}}_n := \mathsf{diag} \left(\sqrt{\hat{v}_{n,i}} \right),
\end{split}
\end{align}
where $\theta \in [0,1)$, $\bm{v}_n = \hat{\bm{v}}_n = \bm{0}$, and $\bm{x} \odot \bm{x} := (x_i^2)_{i=1}^d$ for $\bm{x} = (x_i)_{i=1}^d \in \mathbb{R}^d$. Algorithm \ref{algo:1} with $\mathsf{H}_n = \mathsf{H}^{\mathrm{Adam}}_n$ and $\gamma_n = \delta_n = 0$ resembles Adam \citep{adam}\footnote{While Adam \citep{adam} used $\mathsf{H}_n := \mathsf{diag} (\sqrt{\bar{v}_{n,i}})$, we use $\mathsf{H}^{\mathrm{Adam}}_n$ to guarantee convergence \citep{iiduka2021}.}. We also define $\mathsf{H}^{\mathrm{AMSGrad}}_n$ by
\begin{align}\label{amsg}
\begin{split}
&\bm{v}_n := \theta \bm{v}_{n-1} + (1-\theta) \bm{\mathsf{G}}_{n} \odot \bm{\mathsf{G}}_{n},\\
&\hat{\bm{v}}_n := \max \{ \hat{v}_{n-1,i}, {v}_{n,i} \},\\
&\mathsf{H}^{\mathrm{AMSGrad}}_n := \mathsf{diag} \left(\sqrt{\hat{v}_{n,i}} \right).
\end{split}
\end{align}
Algorithm \ref{algo:1} with $\mathsf{H}_n = \mathsf{H}^{\mathrm{AMSGrad}}_n$ and $\gamma_n = \delta_n = 0$ is equivalent to AMSGrad \citep{reddi2018}. We compared Algorithm \ref{algo:1} with SGD \citep{robb1951}, Momentum \citep{polyak1964,nes1983}, RMSprop \citep{rmsprop}, Adagrad \citep{adagrad}, Adam \citep{adam}, AMSGrad \citep{reddi2018}, and AdamW \citep{loshchilov2018decoupled} (see Sections \ref{appen:5} and \ref{appen:6} for details).

The experiments used two Intel(R) Xeon(R) Gold 6148 2.4-GHz CPUs with 20 cores each and a 16-GB NVIDIA Tesla V100 900-Gbps GPU. The experimental code was written in Python 3.7.5 using the NumPy 1.21.6 package and PyTorch 1.7.1 package. Python implementations of the optimizers used in the numerical experiments are available at\\ 
\url{https://github.com/iiduka-researches/202210-izumi}.

\subsection{Image classification}
For image classification, we used a residual neural network (ResNet), a relatively deep model derived from convolutional neural networks (CNNs). It was applied to the CIFAR-100 and CIFAR-10 benchmark datasets\footnote{\url{https://www.cs.toronto.edu/~kriz/cifar.html}} for image classification. These datasets respectively comprise 100 and 10 classes and each contain 6,000 32 $\times$ 32 pixel images. Each dataset was separated into 50,000 training images and 10,000 test images (1,000 randomly selected images per class). A 18-layer ResNet was trained on CIFAR-100 and CIFAR-10 \citep{he2015}. Following common practice in image classification, cross-entropy was used as the loss function for model fitting.

The image classification results for Algorithm \ref{algo:1} with constant and diminishing learning rates are shown in Figures \ref{fig:1_c}, \ref{fig:1_d}, \ref{fig:3_c}, and \ref{fig:3_d}, where panels (a), (b), and (c) respectively show the training loss function value, training classification accuracy score, and test classification accuracy score as functions of the number of epochs. Figures \ref{fig:5_0}, \ref{fig:5}, \ref{fig:5_1}, \ref{fig:7_0}, \ref{fig:9}, and \ref{fig:7_1} present box-plot comparisons of Algorithm \ref{algo:1} with constant (panel (a)) and diminishing (panel (b)) learning rates in terms of the training loss function value, training classification accuracy score, and test classification accuracy score. As shown, Algorithm \ref{algo:1} performed better with constant learning rates than with diminishing learning rates. For example, looking at Figures \ref{fig:1_c}(a) and \ref{fig:5_0}(a)(ResNet-18 on CIFAR-100), we see that SCGAdam performed best at minimizing the training loss function. Looking at Figures \ref{fig:3_c}(a) and \ref{fig:7_0}(a)(ResNet-18 on CIFAR-10), we see that SCGAdam performed best at minimizing the training loss function. Figures \ref{fig:1_c}(c) and \ref{fig:5_1}(a)(ResNet-18 on CIFAR-100) also show that, when the number of epochs was 200, Momentum had a higher classification accuracy in terms of test classification accuracy score than those of the other algorithms. Meanwhile, AMSGrad and SCGAdam had 70 \% classification accuracy in terms of test classification accuracy score faster than the other algorithms (see Figure \ref{fig:1_c}(c)). Figure \ref{fig:3_c}(c) (ResNet-18 on CIFAR-10) also shows that, when the number of epochs was 200, Momentum had a higher classification accuracy in that score than those of the other algorithms, while SCGAdam had 90 \% classification accuracy in terms of test classification accuracy score faster than the other algorithms.

\begin{figure*}[htbp] 
 	\begin{tabular}{ccc}
 	\begin{minipage}{0.323\linewidth}
 		\centering
 		\includegraphics[width=1\textwidth]{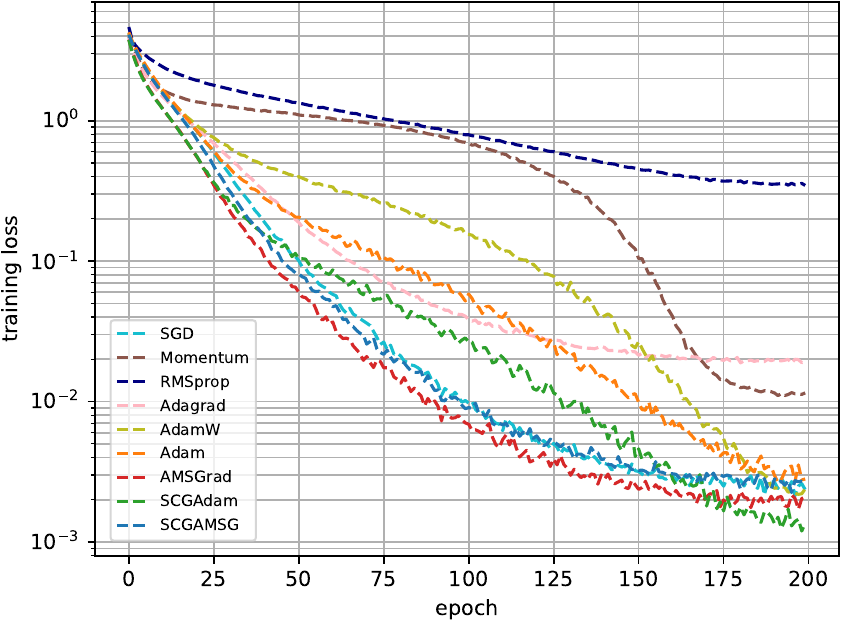}%
		\subcaption{}
		\label{fig:1_c_l}
 	\end{minipage}
 	\begin{minipage}{0.323\linewidth}
 		\centering
 		\includegraphics[width=1\textwidth]{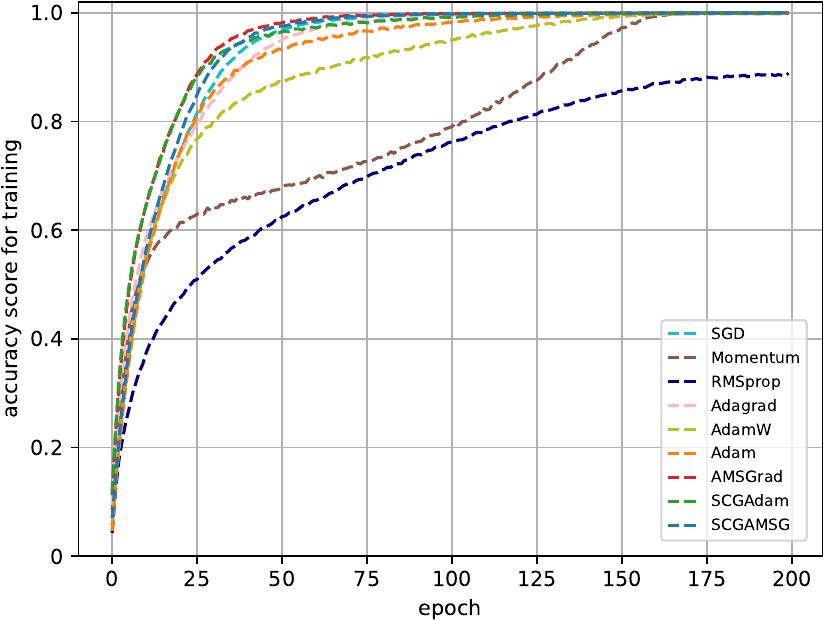}%
		\subcaption{}
		\label{fig:1_c_a}
 	\end{minipage}
 	\begin{minipage}{0.323\linewidth}
 		\centering
 		\includegraphics[width=1\textwidth]{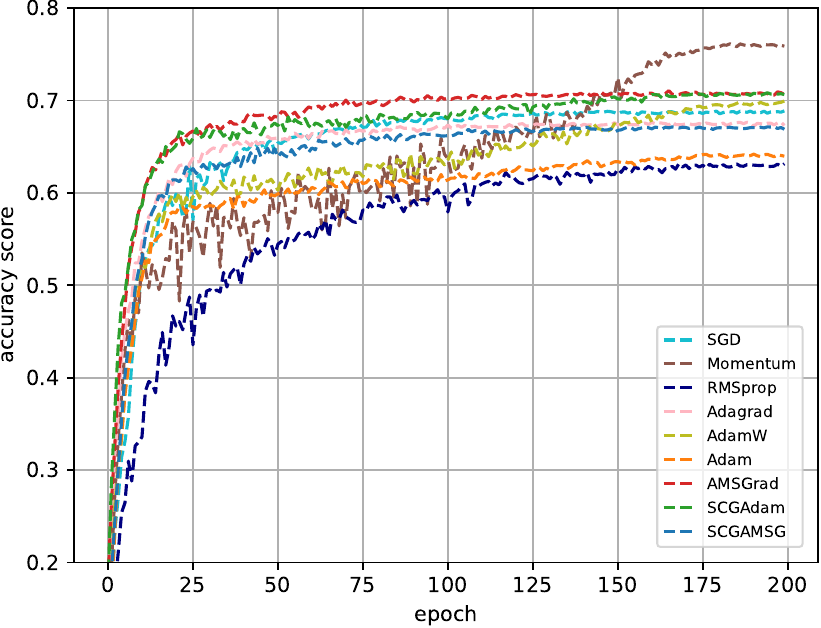}%
		\subcaption{}
		\label{fig:1_c_a_t}
 	\end{minipage}
 	\end{tabular}
	\vspace*{-5pt}
\caption{Results of Algorithm \ref{algo:1} with constant learning rates for training ResNet-18 on the CIFAR-100 dataset: (a) training loss function value, (b) training classification accuracy score, and (c) test classification accuracy score.}
\label{fig:1_c}
\end{figure*}

\begin{figure*}[h] 
 	\begin{tabular}{ccc}
 	\begin{minipage}{0.323\linewidth}
 		\centering
 		\includegraphics[width=1\textwidth]{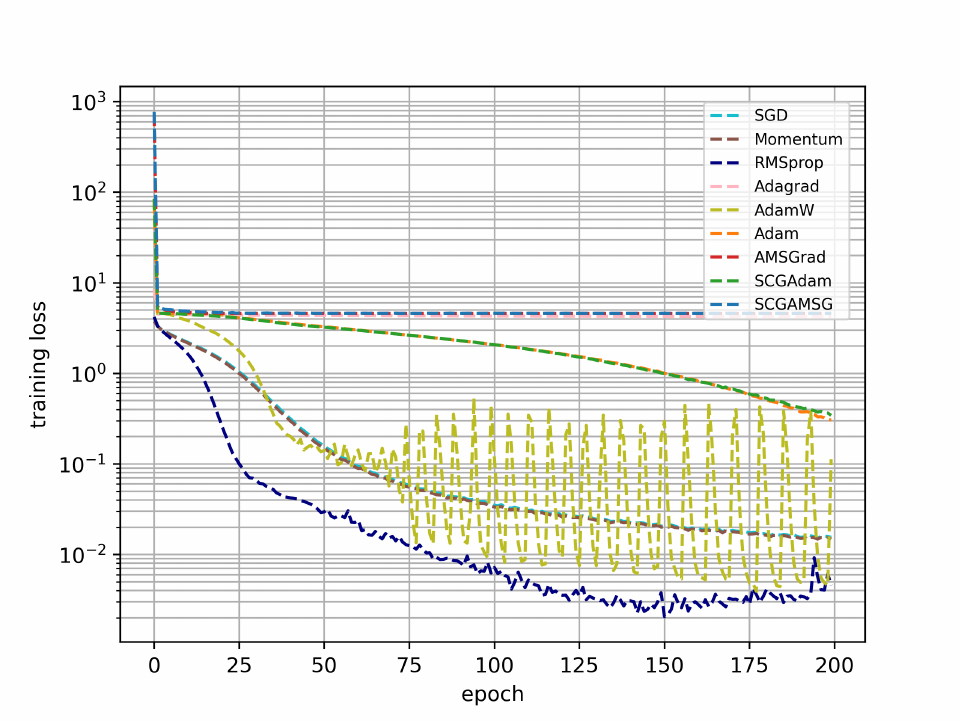}%
		\subcaption{}
		\label{fig:1_d_l}
 	\end{minipage}
 	\begin{minipage}{0.323\linewidth}
 		\centering
 		\includegraphics[width=1\textwidth]{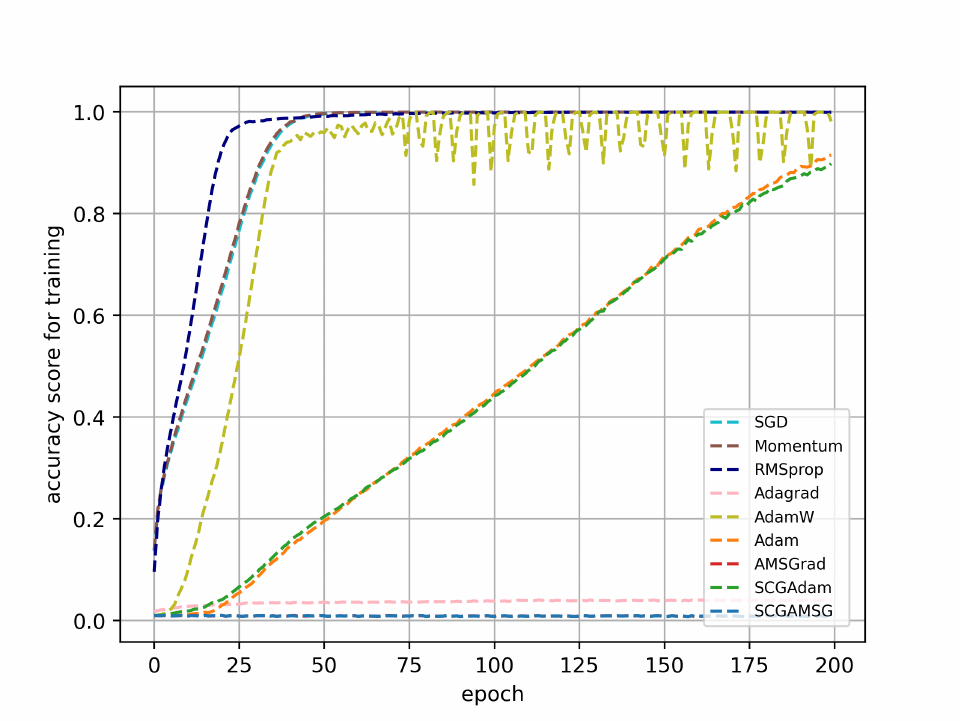}%
		\subcaption{}
		\label{fig:1_d_a}
 	\end{minipage}
 	\begin{minipage}{0.323\linewidth}
 		\centering
 		\includegraphics[width=1\textwidth]{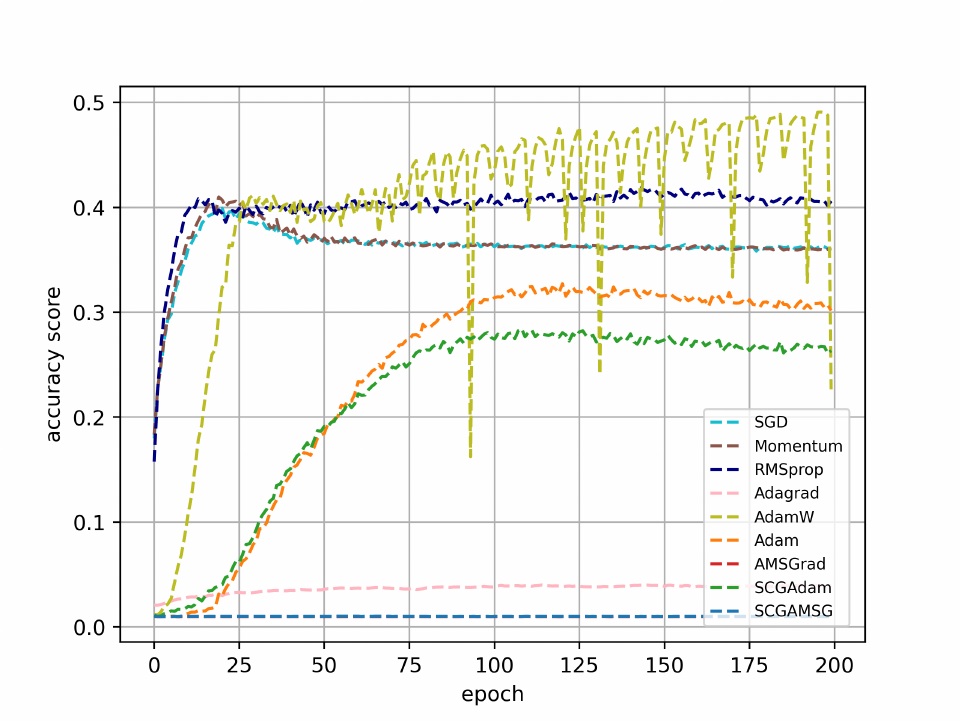}%
		\subcaption{}
		\label{fig:1_d_a_t}
 	\end{minipage}
 	\end{tabular}
	\vspace*{-5pt}
\caption{Results of Algorithm \ref{algo:1} with diminishing learning rates for training ResNet-18 on the CIFAR-100 dataset: (a) training loss function value, (b) training classification accuracy score, and (c) test classification accuracy score.}
\label{fig:1_d}
\end{figure*}

\begin{figure*}[htbp] 
 	\begin{tabular}{cc}
 	\begin{minipage}{0.5\linewidth}
 		\centering
 		\includegraphics[width=0.85\textwidth]{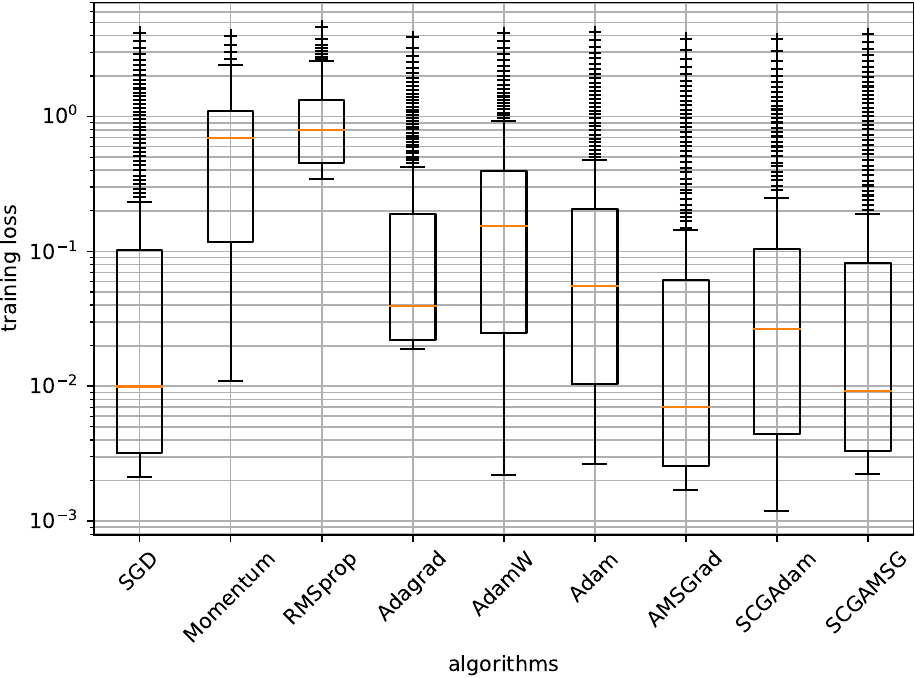}
		\subcaption{}
		\label{fig:5_c_f}
 	\end{minipage}
 	\begin{minipage}{0.5\linewidth}
 		\centering
 		\includegraphics[width=0.85\textwidth]{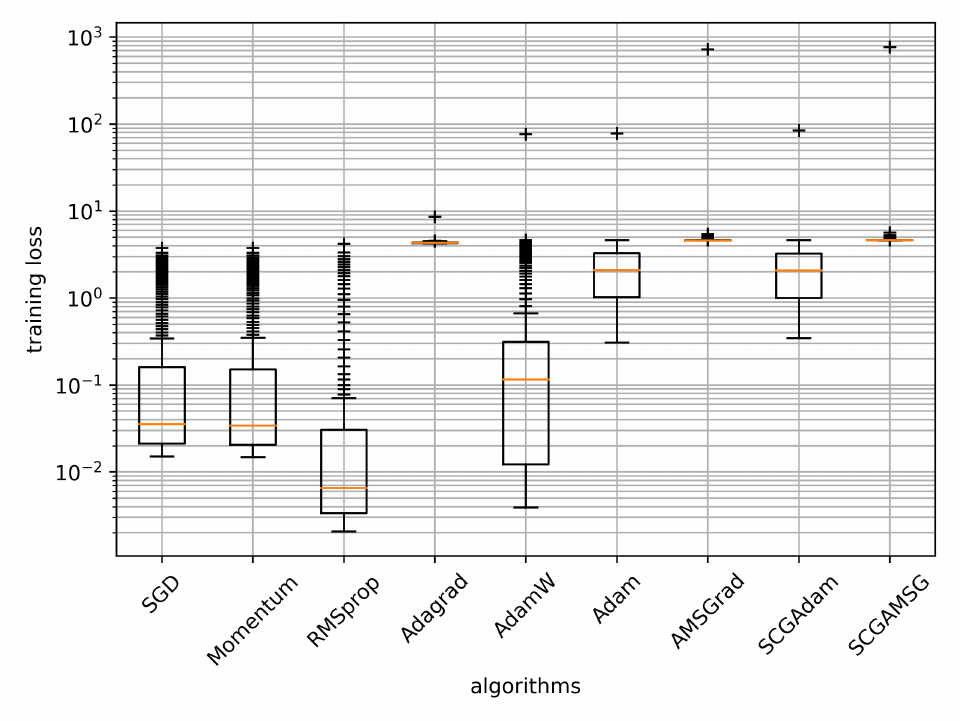}
		\subcaption{}
		\label{fig:5_d_f}
 	\end{minipage}
 	\end{tabular}
\caption{Box plots of training loss function values for Algorithm \ref{algo:1} for training ResNet-18 on the CIFAR-100 dataset: (a) constant learning rates and (b) diminishing learning rates.}
\label{fig:5_0}
\end{figure*}

\begin{figure*}[htbp] 
 	\begin{tabular}{cc}
 	\begin{minipage}{0.5\linewidth}
 		\centering
 		\includegraphics[width=0.85\textwidth]{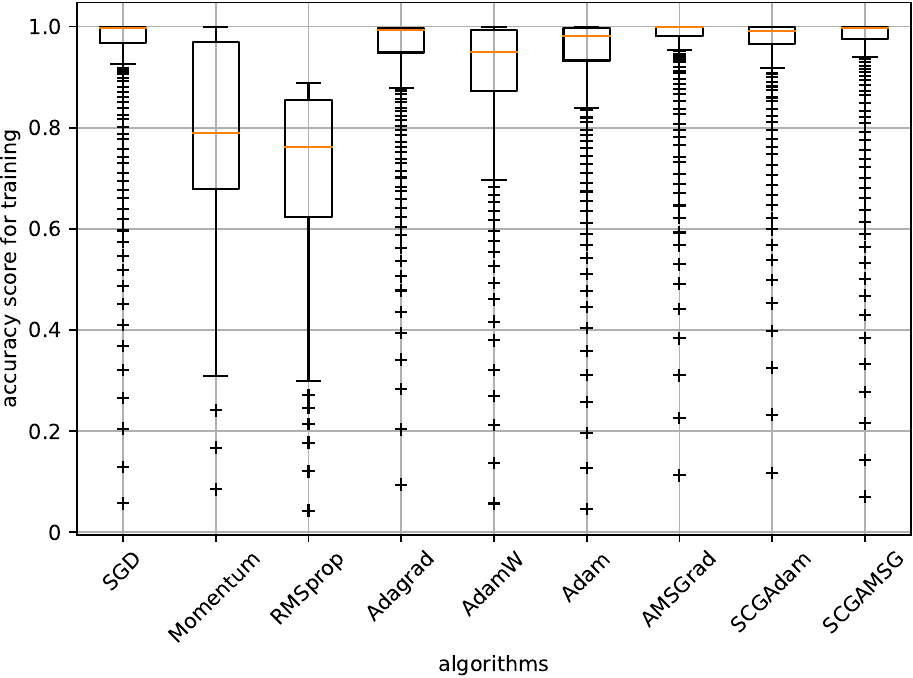}
		\subcaption{}
		\label{fig:5_c}
 	\end{minipage}
 	\begin{minipage}{0.5\linewidth}
 		\centering
 		\includegraphics[width=0.85\textwidth]{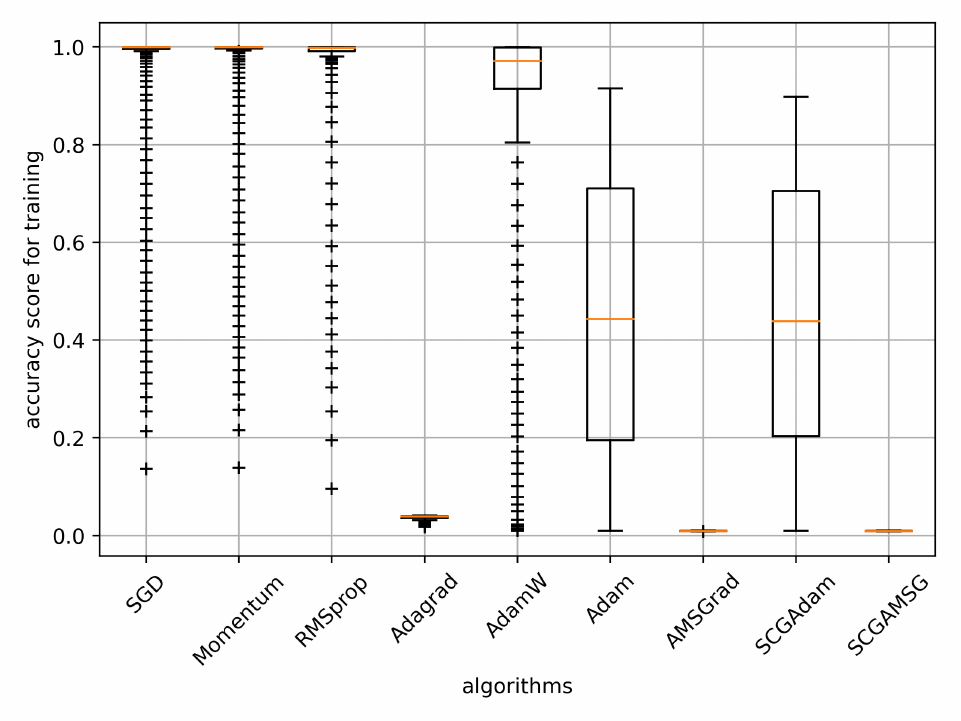}
		\subcaption{}
		\label{fig:5_d}
 	\end{minipage}
 	\end{tabular}
\caption{Box plots of training classification accuracy score for Algorithm \ref{algo:1} for training ResNet-18 on the CIFAR-100 dataset: (a) constant learning rates and (b) diminishing learning rates.}
\label{fig:5}
\end{figure*}

\begin{figure}[htbp] 
 	\begin{tabular}{cc}
 	\begin{minipage}{0.5\linewidth}
 		\centering
 		\includegraphics[width=0.85\textwidth]{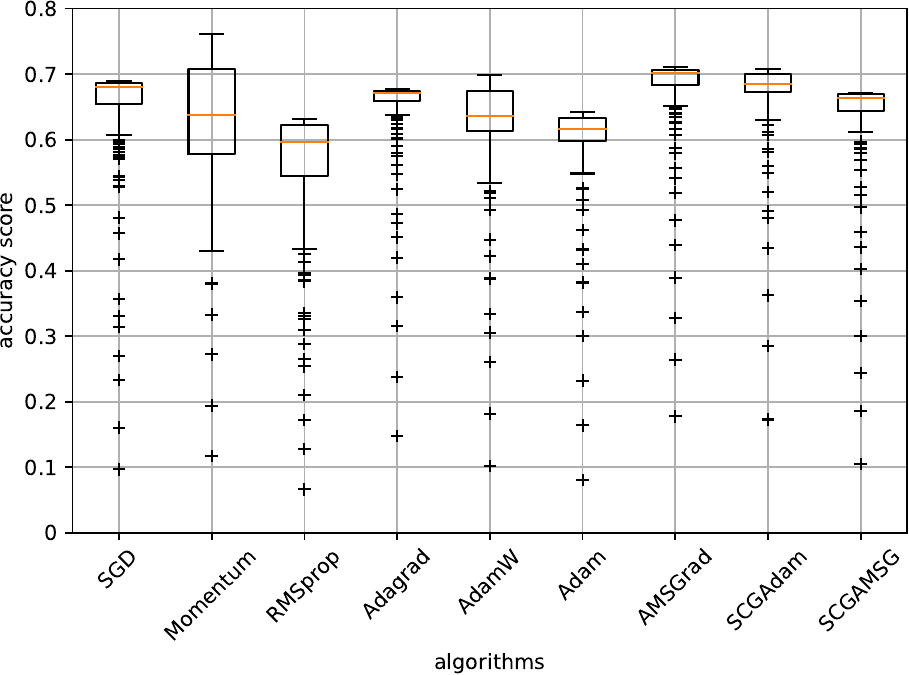}
		\subcaption{}
		\label{fig:5_c_1}
 	\end{minipage}
 	\begin{minipage}{0.5\linewidth}
 		\centering
 		\includegraphics[width=0.85\textwidth]{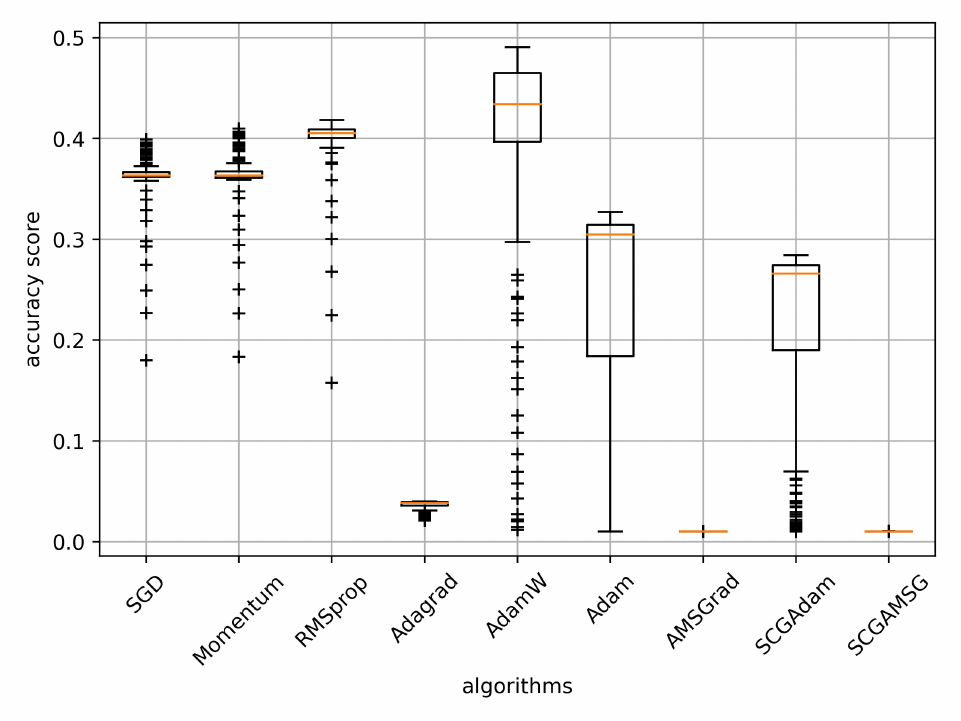}
		\subcaption{}
		\label{fig:5_d_1}
 	\end{minipage}
 	\end{tabular}
\caption{Box plots of test classification accuracy score for Algorithm \ref{algo:1} for training ResNet-18 on the CIFAR-100 dataset: (a) constant learning rates and (b) diminishing learning rates.}
\label{fig:5_1}
\end{figure}

\begin{figure}[h] 
 	\begin{tabular}{ccc}
 	\begin{minipage}{0.323\linewidth}
 		\centering
 		\includegraphics[width=1\textwidth]{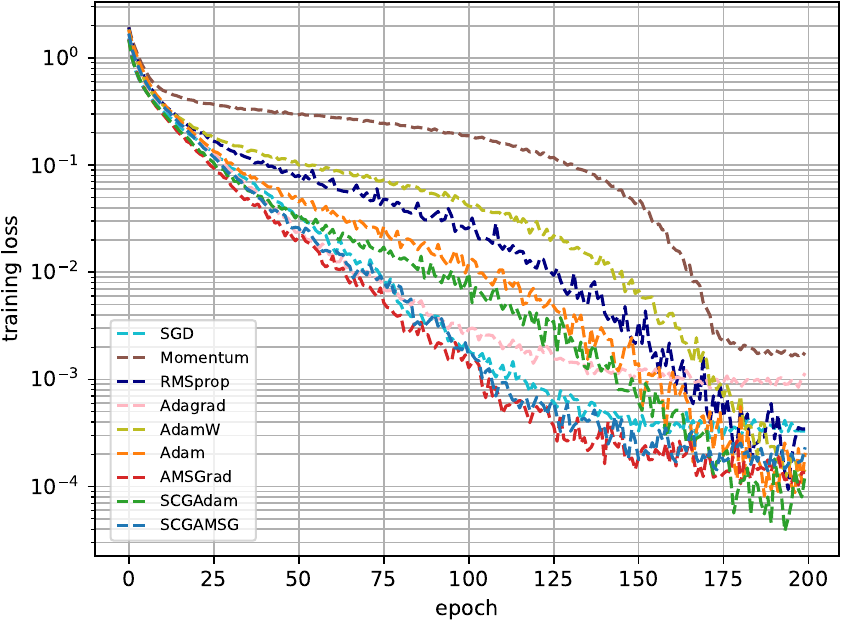}
		\subcaption{}
		\label{fig:3_c_l}
 	\end{minipage}
 	\begin{minipage}{0.323\linewidth}
 		\centering
 		\includegraphics[width=1\textwidth]{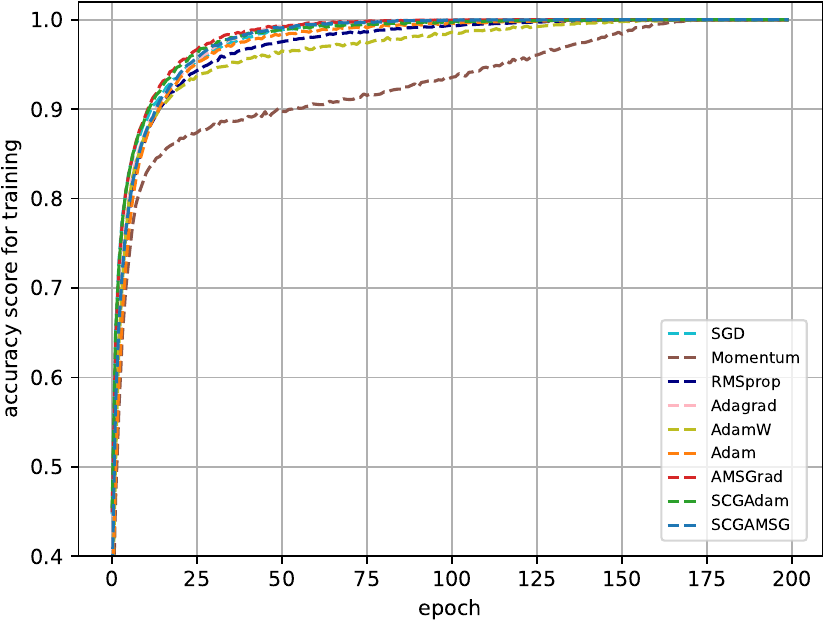}
		\subcaption{}
		\label{fig:3_c_a}
 	\end{minipage}
 	\begin{minipage}{0.323\linewidth}
 		\centering
 		\includegraphics[width=1\textwidth]{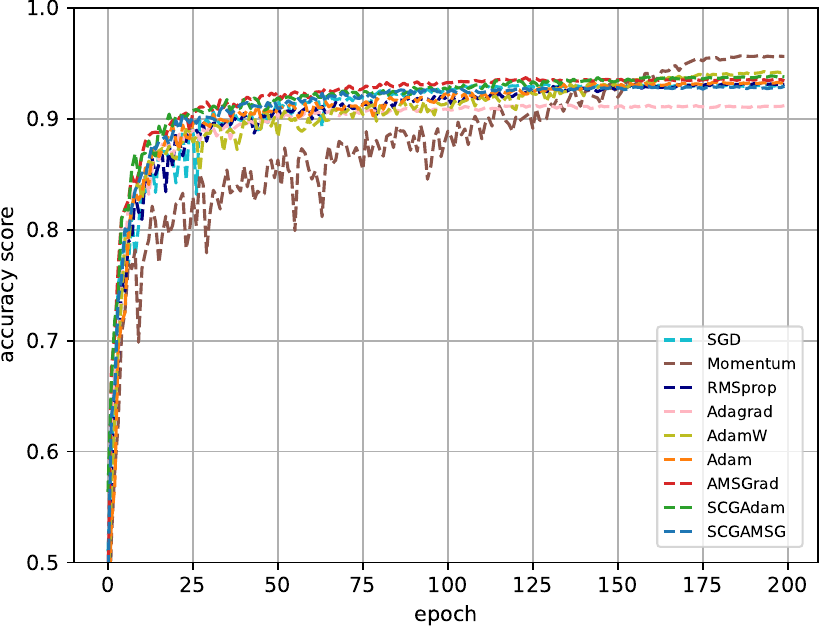}
		\subcaption{}
		\label{fig:3_c_a_t}
 	\end{minipage}
 	\end{tabular}
\caption{Results of Algorithm \ref{algo:1} with constant learning rates for training ResNet-18 on the CIFAR-10 dataset: (a) training loss function value, (b) training classification accuracy score, and (c) test classification accuracy score.}
\label{fig:3_c}
\end{figure}

\begin{figure}[h] 
 	\begin{tabular}{ccc}
 	\begin{minipage}{0.323\linewidth}
 		\centering
 		\includegraphics[width=1\textwidth]{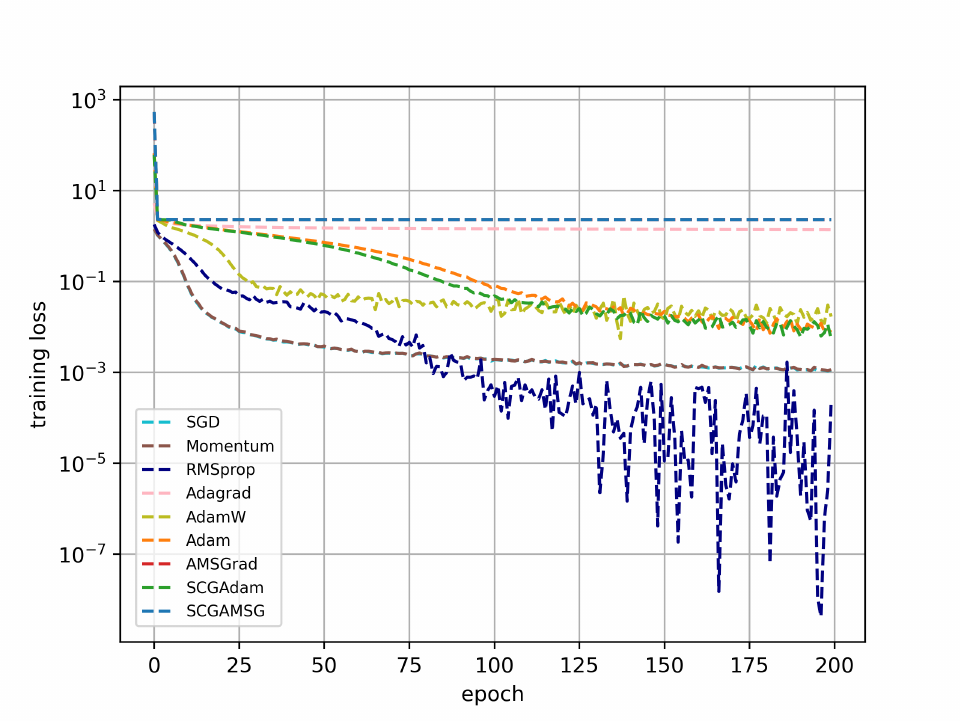}
		\subcaption{}
		\label{fig:3_d_l}
 	\end{minipage}
 	\begin{minipage}{0.323\linewidth}
 		\centering
 		\includegraphics[width=1\textwidth]{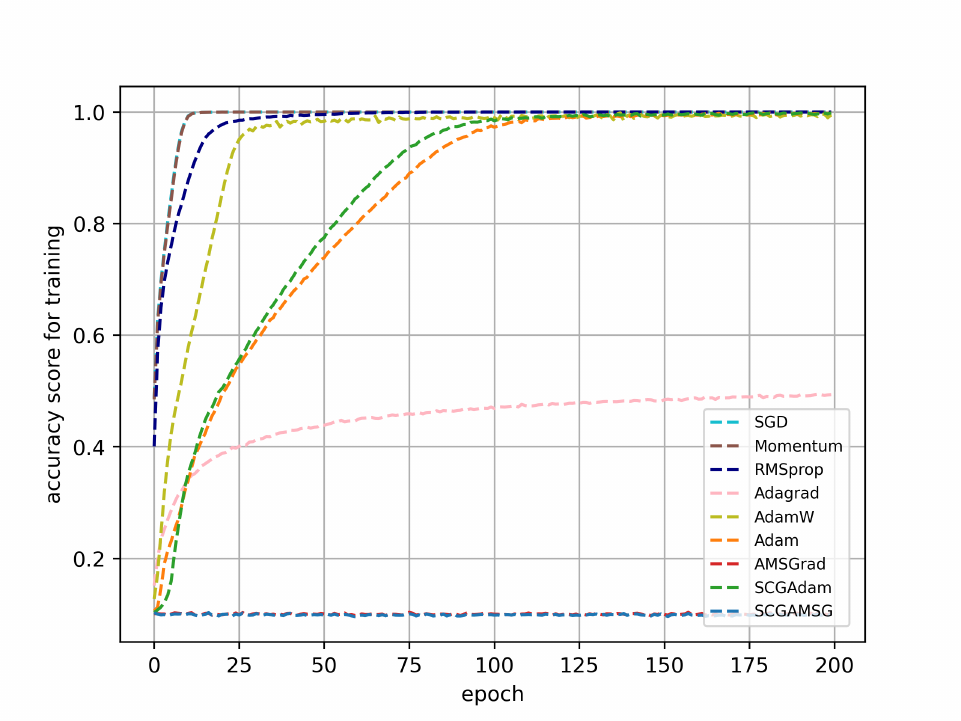}
		\subcaption{}
		\label{fig:3_d_a}
 	\end{minipage}
 	\begin{minipage}{0.323\linewidth}
 		\centering
 		\includegraphics[width=1\textwidth]{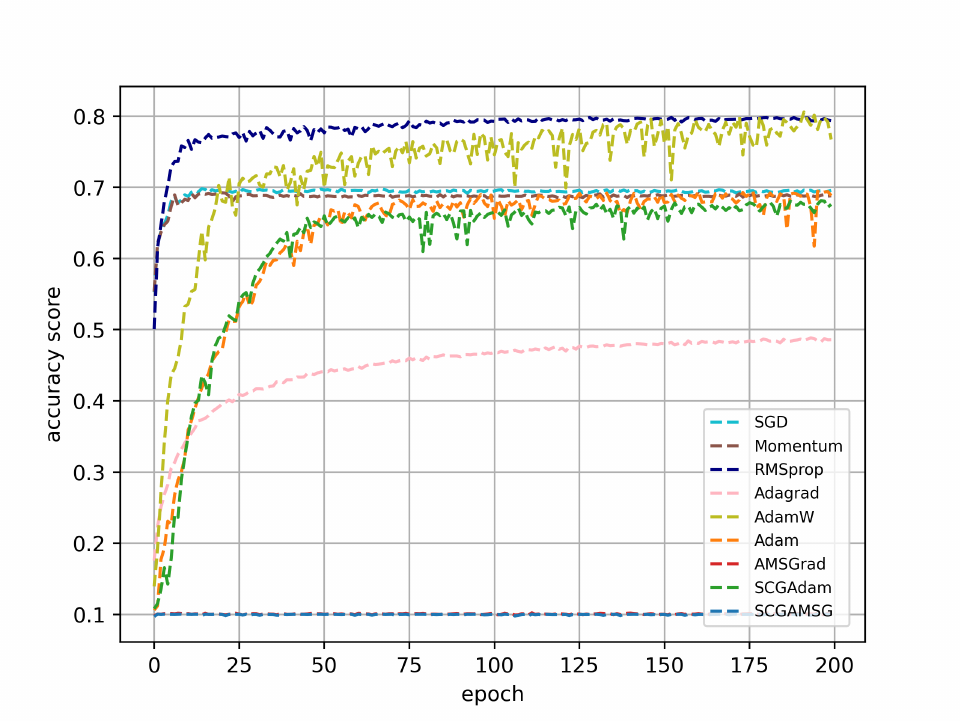}
		\subcaption{}
		\label{fig:3_d_a_t}
 	\end{minipage}
 	\end{tabular}
\caption{Results of Algorithm \ref{algo:1} with diminishing learning rates for training ResNet-18 on the CIFAR-10 dataset: (a) training loss function value, (b) training classification accuracy score, and (c) test classification accuracy score.}
\label{fig:3_d}
\end{figure}

\begin{figure*}[h] 
 	\begin{tabular}{cc}
 	\begin{minipage}{0.5\linewidth}
 		\centering
 		\includegraphics[width=1\textwidth]{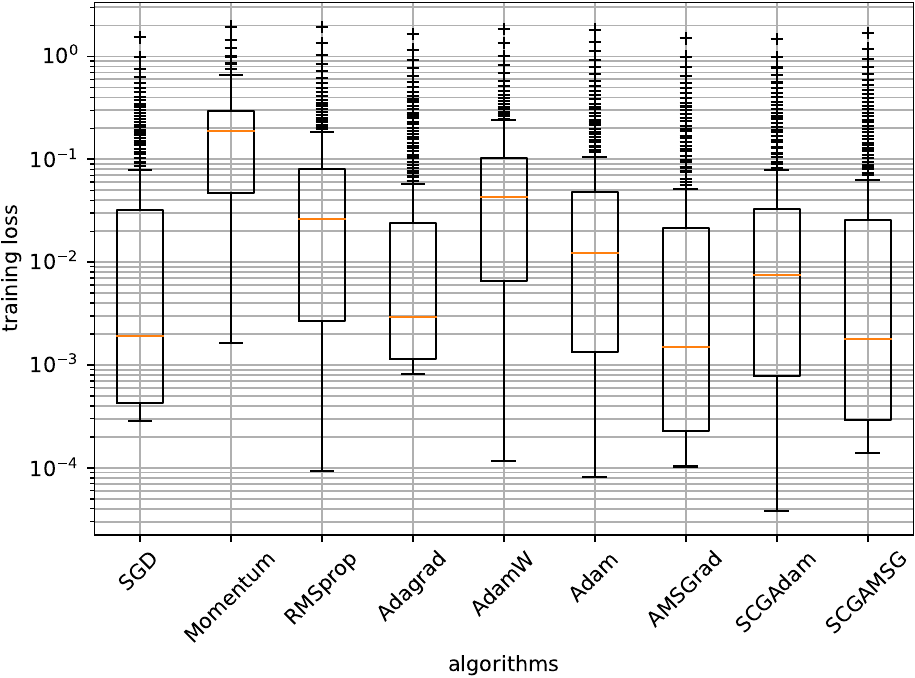}
		\subcaption{}
		\label{fig:7_c_f}
 	\end{minipage}
 	\begin{minipage}{0.5\linewidth}
 		\centering
 		\includegraphics[width=1\textwidth]{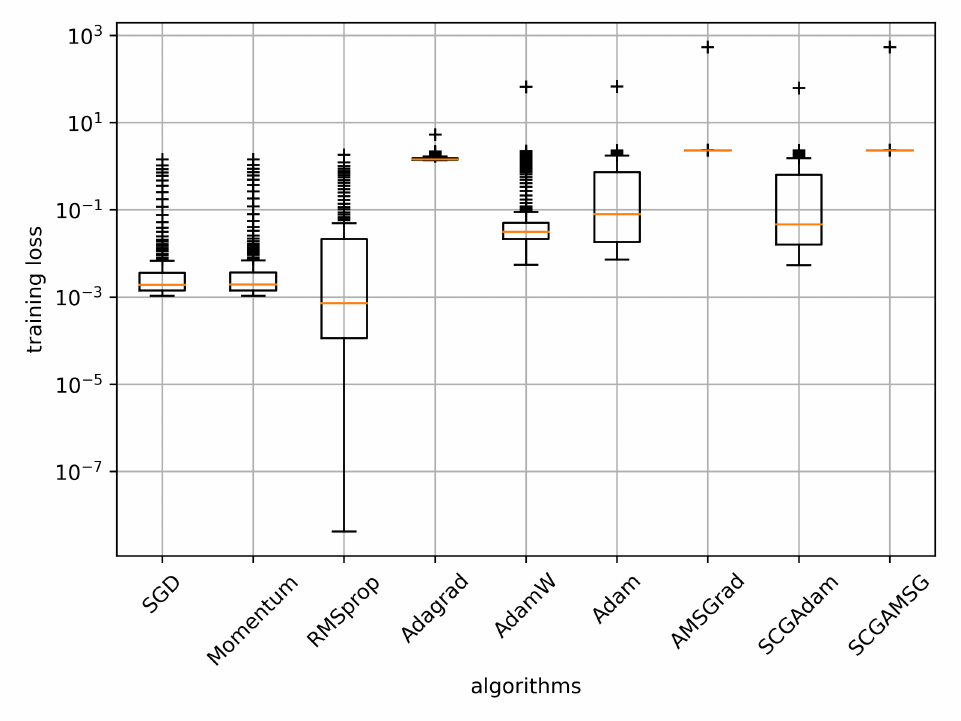}
		\subcaption{}
		\label{fig:7_d_f}
 	\end{minipage}
 	\end{tabular}
\caption{Box plots of training loss function values for Algorithm \ref{algo:1} for training ResNet-18 on the CIFAR-10 dataset: (a) constant learning rates and (b) diminishing learning rates.}
\label{fig:7_0}
\end{figure*}

\begin{figure*}[h] 
 	\begin{tabular}{cc}
 	\begin{minipage}{0.5\linewidth}
 		\centering
 		\includegraphics[width=1\textwidth]{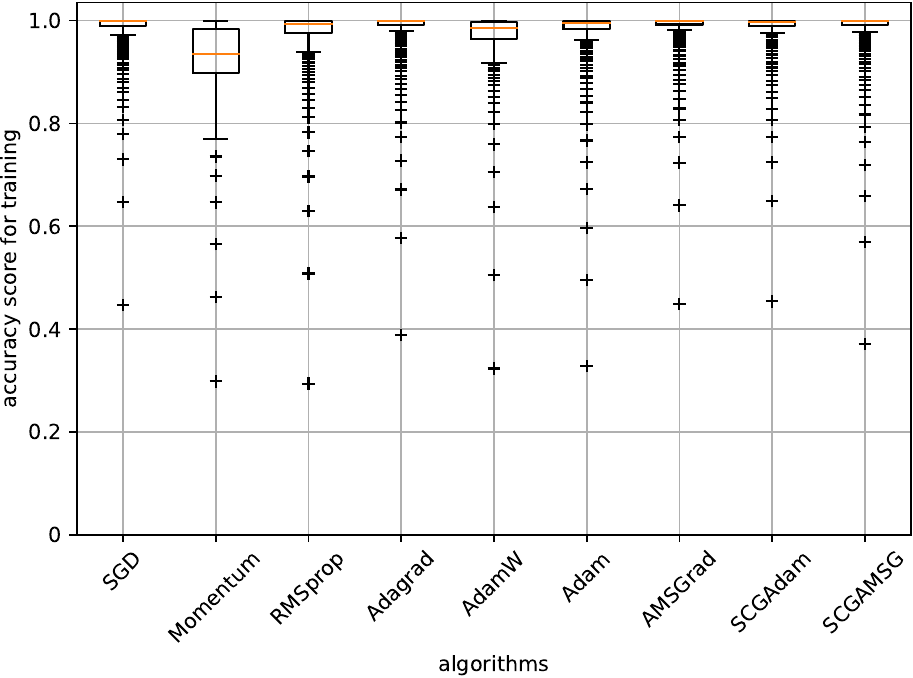}
		\subcaption{}
		\label{fig:7_c}
 	\end{minipage}
 	\begin{minipage}{0.5\linewidth}
 		\centering
 		\includegraphics[width=1\textwidth]{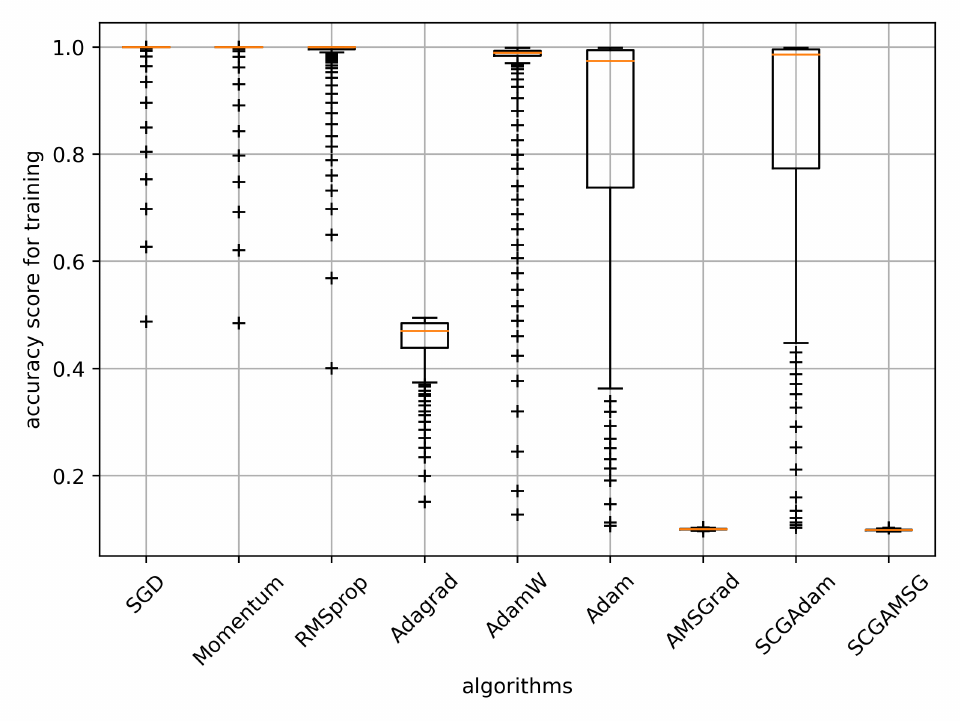}
		\subcaption{}
		\label{fig:7_d}
 	\end{minipage}
 	\end{tabular}
\caption{Box plots of training classification accuracy score for Algorithm \ref{algo:1} for training ResNet-18 on the CIFAR-10 dataset: (a) constant learning rates and (b) diminishing learning rates.}
\label{fig:9}
\end{figure*}

\begin{figure*}[h] 
 	\begin{tabular}{cc}
 	\begin{minipage}{0.5\linewidth}
 		\centering
 		\includegraphics[width=1\textwidth]{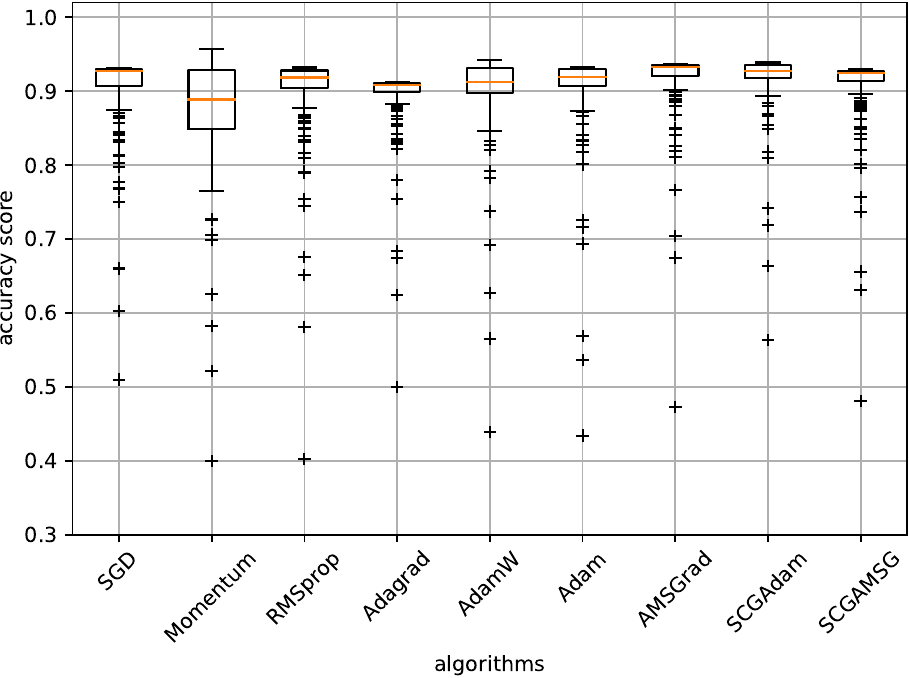}
		\subcaption{}
		\label{fig:7_c_1}
 	\end{minipage}
 	\begin{minipage}{0.5\linewidth}
 		\centering
 		\includegraphics[width=1\textwidth]{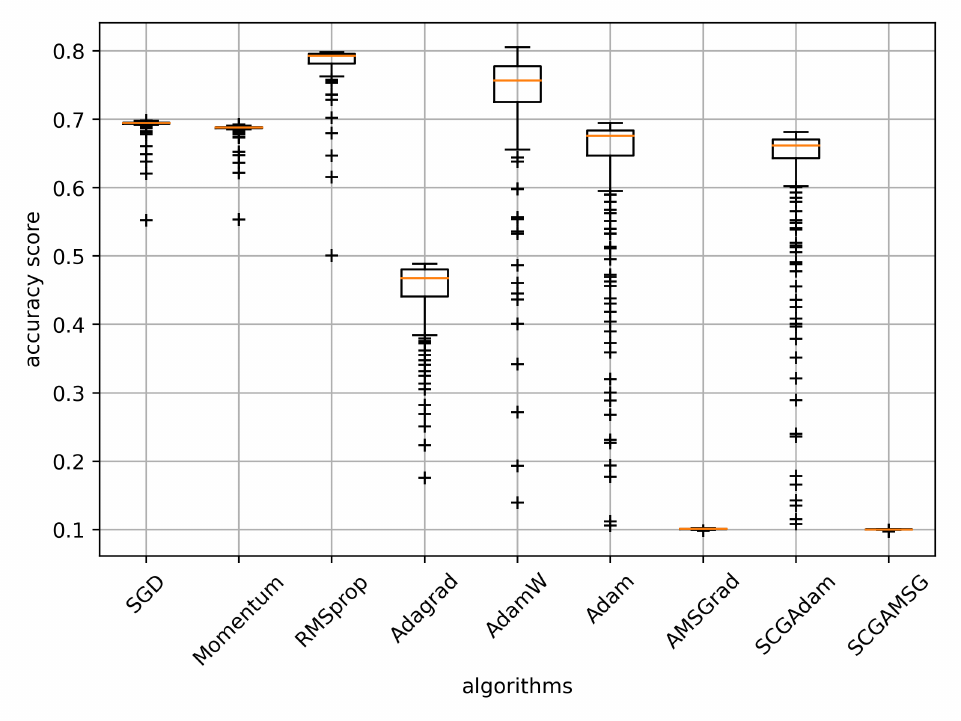}
		\subcaption{}
		\label{fig:7_d_1}
 	\end{minipage}
 	\end{tabular}
\caption{Box plots of test classification accuracy score for Algorithm \ref{algo:1} for training ResNet-18 on the CIFAR-10 dataset: (a) constant learning rates and (b) diminishing learning rates.}
\label{fig:7_1}
\end{figure*}

\clearpage
\subsection{Text classification}
For text classification, we used long short-term memory (LSTM), an artificial recurrent neural network (RNN) architecture used for deep learning that is based on natural language processing. The IMDb dataset \citep{imdb} 
was used; this dataset comprises 50,000 movie reviews and associated binary sentiment polarity labels. The data were split into 25,000 training sets and 25,000 test sets. The data were classified using a multilayer LSTM neural network and AlphaDropout 
 for overfitting suppression. This network included one affine layer and employed a sigmoid activation function for the output. As with the image classification, cross-entropy was used as the loss function for the model fitting.

The text classification results for Algorithm \ref{algo:1} with constant learning rates are shown in Figures \ref{fig:2_c} and \ref{fig:2_d}, where panels (a), (b), and (c) respectively show the training loss function value, training classification accuracy score, and test classification accuracy score as functions of the number of epochs. Figures \ref{fig:6_0}, \ref{fig:6}, and \ref{fig:7} present box-plot comparisons of Algorithm \ref{algo:1} with constant (panel (a)) and diminishing (panel (b)) learning rates in terms of the training loss function value, training classification accuracy score, and test classification accuracy score. As was the case with the image datasets, Algorithm \ref{algo:1} performed better with constant learning rates than with diminishing learning rates. For example, looking at Figures \ref{fig:2_c}(a), and \ref{fig:6_0}(a), we see that SCGAdam, along with RMSprop, performed the best at minimizing the training loss function.
\vspace*{-10pt}
\begin{figure*}[h] 
 	\begin{tabular}{ccc}
 	\begin{minipage}{0.323\linewidth}
 		\centering
 		\includegraphics[width=1\textwidth]{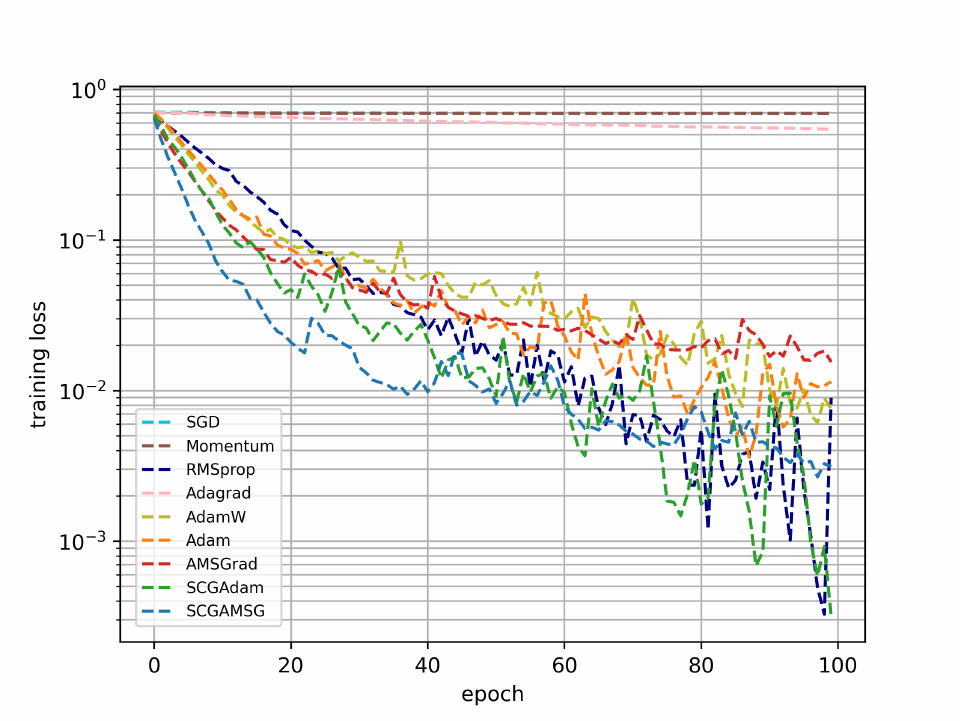}
		\subcaption{}
		\label{fig:2_c_l}
 	\end{minipage}
 	\begin{minipage}{0.323\linewidth}
 		\centering
 		\includegraphics[width=1\textwidth]{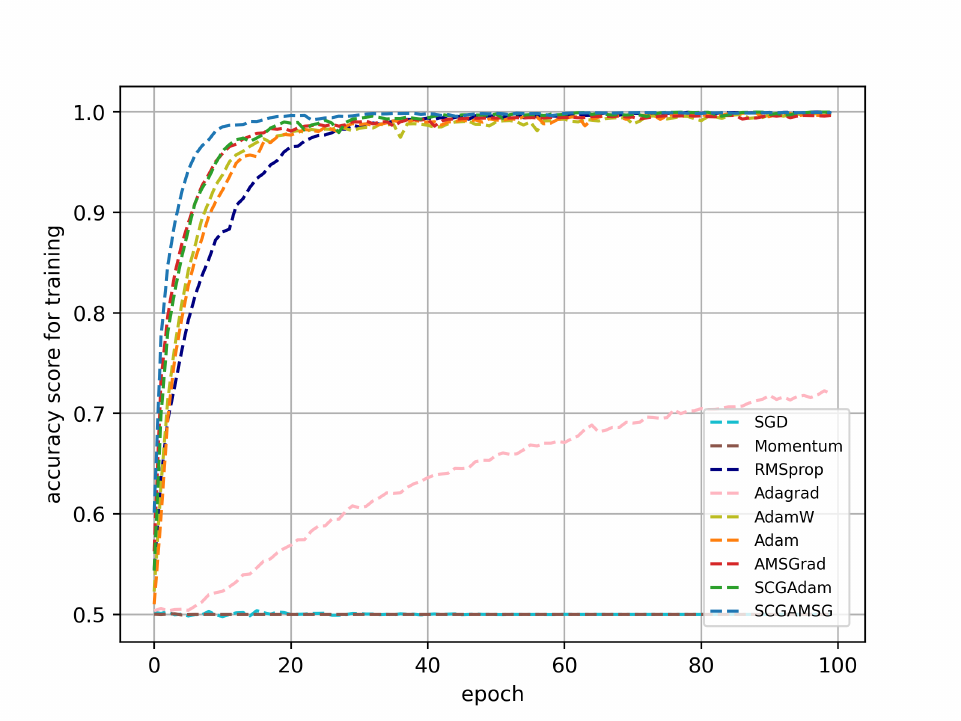}
		\subcaption{}
		\label{fig:2_c_a}
 	\end{minipage}
 	\begin{minipage}{0.323\linewidth}
 		\centering
 		\includegraphics[width=1\textwidth]{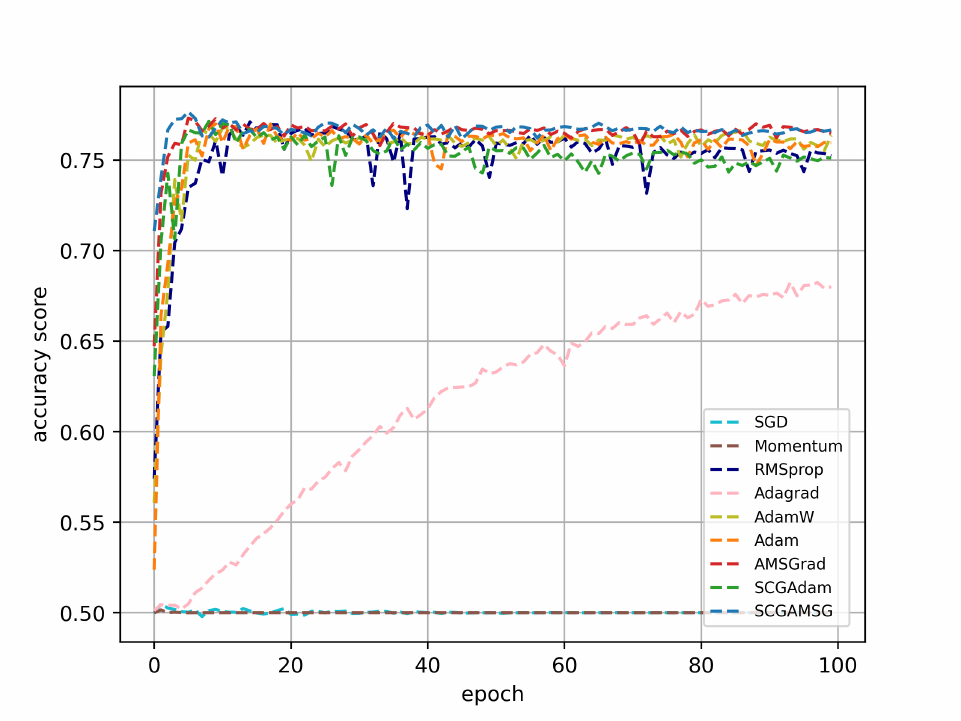}
		\subcaption{}
		\label{fig:2_c_a_t}
 	\end{minipage}
 	\end{tabular}
	\vspace*{-10pt}
\caption{Results of Algorithm \ref{algo:1} with constant learning rates on the IMDb dataset: (a) training loss function value, (b) training classification accuracy score, and (c) test classification accuracy score.}
\label{fig:2_c}
\end{figure*}

\vspace*{-10pt}
\begin{figure*}[htbp] 
 	\begin{tabular}{ccc}
 	\begin{minipage}{0.323\linewidth}
 		\centering
 		\includegraphics[width=1\textwidth]{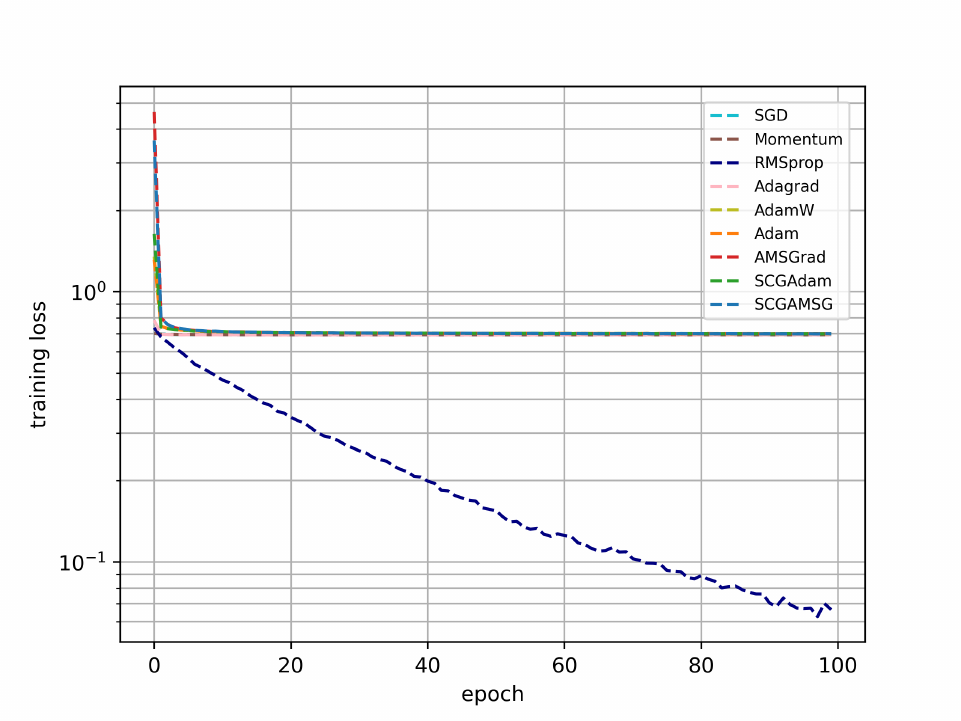}
		\subcaption{}
		\label{fig:2_d_l}
 	\end{minipage}
 	\begin{minipage}{0.323\linewidth}
 		\centering
 		\includegraphics[width=1\textwidth]{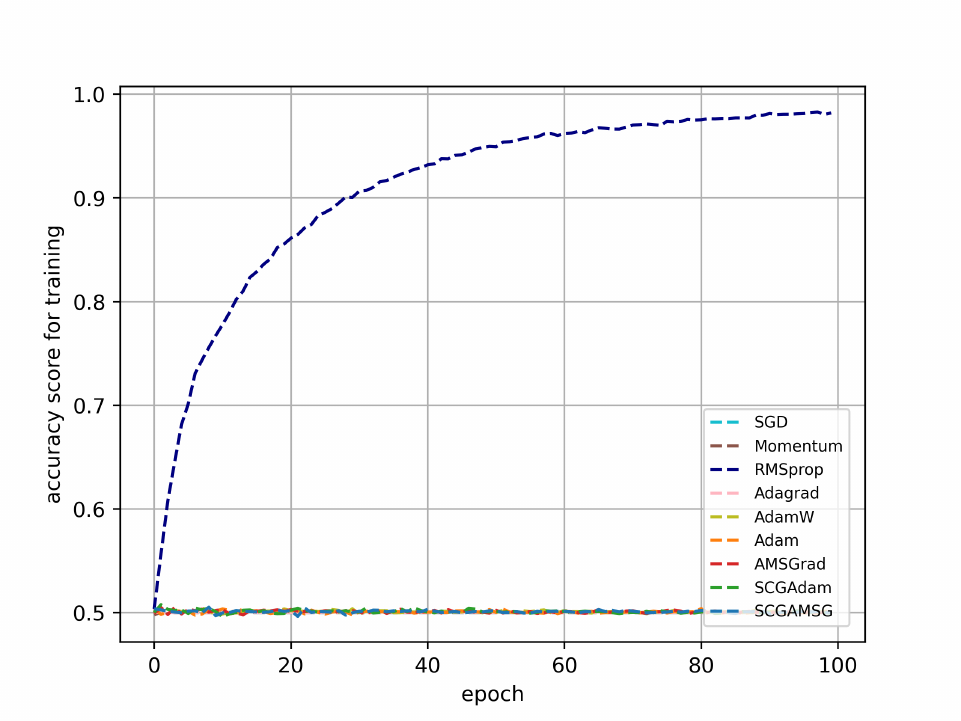}
		\subcaption{}
		\label{fig:2_d_a}
 	\end{minipage}
 	\begin{minipage}{0.323\linewidth}
 		\centering
 		\includegraphics[width=1\textwidth]{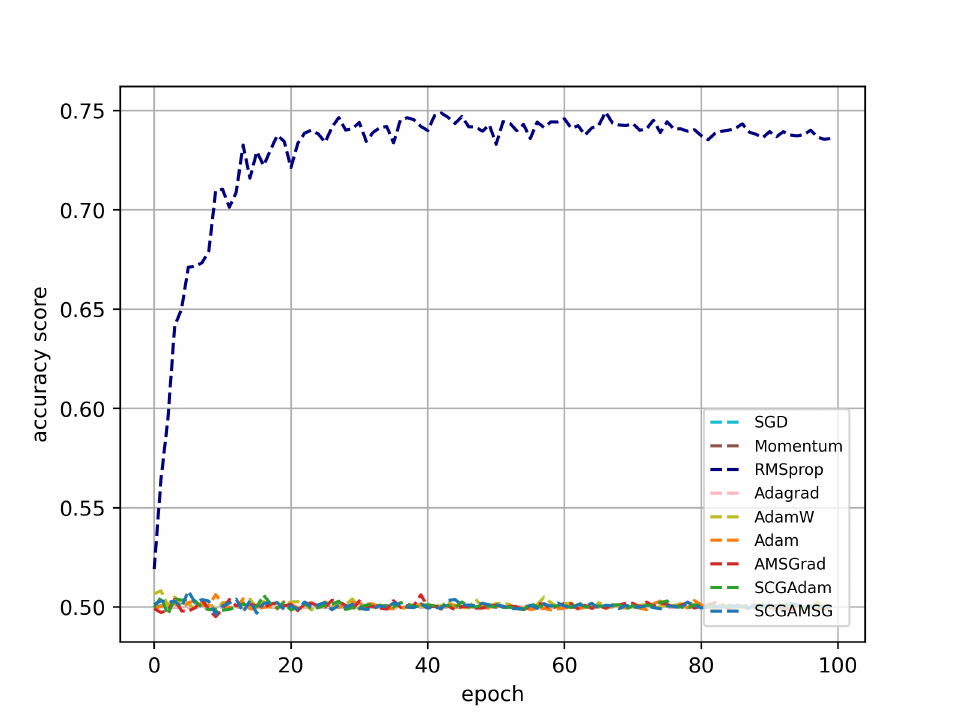}
		\subcaption{}
		\label{fig:2_d_a_t}
 	\end{minipage}
 	\end{tabular}
	\vspace*{-10pt}
\caption{Results of Algorithm \ref{algo:1} with diminishing learning rates on the IMDb dataset: (a) training loss function value, (b) training classification accuracy score, and (c) test classification accuracy score.}
\label{fig:2_d}
\end{figure*}

\begin{figure*}[h] 
 	\begin{tabular}{cc}
 	\begin{minipage}{0.5\linewidth}
 		\centering
 		\includegraphics[width=0.85\textwidth]{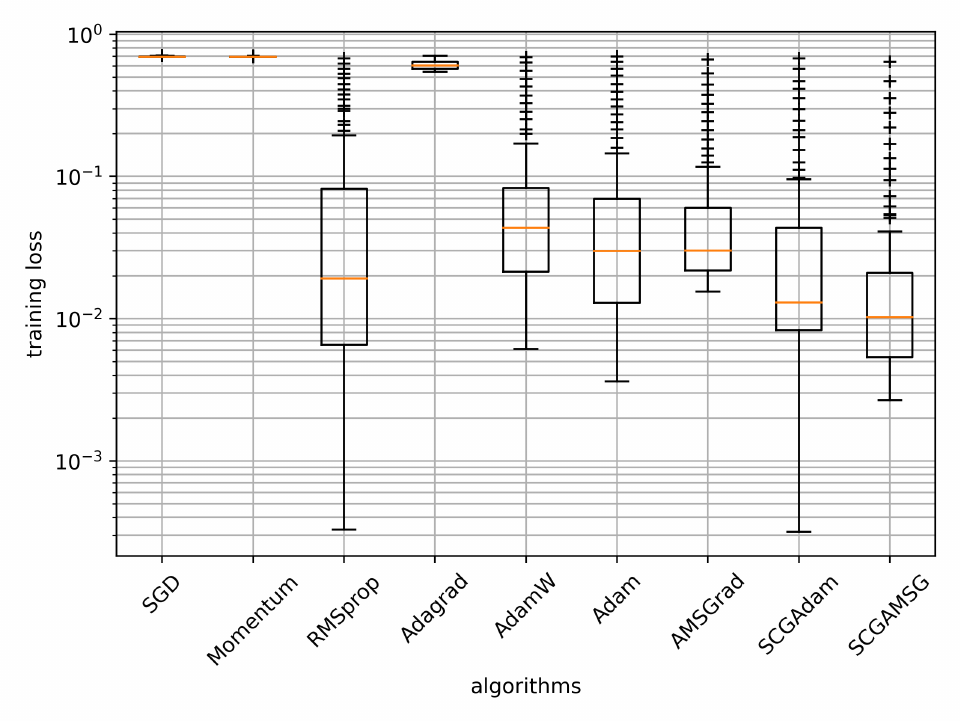}
		\subcaption{}
		\label{fig:6_c_f}
 	\end{minipage}
 	\begin{minipage}{0.5\linewidth}
 		\centering
 		\includegraphics[width=0.85\textwidth]{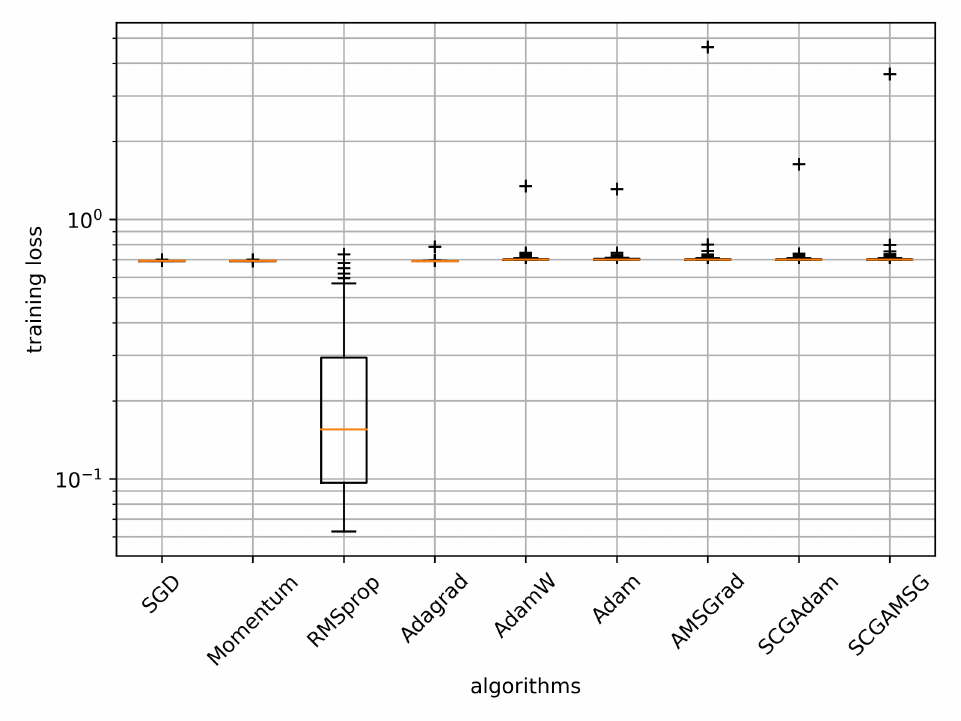}
		\subcaption{}
		\label{fig:6_d_f}
 	\end{minipage}
 	\end{tabular}
\caption{Box plots of training loss function values for Algorithm \ref{algo:1} on the IMDb dataset: (a) constant learning rates and (b) diminishing learning rates.}
\label{fig:6_0}
\end{figure*}

\begin{figure*}[h] 
 	\begin{tabular}{cc}
 	\begin{minipage}{0.5\linewidth}
 		\centering
 		\includegraphics[width=0.85\textwidth]{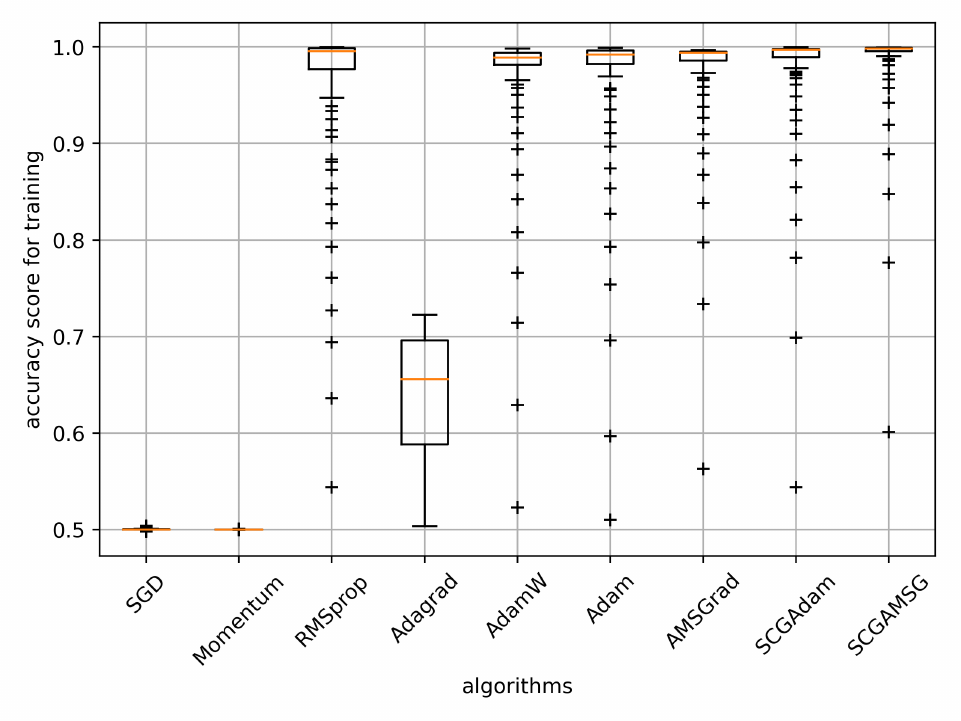}
		\subcaption{}
		\label{fig:6_c}
 	\end{minipage}
 	\begin{minipage}{0.5\linewidth}
 		\centering
 		\includegraphics[width=0.85\textwidth]{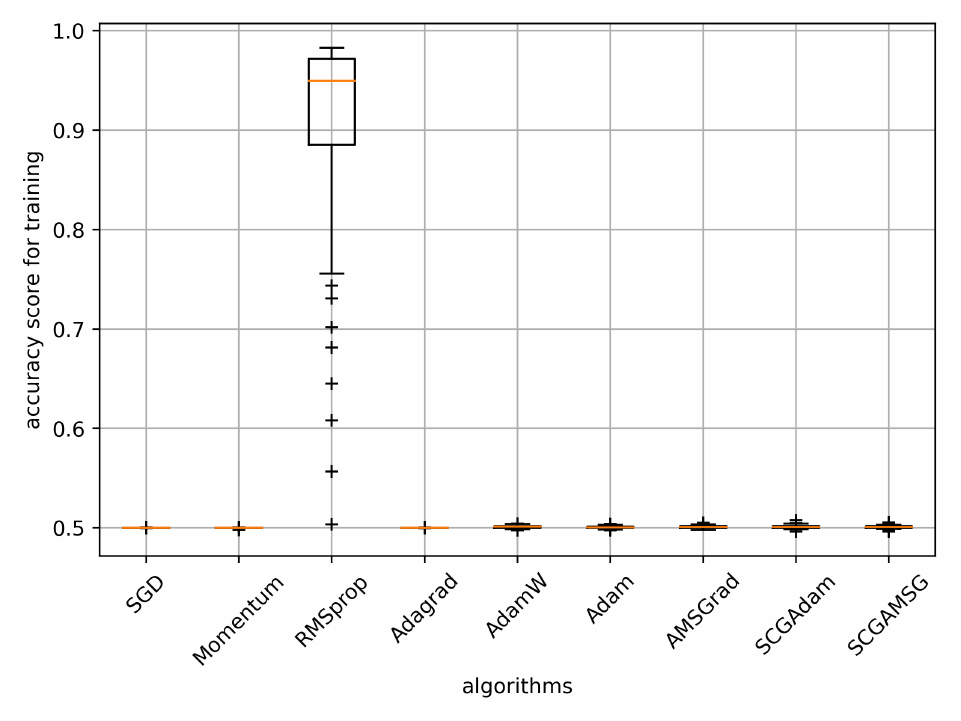}
		\subcaption{}
		\label{fig:6_d}
 	\end{minipage}
 	\end{tabular}
\caption{Box plots of training classification accuracy score for Algorithm \ref{algo:1} on the IMDb dataset: (a) constant learning rates and (b) diminishing learning rates.}
\label{fig:6}
\end{figure*}

\begin{figure*}[h] 
 	\begin{tabular}{cc}
 	\begin{minipage}{0.5\linewidth}
 		\centering
 		\includegraphics[width=0.85\textwidth]{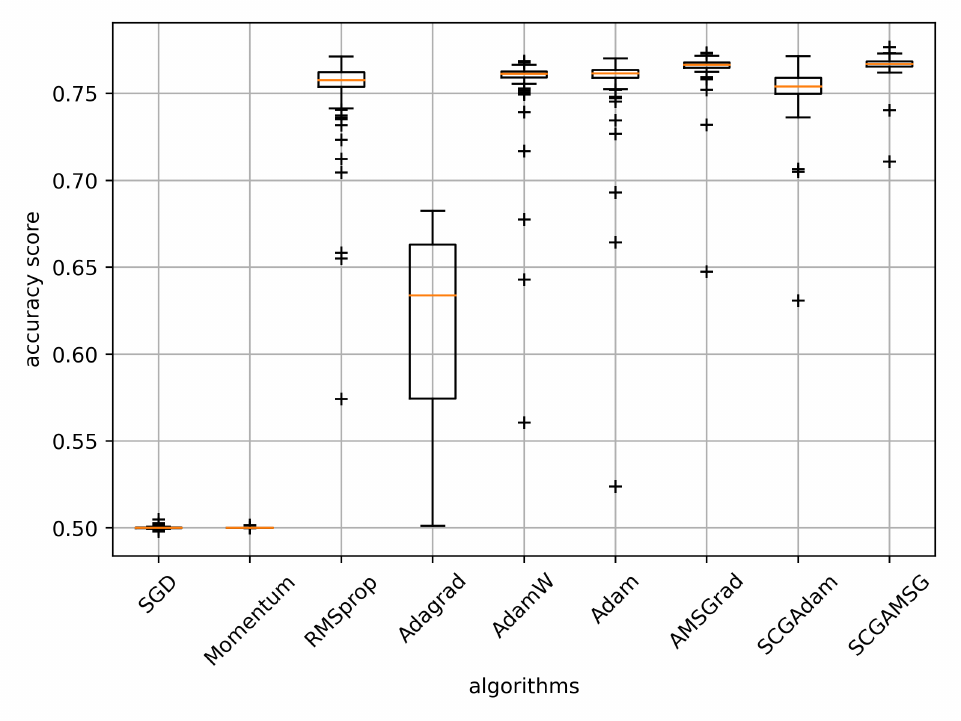}
		\subcaption{}
		\label{fig:7_c}
 	\end{minipage}
 	\begin{minipage}{0.5\linewidth}
 		\centering
 		\includegraphics[width=0.85\textwidth]{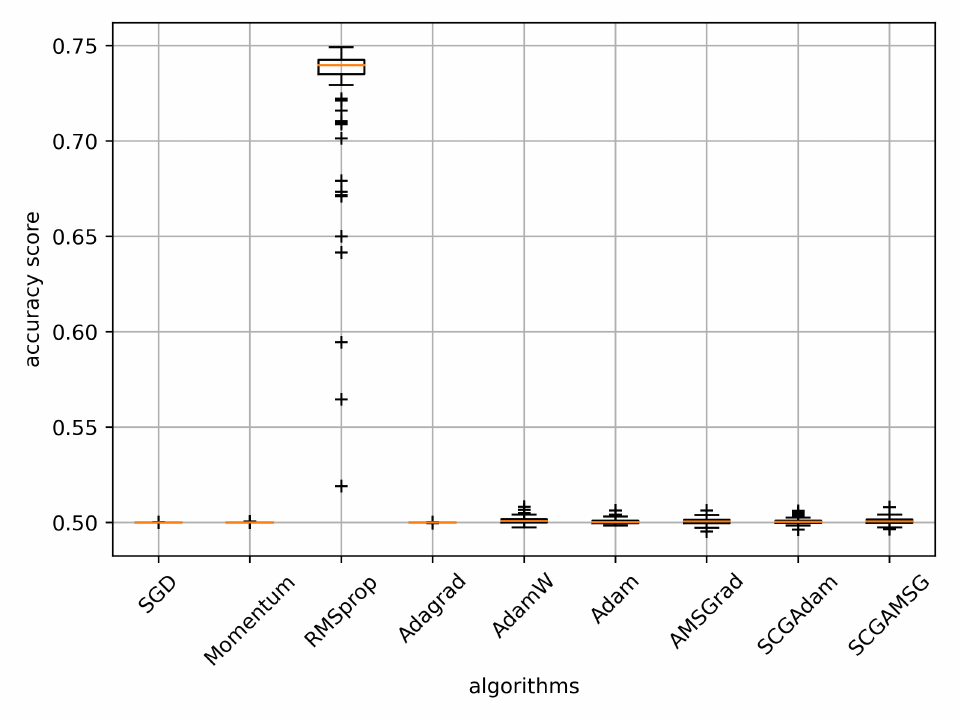}
		\subcaption{}
		\label{fig:7_d}
 	\end{minipage}
 	\end{tabular}
\caption{Box plots of test classification accuracy score for Algorithm \ref{algo:1} on the IMDb dataset: (a) constant learning rates and (b) diminishing learning rates.}
\label{fig:7}
\end{figure*}

\clearpage
\subsection{Image generation with GANs}
\label{subsec:4.4}
Generative adversarial networks (GANs) consist of two deep neural networks, a generator and a discriminator. The training of GANs is generally unstable because it is performed simultaneously by the generator and the discriminator in a minimax game. Since SGD and Momentum often cause mode collapse and their convergence is slow, most of the previous studies \citep{ian2016nips, NIPS2017_892c3b1c, Tim2016Imp, adab} have used adaptive methods such as Adam, RMSProp, and AdaBelief. Therefore, training of GANs is suitable for testing the stability of the optimizer. We trained a deep convolutional GAN (DCGAN) \citep{radford2016unsupervised} on the LSUN-Bedroom dataset \citep{lsun}, SNGAN \citep{Miyato2018Spe} on the CIFAR10 dataset, and Wasserstein GAN with gradient penalty (WGAN-GP) \citep{NIPS2017_892c3b1c} on the Tiny ImageNet dataset \citep{Yang2015Tin} with some adaptive optimizers. We used the widely used Fr\'echet inception distance (FID) \citep{Heusel2017} to evaluate the quality of the generated images. 

For training DCGAN, five 250,000 steps trainings were performed with each optimizer under the optimal hyperparameter settings, and the results are summarized in Figure \ref{fig:16}. Figure \ref{fig:17} compares the box plots of mean FID scores for the last 50,000 steps, where the decrease in FID score has subsided. According to Figure \ref{fig:17}, SCGAMSGrad achieves the lowest FID score on average, with SCGAdam being second best to AdaBelief. See Table \ref{table:hyper} in Appendix \ref{a:hyper} for the hyperparameters used in the experiments.

\begin{figure}[htbp]
\centering
\begin{minipage}[t]{0.8\linewidth}
\includegraphics[width=1\textwidth]{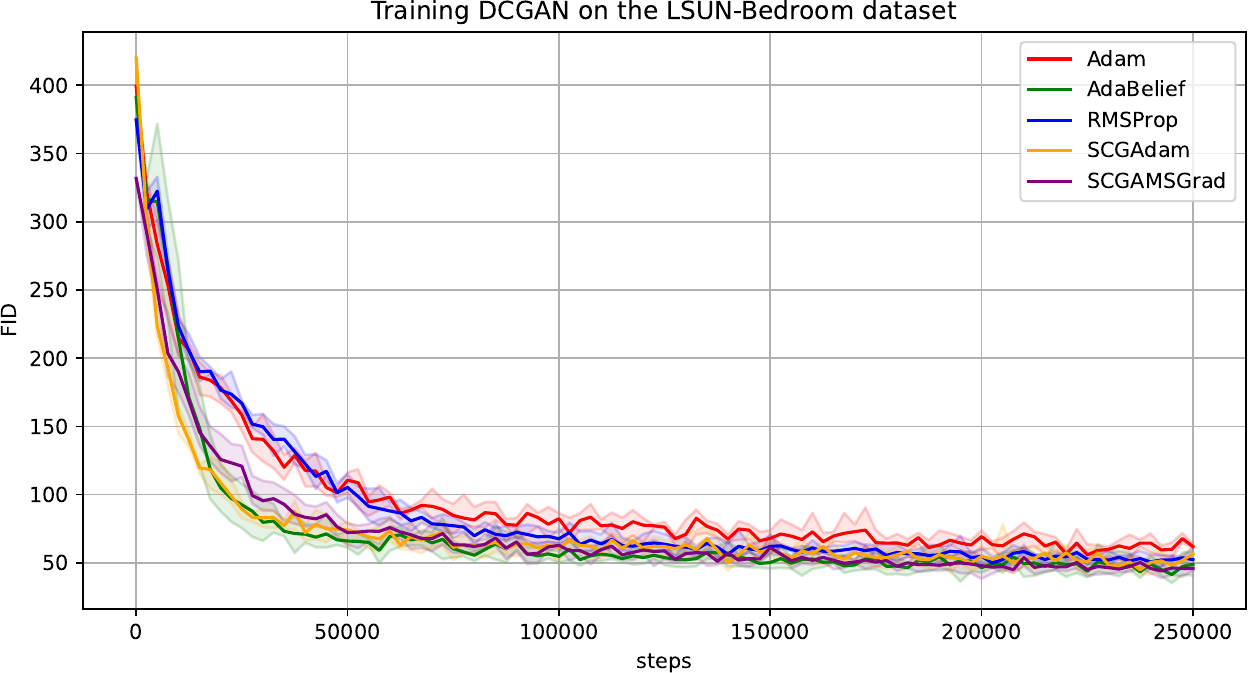}
\end{minipage}
\caption{Mean FID (solid line) bounded by the maximum and the minimum over 5 runs (shaded area) for training DCGAN on the LSUN-Bedroom dataset for five optimizers. For all runs, the batch size is 64 and the learning rate combinations are determined with a grid search (see Figure \ref{fig:18}).}
\label{fig:16}
\end{figure}

\begin{figure}[htbp]
\centering
\begin{minipage}[t]{0.7\linewidth}
\includegraphics[width=1\textwidth]{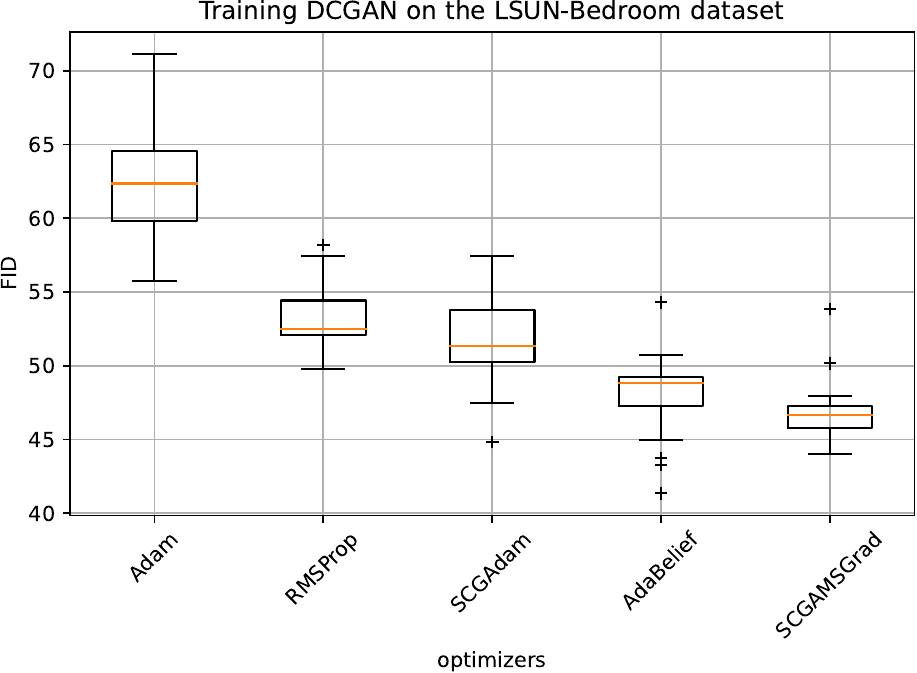}
\end{minipage}
\caption{Box plots of mean FID scores for the last 50,000 steps in training DCGAN on the LSUN-Bedroom dataset for five optimizers.}
\label{fig:17}
\end{figure}

For training SNGAN, five 150,000 steps trainings were performed with each optimizer under the optimal hyperparameter settings; the results are summarized in Figure \ref{fig:19}. Moreover, Figure \ref{fig:20} compares the box plots of mean FID scores for the last 30,000 steps, where the decrease in FID score has subsided. According to this figure, SCGAMSGrad achieves the lowest FID score on average. See Table \ref{table:hyper-sn} in Appendix \ref{a:hyper} for the hyperparameters used in the experiments.

\begin{figure}[h]
\centering
\begin{minipage}[t]{0.8\linewidth}
\includegraphics[width=1\textwidth]{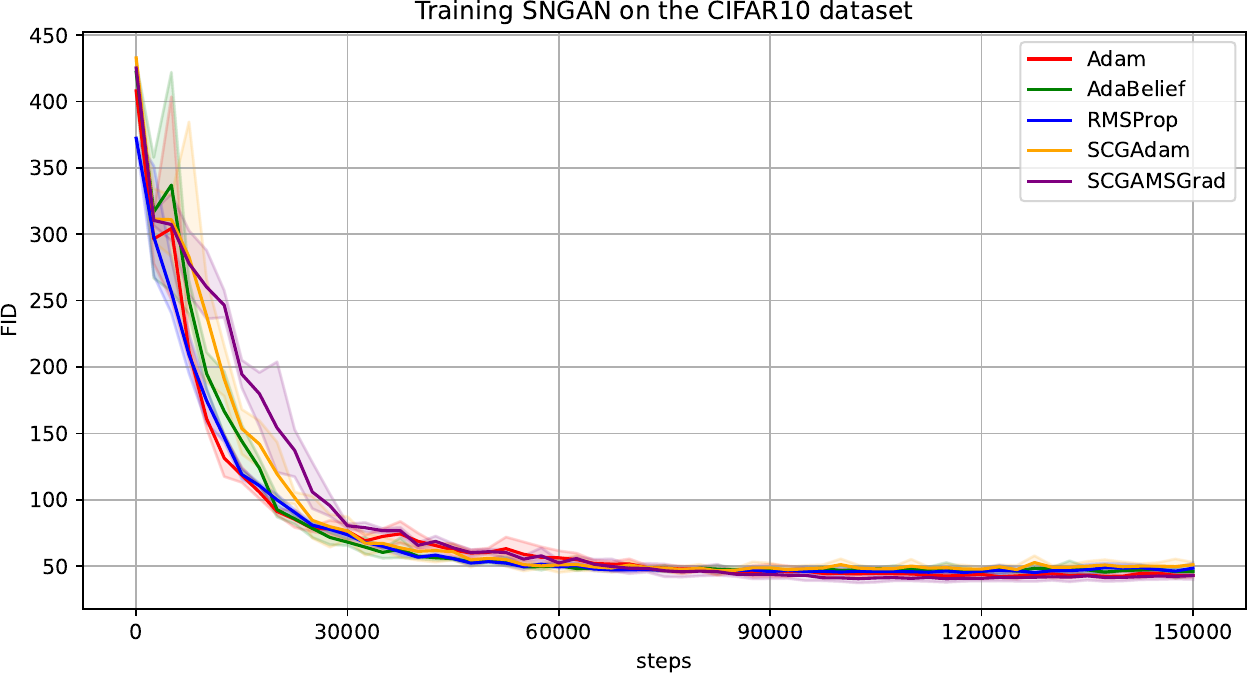}
\end{minipage}
\caption{Mean FID (solid line) bounded by the maximum and the minimum over 5 runs (shaded area) for training SNGAN on the CIFAR10 dataset for five optimizers. For all runs, the batch size is 64 and the learning rate combinations are determined with a grid search (see Figure \ref{fig:29}).}
\label{fig:19}
\end{figure}

\begin{figure}[htbp]
\centering
\begin{minipage}[t]{0.7\linewidth}
\includegraphics[width=1\textwidth]{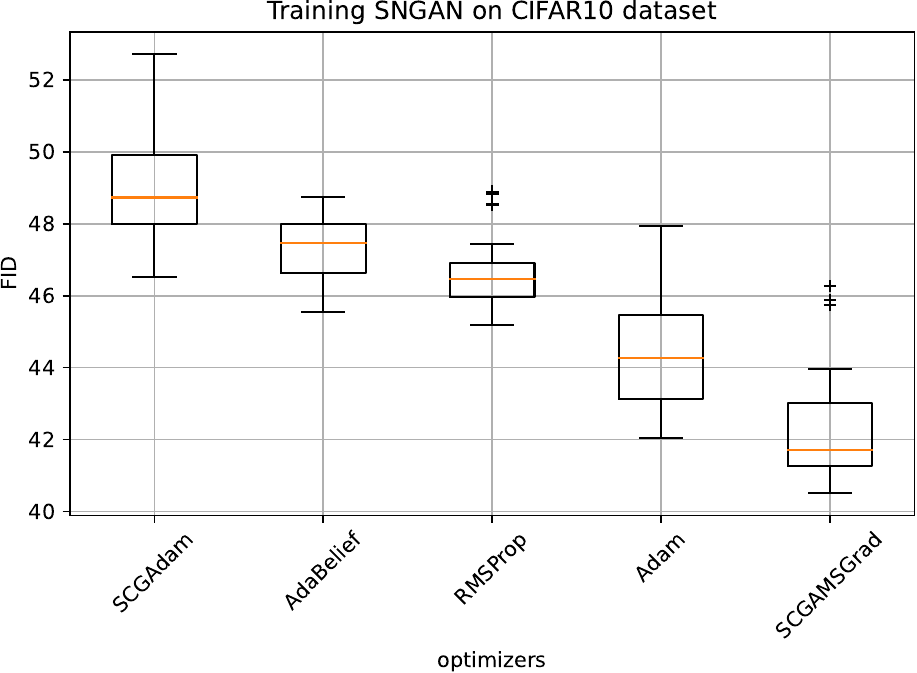}
\end{minipage}
\caption{Box plots of mean FID scores for the last 30,000 steps in training SNGAN on the CIFAR10 dataset for five optimizers.}
\label{fig:20}
\end{figure}

For training WGAN-GP, five 400,000 steps trainings were performed with each optimizer under the optimal hyperparameter settings; the results are summarized in Figure \ref{fig:21}. Moreover, Figure \ref{fig:22} compares the box plots of mean FID scores for the last 80,000 steps, where the decrease in FID score has subsided. According to this figure, SCGAdam achieves the lowest FID score on average. See Table \ref{table:hyper-w} in Appendix \ref{a:hyper} for the hyperparameters used in the experiments.

\begin{figure}[htbp]
\centering
\begin{minipage}[t]{0.8\linewidth}
\includegraphics[width=1\textwidth]{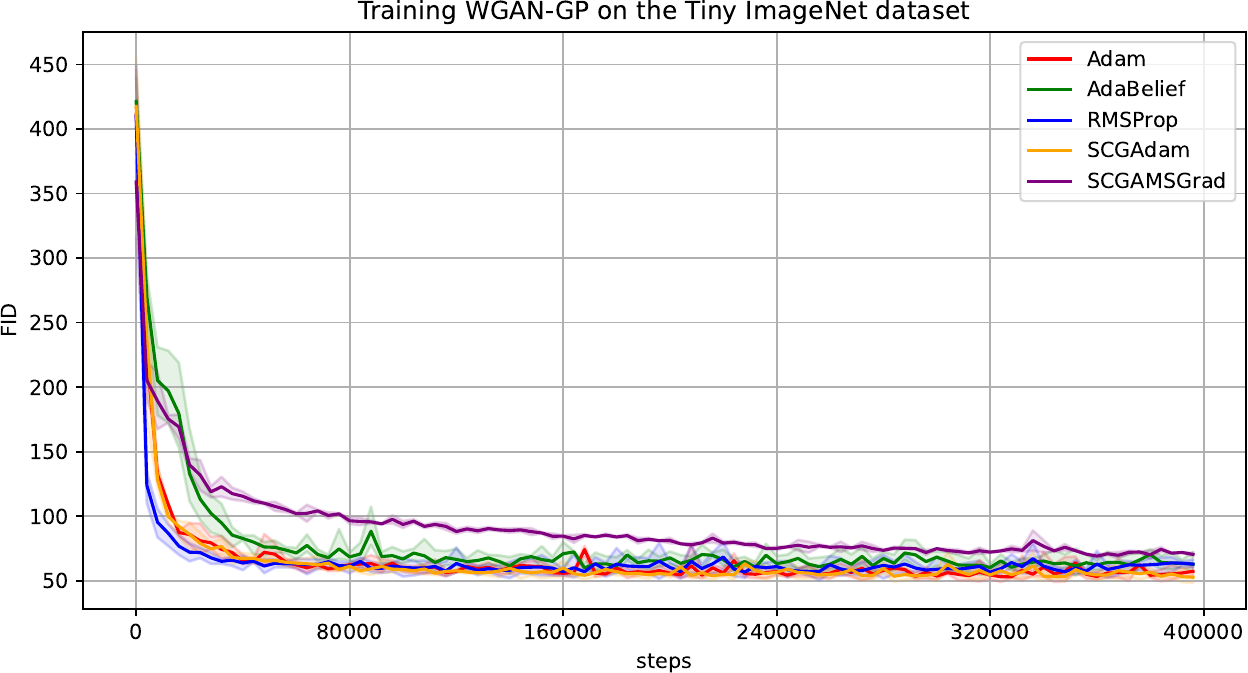}
\end{minipage}
\caption{Mean FID (solid line) bounded by the maximum and the minimum over 5 runs (shaded area) for training WGAN-GP on the Tiny ImageNet dataset for five optimizers. For all runs, the batch size is 64 and the learning rate combinations are determined with a grid search (see Figure \ref{fig:30}).}
\label{fig:21}
\end{figure}

\begin{figure}[htbp]
\centering
\begin{minipage}[t]{0.7\linewidth}
\includegraphics[width=1\textwidth]{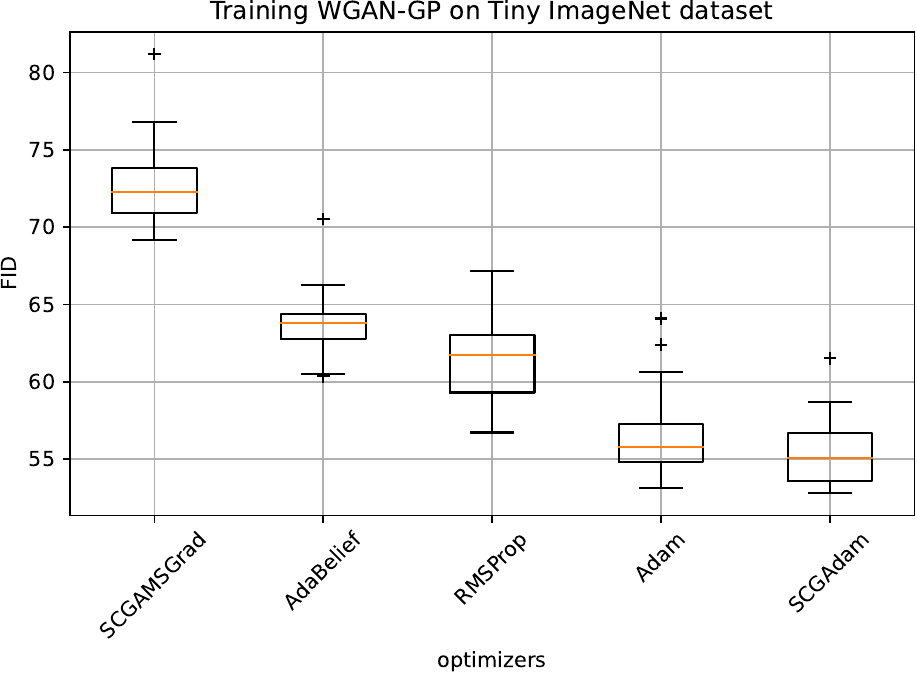}
\end{minipage}
\caption{Box plots of mean FID scores for the last 80,000 steps in training WGAN-GP on the Tiny ImageNet dataset for five optimizers.}
\label{fig:22}
\end{figure}

All of the results show that the proposed method successfully trains unstable GANs as well as other adaptive methods. They also show that the proposed method is able to achieve the lowest FID for certain datasets and models.

\section{Conclusion}
\label{sec:5}
We proposed a scaled conjugate gradient method for nonconvex optimization in deep learning. With constant learning rates, this method has approximately $\mathcal{O}(1/N)$ convergence, and with diminishing learning rates, it has $\mathcal{O}(1/\sqrt{N})$ convergence, which is an improvement on the previous results reported in \citep{electronics9111809}. We also showed that the proposed method is suitable for practical applications, in particular image and text classification tasks. Using constant learning rates, it minimized the training loss functions faster than other methods. Furthermore, our experimental results show that the proposed method can achieve lowest FID score among adaptive optimizers in training several GANs.

\acks{We are sincerely grateful to Action Editor Prateek Jain and the three anonymous reviewers for helping us improve the original manuscript. This research is partly supported by the computational resources of the DGX A100 named TAIHO at Meiji University. This work was supported by the Japan Society for the Promotion of Science (JSPS) KAKENHI, Grant Numbers 21K11773 and 24K14846.}


\newpage

\appendix
\section{Appendix A}
\label{appen:1}
Suppose that $H \in \mathbb{S}_{++}^d$. The $H$-inner product of $\mathbb{R}^d$ is defined for all $\bm{x}, \bm{y} \in \mathbb{R}^d$ by $\langle \bm{x},\bm{y} \rangle_H := \langle \bm{x}, H \bm{y} \rangle$, and the $H$-norm is defined by $\|\bm{x}\|_H := \sqrt{\langle \bm{x}, H \bm{x} \rangle}$.

\subsection{Lemmas}
\begin{lemma}\label{lem:1}
Under (A1) and (A2), for all $\bm{x} \in \mathbb{R}^d$ and all $n \in \mathbb{N}$,
\begin{align*}
&\mathbb{E}\left[\left\| \bm{x}_{n+1} - \bm{x} \right\|_{\mathsf{H}_n}^2\right]\\
&\leq
\mathbb{E}\left[\left\| \bm{x}_{n} - \bm{x} \right\|_{\mathsf{H}_n}^2\right]
+ 2 \alpha_n 
 \bigg[
 \frac{\beta_n}{\tilde{\zeta}_n} 
 {\mathbb{E} \left[\left\langle \bm{x} - \bm{x}_n, \bm{m}_{n-1} \right\rangle \right]}
 + \frac{\tilde{\beta}_n}{\tilde{\zeta}_n} 
\bigg\{ 
(1 + \gamma_n) \mathbb{E} \left[ \left\langle \bm{x} - \bm{x}_n, \nabla f(\bm{x}_n) \right\rangle \right]\\
&\quad + \delta_n \mathbb{E} \left[ \left\langle \bm{x}_n - \bm{x}, \bm{\mathsf{G}}_{n-1} \right\rangle \right]
\bigg\} \bigg]
 + \alpha_n^2 \mathbb{E} \left[ \left\|\bm{\mathsf{d}}_n \right\|_{\mathsf{H}_n}^2 \right], 
\end{align*}
where ${\tilde{\zeta}}_n := 1 - {\zeta}^{n+1}$ and $\tilde{\beta}_n := 1 - \beta_n$.
\end{lemma}

\begin{proof}
Let $\bm{x}\in \mathbb{R}^d$ and $n\in \mathbb{N}$. We have that 
\begin{align}\label{ineq:001}
\begin{split}
\left\|\bm{x}_{n+1} - \bm{x} \right\|_{\mathsf{H}_n}^2
=
\left\| \bm{x}_n - \bm{x} \right\|_{\mathsf{H}_n}^2 
+ 2 \alpha_n \left\langle \bm{x}_n - \bm{x}, \bm{\mathsf{d}}_n \right\rangle_{\mathsf{H}_n}
+ \alpha_n^2 \left\|\bm{\mathsf{d}}_n \right\|_{\mathsf{H}_n}^2. 
\end{split}
\end{align}
The definition of the $\mathsf{H}_n$-inner product and the definition of Algorithm 1 ensure that 
\begin{align*}
\left\langle \bm{x}_n - \bm{x}, \bm{\mathsf{d}}_n \right\rangle_{\mathsf{H}_n}
=
\frac{1}{{\tilde{\zeta}}_n}\left\langle \bm{x} - \bm{x}_n, \bm{m}_n \right\rangle
=
\frac{\beta_n}{{\tilde{\zeta}}_n} \left\langle \bm{x} - \bm{x}_n, \bm{m}_{n-1} \right\rangle 
+
\frac{\tilde{\beta}_n}{{\tilde{\zeta}}_n} \left\langle \bm{x} - \bm{x}_n, \bm{\mathsf{G}}_n \right\rangle.
\end{align*}
The definition of $\bm{\mathsf{G}}_n$ implies that 
\begin{align*}
\left\langle \bm{x} - \bm{x}_n, \bm{\mathsf{G}}_n \right\rangle
= 
(1 + \gamma_n) \left\langle \bm{x} - \bm{x}_n, \mathsf{G}(\bm{x}_n, \xi_n) \right\rangle
+ \delta_n \left\langle \bm{x}_n - \bm{x}, \bm{\mathsf{G}}_{n-1}\right\rangle.
\end{align*}
Define the history of process $\xi_0,\xi_1,\ldots$ to time step $n$ by $\xi_{[n]} = (\xi_0,\xi_1,\ldots,\xi_n)$. The condition $\bm{x}_n = \bm{x}_n(\xi_{[n-1]})$ $(n\in \mathbb{N})$ and (A2) guarantee that
\begin{align*}
\mathbb{E} \left[\left\langle \bm{x} - \bm{x}_n, \mathsf{G}(\bm{x}_n,\xi_n) \right\rangle \right]
&=
\mathbb{E} 
 \left[ 
 \mathbb{E} \left[\left\langle \bm{x} - \bm{x}_n, \mathsf{G}(\bm{x}_n,\xi_n) \right\rangle | \xi_{[n-1]} \right]
 \right]\\
&= 
\mathbb{E} 
 \left[ 
 \left\langle \bm{x} - \bm{x}_n, \mathbb{E} \left[ \mathsf{G}(\bm{x}_n,\xi_n) | \xi_{[n-1]} \right] \right\rangle
 \right]\\
&=
\mathbb{E} \left[\left\langle \bm{x} - \bm{x}_n, {\nabla f} (\bm{x}_n) \right\rangle \right].
\end{align*}
Accordingly, taking the expectation of \eqref{ineq:001} gives us the following lemma.
\end{proof}

\begin{lemma}\citep[Lemma 2]{iiduka2021}\label{lem:2}
Under (A3), for all $n\in\mathbb{N}$,
\begin{align*}
\mathbb{E}\left[\|\bm{m}_n\|^2 \right] \leq \tilde{M}^2 := \max \left\{ \|\bm{m}_{-1}\|^2, M^2 \right\}. 
\end{align*}
Additionally, under (A4), for all $n\in\mathbb{N}$,
\begin{align*}
\mathbb{E}\left[\|\bm{\mathsf{d}}_n\|_{\mathsf{H}_n}^2 \right] \leq \frac{\tilde{B}^2 \tilde{M}^2}{\tilde{\zeta}^2}, 
\end{align*}
where $\tilde{\zeta} := 1-{\zeta}$ and $\tilde{B} := \sup\{ {\max_{i=1,2,\ldots,d} h_{n,i}^{-1/2}} \colon n\in\mathbb{N}\} < + \infty$. 
\end{lemma}

\begin{lemma}\label{lem:3}
Under (A3), for all $n\in\mathbb{N}$,
\begin{align*}
\mathbb{E}\left[\|\bm{\mathsf{G}}_n\|^2 \right] 
\leq 16 \hat{M}^2, 
\end{align*}
where $\hat{M} := \max \{M^2, \|\bm{\mathsf{G}}_{-1} \|^2 \}$.
\end{lemma}

\begin{proof} 
From the definition of $\bm{\mathsf{G}}_n$, we find that, for all $n\in\mathbb{N}$,
\begin{align*}
\mathbb{E}\left[ \|\bm{\mathsf{G}}_n\|^2 \right] 
&\leq 2 (1 + \gamma_n)^2 \mathbb{E}\left[ \|\mathsf{G}(\bm{x}_n,\xi_n) \|^2 \right] 
+ 2 \delta_n^2 \mathbb{E}\left[ \|\bm{\mathsf{G}}_{n-1}\|^2 \right]\\
&\leq 8 M^2 + 2 \delta_n^2 \mathbb{E}\left[ \|\bm{\mathsf{G}}_{n-1}\|^2 \right],
\end{align*}
where $\gamma_n \leq 1$ and (A3) have been used. In the case of $n = 0$, we have that $\mathbb{E}[ \|\bm{\mathsf{G}}_0\|^2 ] \leq 8 \hat{M}^2 + 2 (1/2)^2 \hat{M}^2 \leq 16 \hat{M}^2$, where $\delta_n \leq 1/2$. Assume that $\mathbb{E}[ \|\bm{\mathsf{G}}_n \|^2 ] \leq 16 \hat{M}^2$ for some $n$. Then, we have 
\begin{align*}
\mathbb{E}\left[ \|\bm{\mathsf{G}}_{n+1}\|^2 \right] 
&\leq 8 M^2 + 2 \delta_{n+1}^2 \mathbb{E}\left[ \|\bm{\mathsf{G}}_{n}\|^2 \right]\\
&\leq 8 \hat{M}^2 + 2 \cdot \frac{1}{4} \cdot 16 \hat{M}^2
= 16 \hat{M}^2.
\end{align*}
Induction thus shows that, for all $n\in\mathbb{N}$, $\mathbb{E}[ \|\bm{\mathsf{G}}_{n}\|^2 ] \leq 16 \hat{M}^2$.
\end{proof}

\subsection{Theorem}
\begin{theorem}\label{thm:3}
Let $(\kappa_n)_{n\in \mathbb{N}}$ be defined by $\kappa_n := \alpha_n \tilde{\beta}_n (1 + \gamma_n) /\tilde{\zeta}_n$ and $(\beta_n)_{n\in \mathbb{N}}$ satisfy $\kappa_{n+1} \leq \kappa_n$ ($n\in\mathbb{N}$) and $\limsup_{n\to + \infty} \beta_n < 1$. Define $V_n (\bm{x}) := \mathbb{E}\left[ \langle \bm{x}_n - \bm{x}, \nabla f (\bm{x}_n) \rangle \right]$ for all $\bm{x} \in \mathbb{R}^d$ and all $n\in \mathbb{N}$. Under (A1)--(A6), for $\bm{x} \in \mathbb{R}^d$ and $n\geq 1$, 
\begin{align*}
\frac{1}{n} \sum_{k=1}^n V_k (\bm{x})
\leq 
\frac{D \sum_{i=1}^d 
B_{i}}{2 \tilde{b} \alpha_n n}
+
\frac{\tilde{B}^2 \tilde{M}^2}{2 \tilde{b}\tilde{\zeta}^2 n} \sum_{k=1}^n \alpha_k
+
\frac{\tilde{M}\sqrt{Dd}}{\tilde{b} n} \sum_{k=1}^n \beta_k
+
\frac{4 \hat{M}\sqrt{Dd}}{n} \sum_{k=1}^n \frac{\delta_k}{1 + \gamma_k}, 
\end{align*}
where $\tilde{b} := 1 - b$, $\tilde{\zeta} := 1-\zeta$, $(\beta_n)_{n\in\mathbb{N}} \subset (0,b] \subset (0,1)$, $\tilde{M}$, $\hat{M}$, and $\tilde{B}$ are defined in Lemmas \ref{lem:2} and \ref{lem:3}, and $D$ and $B_i$ are defined in Assumption \ref{assum:1}.
\end{theorem}

\begin{proof}
Fix $\bm{x}\in X$ arbitrarily. Lemma \ref{lem:1} guarantees that, for all $n \geq 1$,
\begin{align}\label{key}
\begin{split}
{\sum_{k=1}^n V_k (\bm{x})} 
&\leq
\frac{1}{2} \underbrace{\sum_{k=1}^n \frac{1}{\kappa_k}
\left\{ 
\mathbb{E}\left[\left\| \bm{x}_{k} - \bm{x} \right\|_{\mathsf{H}_k}^2\right]
-
\mathbb{E}\left[\left\| \bm{x}_{k+1} - \bm{x} \right\|_{\mathsf{H}_k}^2\right]
\right\}}_{K_n}\\
&\quad + \underbrace{
\sum_{k=1}^n \frac{\beta_k}{\tilde{\beta}_k (1+\gamma_k)} \mathbb{E} \left[ \left\langle \bm{x} - \bm{x}_k,\bm{m}_{k-1} \right\rangle \right]}_{B_n} \\
&\quad + \underbrace{
\sum_{k=1}^n \frac{\delta_k}{(1+\gamma_k)} \mathbb{E} \left[ \left\langle \bm{x}_k - \bm{x},\bm{\mathsf{G}}_{k-1} \right\rangle \right]}_{\Delta_n}
+ \frac{1}{2} \underbrace{\sum_{k=1}^n \frac{\alpha_k \tilde{\zeta}_k}{\tilde{\beta}_k} \mathbb{E} \left[ \left\|\bm{\mathsf{d}}_k \right\|_{\mathsf{H}_k}^2 \right]}_{A_n}.
\end{split}
\end{align} 
From the definition of $K_n$ and $\mathbb{E} [ \| \bm{x}_{n+1} - \bm{x} \|_{\mathsf{H}_{n}}^2]/\kappa_n \geq 0$, 
\begin{align}\label{LAM}
K_n
\leq
\frac{\mathbb{E}\left[\left\| \bm{x}_{1} - \bm{x} \right\|_{\mathsf{H}_{1}}^2\right]}{\kappa_1}
+
\underbrace{
\sum_{k=2}^n \left\{
\frac{\mathbb{E}\left[\left\| \bm{x}_{k} - \bm{x} \right\|_{\mathsf{H}_{k}}^2\right]}{\kappa_k}
-
\frac{\mathbb{E}\left[\left\| \bm{x}_{k} - \bm{x} \right\|_{\mathsf{H}_{k-1}}^2\right]}{\kappa_{k-1}} 
\right\}
}_{\tilde{K}_n}.
\end{align}
Since $\overline{\mathsf{H}}_k \in \mathbb{S}_{++}^d$ exists such that $\mathsf{H}_k = \overline{\mathsf{H}}_k^2$, we have $\|\bm{x}\|_{\mathsf{H}_k}^2 = \| \overline{\mathsf{H}}_k \bm{x} \|^2$ for all $\bm{x}\in\mathbb{R}^d$. Accordingly, we have 
\begin{align*}
\tilde{K}_n 
=
\mathbb{E} \left[ 
\sum_{k=2}^n 
\left\{
\frac{\left\| \overline{\mathsf{H}}_{k} (\bm{x}_{k} - \bm{x}) \right\|^2}{\kappa_k}
-
\frac{\left\| \overline{\mathsf{H}}_{k-1} (\bm{x}_{k} - \bm{x}) \right\|^2}{\kappa_{k-1}}
\right\}
\right].
\end{align*}
From $\mathsf{H}_k = \mathsf{diag}(h_{k,i})$, for all $k\in\mathbb{N}$ and all $\bm{x} := (x_i) \in \mathbb{R}^d$,
\begin{align}\label{HK}
\overline{\mathsf{H}}_{k} = \mathsf{diag}\left(\sqrt{h_{k,i}}\right) 
\text{ and }
\left\| \overline{\mathsf{H}}_{k} \bm{x} \right\|^2
=
\sum_{i=1}^d h_{k,i} x_i^2.
\end{align}
Hence, for all $n\geq 2$,
\begin{align*}
\tilde{K}_n 
= 
\mathbb{E} \left[ 
\sum_{k=2}^n
\sum_{i=1}^d 
\left(
\frac{h_{k,i}}{\kappa_k}
-
\frac{h_{k-1,i}}{\kappa_{k-1}}
\right)
(x_{k,i} - x_i)^2
\right].
\end{align*}
From $\kappa_k \leq \kappa_{k-1}$ $(k\geq 1)$ and (A4), we have $h_{k,i}/\kappa_k - h_{k-1,i}/\kappa_{k-1} \geq 0$ $(k \geq 1, i \in [d])$. Accordingly, (A6) implies that, for all $n\geq 2$,
\begin{align*}
\tilde{K}_n
&\leq
D
\mathbb{E} \left[ 
\sum_{k=2}^n
\sum_{i=1}^d 
\left(
\frac{h_{k,i}}{\kappa_k}
-
\frac{h_{k-1,i}}{\kappa_{k-1}}
\right)
\right]
= 
D
\mathbb{E} \left[ 
\sum_{i=1}^d
\left(
\frac{h_{n,i}}{\kappa_n}
-
\frac{h_{1,i}}{\kappa_{1}}
\right)
\right].
\end{align*}
Therefore, \eqref{LAM}, $\mathbb{E} [\| \bm{x}_{1} - \bm{x}\|_{\mathsf{H}_{1}}^2]/\kappa_1 \leq D \mathbb{E} [ \sum_{i=1}^d h_{1,i}/\kappa_1]$, and (A5) imply, for all $n\in\mathbb{N}$,
\begin{align*}
K_n 
\leq
D \mathbb{E} \left[ \sum_{i=1}^d \frac{h_{1,i}}{\kappa_1} \right]
+
D
\mathbb{E} \left[
\sum_{i=1}^d 
\left(
\frac{h_{n,i}}{\kappa_n}
-
\frac{h_{1,i}}{\kappa_{1}}
\right)
\right]
\leq
\frac{D}{\kappa_n}
\mathbb{E} \left[
\sum_{i=1}^d 
h_{n,i}
\right]
\leq 
\frac{D}{\kappa_n}
\sum_{i=1}^d 
B_{i}.
\end{align*}
Moreover, from $\gamma_n, \zeta \geq 0$ and $\tilde{b} := 1 -b$,
\begin{align*}
\kappa_n := \frac{\alpha_n \tilde{\beta}_n (1+\gamma_n)}{\tilde{\zeta}_n} 
\geq 
\tilde{b} \alpha_n.
\end{align*}
Hence, 
\begin{align}\label{L} 
K_n 
\leq 
\frac{D \sum_{i=1}^d 
B_{i}}{\tilde{b} \alpha_n}.
\end{align}
The Cauchy-Schwarz inequality and (A6) imply that
\begin{align}\label{B}
\begin{split}
B_n 
\leq \sum_{k=1}^n \frac{\beta_k}{\tilde{\beta}_k (1+\gamma_k)} \mathbb{E} \left[ \left\| \bm{x} - \bm{x}_k \right\|\left\|\bm{m}_{k-1} \right\| \right]
\leq 
\frac{\sqrt{Dd}}{\tilde{b}} \sum_{k=1}^n \beta_k \mathbb{E} \left[
\left\|\bm{m}_{k-1} \right\|
\right]
\leq
\frac{\tilde{M}\sqrt{Dd}}{\tilde{b}} \sum_{k=1}^n \beta_k,
\end{split}
\end{align}
where the final inequality comes from Lemma \ref{lem:2} and Jensen's inequality, i.e., for all $n\in\mathbb{N}$, $\mathbb{E}[\|\bm{m}_n\|] \leq \tilde{M}$. A discussion similar to the one showing \eqref{B} implies that
\begin{align}\label{Delta}
\begin{split}
\Delta_n 
\leq \sum_{k=1}^n \frac{\delta_k}{1+\gamma_k} \mathbb{E} \left[ \left\| \bm{x}_k - \bm{x} \right\|\left\|\bm{\mathsf{G}}_{k-1} \right\| \right]
\leq
4 \hat{M}\sqrt{Dd} \sum_{k=1}^n \frac{\delta_k}{1+\gamma_k},
\end{split}
\end{align}
where the final inequality comes from Lemma \ref{lem:3} and Jensen's inequality, i.e., for all $n\in\mathbb{N}$, $\mathbb{E}[\|\bm{\mathsf{G}}_n\|] \leq 16 \hat{M}$. Lemma \ref{lem:2} ensures that, for all $n\in\mathbb{N}$, 
\begin{align}\label{D}
A_n := \sum_{k=1}^n \frac{\alpha_k \tilde{\zeta}_k}{\tilde{\beta}_k} \mathbb{E} \left[ \left\|\bm{\mathsf{d}}_k \right\|_{\mathsf{H}_k}^2 \right] 
\leq 
\frac{\tilde{B}^2 \tilde{M}^2}{\tilde{b}\tilde{\zeta}^2} \sum_{k=1}^n \alpha_k.
\end{align}
Therefore, the assertion of Theorem \ref{thm:3} follows from \eqref{key} and \eqref{L}--\eqref{D}.
\end{proof}

\subsection{Proofs of Theorem \ref{thm:1} and Proposition \ref{prop:1}}
\begin{proof}
[Proof of Theorem \ref{thm:1}] Let $\bm{x} \in \mathbb{R}^d$, $\alpha_n := \alpha \in (0,1)$, $\beta_n := \beta = b \in (0,1)$, $\gamma_n := \gamma \in [0, +\infty)$, and $\delta_n := \delta \in [0,1/2]$. Define $X_n := \mathbb{E}[\|\bm{x}_n - \bm{x}\|_{\mathsf{H}_n}^2]$. Lemmas \ref{lem:1}, \ref{lem:2}, and \ref{lem:3} imply that, for all $n\in\mathbb{N}$, 
\begin{align*}
X_{n+1} 
&\leq X_n + \left(X_{n+1} - \mathbb{E}\left[\|\bm{x}_{n+1} - \bm{x}\|_{\mathsf{H}_n}^2 \right] \right)\\
&\quad + \frac{2 \alpha \beta \tilde{M} \sqrt{Dd}}{\tilde{\zeta}}
- \frac{2 \alpha \tilde{b}\tilde{\gamma}}{1 - \zeta^{n+1}} V_n(\bm{x}) 
+ \frac{8 \alpha \tilde{b} \delta \hat{M} \sqrt{Dd}}{\tilde{\zeta}}
+ \frac{\alpha^2 \tilde{B}^2 \tilde{M}^2}{\tilde{\zeta}^2}.
\end{align*}
From $1 - \zeta^{n+1} \leq 1$ and $(X_{n+1} - X_n) \zeta^{n+1} \leq X_{n+1} \zeta^{n+1}$, we have, for all $n\in\mathbb{N}$, 
\begin{align}\label{zeta}
\begin{split}
X_{n+1} 
&\leq X_n + X_{n+1} \zeta^{n+1} + \left(X_{n+1} - \mathbb{E}\left[\|\bm{x}_{n+1} - \bm{x}\|_{\mathsf{H}_n}^2 \right] \right) \\
&\quad + \frac{2 \alpha \beta \tilde{M} \sqrt{Dd}}{\tilde{\zeta}}
- 2 \alpha \tilde{b} \tilde{\gamma} V_n(\bm{x}) 
+ \frac{8 \alpha \tilde{b} \delta \hat{M} \sqrt{Dd}}{\tilde{\zeta}} 
+ \frac{\alpha^2 \tilde{B}^2 \tilde{M}^2}{\tilde{\zeta}^2}.
\end{split}
\end{align}
We will show that, for all $\epsilon > 0$,
\begin{align}\label{Ineq:1}
\begin{split}
\liminf_{n\to +\infty} 
V_n(\bm{x})
\leq 
\frac{\alpha \tilde{B}^2 \tilde{M}^2}{2\tilde{b} \tilde{\gamma} \tilde{\zeta}^2}
+
\frac{\beta \tilde{M}\sqrt{Dd}}{\tilde{b}\tilde{\gamma}\tilde{\zeta}}
+
\frac{4 \delta \hat{M} \sqrt{Dd}}{\tilde{\gamma} \tilde{\zeta}}
+
\frac{Dd\epsilon}{2 \tilde{b} \tilde{\gamma}}
+ 
\epsilon.
\end{split}
\end{align}
If \eqref{Ineq:1} does not hold for all $\epsilon > 0$, then there exists $\epsilon_0 > 0$ such that 
\begin{align}\label{INF}
\begin{split}
\liminf_{n\to +\infty} V_n (\bm{x})
> 
\frac{\alpha \tilde{B}^2 \tilde{M}^2}{2\tilde{b} \tilde{\gamma}\tilde{\zeta}^2}
+
\frac{\beta \tilde{M}\sqrt{Dd}}{\tilde{b}\tilde{\gamma}\tilde{\zeta}}
+
\frac{4 \delta \hat{M} \sqrt{Dd}}{\tilde{\gamma}\tilde{\zeta}}
+
\frac{Dd\epsilon_0}{2 \tilde{b}\tilde{\gamma}}
+ 
\epsilon_0.
\end{split}
\end{align}
Assumptions (A4) and (A5) guarantee that there exists $n_0\in\mathbb{N}$ such that $n \in \mathbb{N}$ with $n\geq n_0$ implies 
\begin{align}\label{I}
\mathbb{E}\left[\sum_{i=1}^d (h_{n+1,i} - h_{n,i}) \right] 
\leq \frac{d \alpha \epsilon_0}{2},
\end{align}
which, together with \eqref{HK}, \eqref{I}, and (A6), implies that, for all $n \geq n_0$,
\begin{align}\label{IV}
X_{n+1}
-
\mathbb{E}\left[ \left\| \bm{x}_{n+1} - \bm{x} \right\|_{\mathsf{H}_{n}}^2 \right]
\leq
\frac{D d \alpha \epsilon_0}{2}.
\end{align}
The condition $\zeta \in [0,1)$ implies that there exists $n_1 \in \mathbb{N}$ such that, for all $n \geq n_1$, 
\begin{align}\label{I_I}
X_{n+1} \zeta^{n+1} 
\leq 
\frac{D d \alpha \epsilon_0}{2}.
\end{align}
Accordingly, \eqref{zeta}, \eqref{IV}, and \eqref{I_I} ensure that, for all $n \geq \max \{n_0,n_1\}$,
\begin{align}\label{zeta_1}
\begin{split}
X_{n+1} 
&\leq X_n + D d \alpha \epsilon_0 - 2 \alpha \tilde{b} V_n(\bm{x})\\
&\quad + \frac{2 \alpha \beta \tilde{M} \sqrt{Dd}}{\tilde{\zeta}} 
+ \frac{8 \alpha \tilde{b} \delta \hat{M} \sqrt{Dd}}{\tilde{\zeta}} 
+ \frac{\alpha^2 \tilde{B}^2 \tilde{M}^2}{\tilde{\zeta}^2}.
\end{split}
\end{align}
Meanwhile, there exists $n_2 \in \mathbb{N}$ such that, for all $n \geq n_2$, $\liminf_{n\to +\infty} V_n (\bm{x}) - \epsilon_0/2 \leq V_n(\bm{x})$. Hence, \eqref{INF} guarantees that, for all $n \geq n_2$,
\begin{align}\label{II}
V_n
> 
\frac{\alpha \tilde{B}^2 \tilde{M}^2}{2\tilde{b} \tilde{\gamma}\tilde{\zeta}^2}
+
\frac{\beta \tilde{M}\sqrt{Dd}}{\tilde{b}\tilde{\gamma}\tilde{\zeta}}
+
\frac{4 \delta \hat{M} \sqrt{Dd}}{\tilde{\gamma}\tilde{\zeta}}
+
\frac{Dd\epsilon_0}{2 \tilde{b}\tilde{\gamma}}
+ 
\frac{\epsilon_0}{2}.
\end{align}
Therefore, \eqref{zeta_1} and \eqref{II} ensure that, for all $n \geq n_3 := \max\{ n_0, n_1, n_2\}$,
\begin{align*}
X_{n+1} 
&<
X_n + D d \alpha \epsilon_0
-2 \alpha \tilde{b} \tilde{\gamma} 
\bigg\{
\frac{\alpha \tilde{B}^2 \tilde{M}^2}{2\tilde{b} \tilde{\gamma}\tilde{\zeta}^2}
+
\frac{\beta \tilde{M}\sqrt{Dd}}{\tilde{b}\tilde{\gamma}\tilde{\zeta}}
+
\frac{4 \delta \hat{M} \sqrt{Dd}}{\tilde{\gamma}\tilde{\zeta}}
+ 
\frac{Dd\epsilon_0}{2 \tilde{b}\tilde{\gamma}}
+
\frac{\epsilon_0}{2} 
\bigg\}\\
&\quad + \frac{2 \alpha \beta \tilde{M} \sqrt{Dd}}{\tilde{\zeta}}
+ \frac{8 \alpha \tilde{b} \delta \hat{M} \sqrt{Dd}}{\tilde{\zeta}} 
+ \frac{\alpha^2 \tilde{B}^2 \tilde{M}^2}{\tilde{\zeta}^2}
\\
&=
X_n - \alpha \tilde{b} \tilde{\gamma} \epsilon_0
< 
X_{n_3} - \alpha \tilde{b} \tilde{\gamma} \epsilon_0 (n +1 - n_3),
\end{align*}
which is a contradiction since the right-hand side of the final inequality approaches minus infinity as $n$ approaches positive infinity. Therefore, \eqref{Ineq:1} holds for all $\epsilon > 0$; that is, we have
\begin{align*}
\liminf_{n\to +\infty} 
V_n (\bm{x})
\leq 
\frac{\alpha \tilde{B}^2 \tilde{M}^2}{2\tilde{b} \tilde{\gamma}\tilde{\zeta}^2}
+
\frac{\beta \tilde{M}\sqrt{Dd}}{\tilde{b}\tilde{\gamma}\tilde{\zeta}}
+
\frac{4 \delta \hat{M} \sqrt{Dd}}{\tilde{\gamma}\tilde{\zeta}}.
\end{align*}
Theorem \ref{thm:3} implies the assertions in Theorem \ref{thm:1}. This completes the proof. 
\end{proof}

\begin{proof}
[Proof of Proposition \ref{prop:1}] Since $f_i$ is convex for $i\in [T]$, we have that, for all $\bm{x} \in \mathbb{R}^d$ and all $n\in\mathbb{N}$, 
\begin{align*}
&\mathbb{E}[f(\bm{x}_n) - f^\star] \leq V_n(\bm{x}),\\
&\min_{k\in [n]} 
\mathbb{E} [f(\bm{x}_k) - f^\star ], \mathbb{E} [f(\tilde{\bm{x}}_n) - f^\star ] 
\leq \frac{1}{n} \sum_{k=1}^n V_k(\bm{x}),\\
&R(T) = \sum_{t=1}^T (f_t (\bm{x}_t) - f_t (\bm{x}^\star)) 
\leq \sum_{t=1}^T \langle \bm{x}_t - \bm{x}^\star, \mathsf{G}(\bm{x}_t,\xi_t)\rangle.
\end{align*} 
Accordingly, Theorems \ref{thm:1} and \ref{thm:3} imply Proposition \ref{prop:1}.
\end{proof}

\subsection{Proofs of Theorem \ref{thm:2} and Proposition \ref{prop:2}}
\begin{proof}
[Proof of Theorem \ref{thm:2}] Lemmas \ref{lem:1} and \ref{lem:2} imply that, for all $\bm{x} \in \mathbb{R}^d$ and all $k\in \mathbb{N}$,
\begin{align*}
2 \alpha_k V_k(\bm{x})
&\leq 
X_k - X_{k+1} + X_{k+1} \zeta^{k+1}\\
&\quad + \left(X_{k+1} - \mathbb{E}\left[\|\bm{x}_{k+1} - \bm{x}\|_{\mathsf{H}_k}^2 \right] \right)
+ \frac{\tilde{B}^2 \tilde{M}^2}{\tilde{\zeta}^2}\alpha_k^2
+ 2 \alpha_k \beta_k \mathbb{E}\left[\langle \bm{x} - \bm{x}_k, \bm{m}_{k-1} \rangle \right]\\
&\quad + 2 \alpha_k \tilde{\beta}_k \delta_k 
\mathbb{E}\left[\langle \bm{x}_k - \bm{x}, \bm{\mathsf{G}}_{k-1} \rangle \right]
- 2 \{\alpha_k \gamma_k - \alpha_k \beta_k (1+\gamma_k)\} V_k (\bm{x}), 
\end{align*}
which in turn implies that
\begin{align*}
2 \alpha_k V_k(\bm{x})
&\leq 
X_k - X_{k+1} + X^* \zeta^{k+1}\\
&\quad + D \mathbb{E}\left[ \sum_{i=1}^d (h_{k+1,i} - h_{k,i}) \right]
+ \frac{\tilde{B}^2 \tilde{M}^2}{\tilde{\zeta}^2}\alpha_k^2
+ 2 \tilde{M} \sqrt{Dd} \alpha_k \beta_k
+ 8 \hat{M} \sqrt{Dd} \alpha_k \delta_k\\
&\quad + 2 F \{\alpha_k \gamma_k + \alpha_k \beta_k (1+\gamma_k)\}, 
\end{align*}
where $X^* := \sup \{ X_n \colon n\in\mathbb{N} \}$ and $F := \sup \{ |V_n(\bm{x})| \colon n\in\mathbb{N} \}$ are finite by (A1) and (A6). Summing the above inequality from $k=0$ to $k=n$ implies that, for all $\bm{x} \in \mathbb{R}^d$,
\begin{align*}
2 \sum_{k=0}^n \alpha_k V_k(\bm{x})
&\leq 
X_0
+ X^* \sum_{k=0}^n \zeta^{k+1}
+ D \mathbb{E}\left[ \sum_{i=1}^d (h_{n+1,i} - h_{0,i}) \right]\\
&\quad + \frac{\tilde{B}^2 \tilde{M}^2}{\tilde{\zeta}^2} \sum_{k=0}^n \alpha_k^2
+ 2 \tilde{M} \sqrt{Dd} \sum_{k=0}^n \alpha_k \beta_k\\
&\quad + 8 \hat{M} \sqrt{Dd} \sum_{k=0}^n \alpha_k \delta_k
+ 4 F \sum_{k=0}^n \alpha_k \gamma_k\\
&\quad 
+ 2 F \sum_{k=0}^n \alpha_k \beta_k.
\end{align*}
From (A5), $\zeta \in [0,1)$, $\sum_{n=0}^{+\infty} \alpha_n^2 < + \infty$, $\sum_{n=0}^{+\infty} \alpha_n \beta_n < + \infty$, $\sum_{n=0}^{+\infty} \alpha_n \gamma_n < + \infty$, and $\sum_{n=0}^{+\infty} \alpha_n \delta_n < + \infty$, we have 
\begin{align}\label{sum}
\sum_{k=0}^{+\infty} \alpha_k V_k
< + \infty,
\end{align}
which, together with $\sum_{n=0}^{+\infty} \alpha_n = + \infty$, implies that 
\begin{align*}
\liminf_{n \to +\infty} V_n(\bm{x}) \leq 0.
\end{align*}
Suppose that $(\alpha_n)_{n\in\mathbb{N}}$, $(\beta_n)_{n\in\mathbb{N}}$, $(\gamma_n)_{n\in\mathbb{N}}$, and $(\delta_n)_{n\in\mathbb{N}}$ satisfy $\lim_{n \to +\infty} 1/(n \alpha_n) = 0$, \\$\lim_{n \to +\infty} (1/n) \sum_{k=0}^n \alpha_k = 0$, $\lim_{n \to +\infty} (1/n) \sum_{k=0}^n \beta_k = 0$, $\lim_{n \to +\infty} (1/n) \sum_{k=0}^n \delta_k = 0$, and $\gamma_n \geq 0$ ($n\in \mathbb{N}$). Then, Theorem \ref{thm:3} implies that, for all $\bm{x} \in \mathbb{R}^d$,
\begin{align*}
\limsup_{n \to +\infty} \frac{1}{n} \sum_{k=1}^n V_k (\bm{x}) \leq 0.
\end{align*} 
In the case of $\alpha_n := 1/n^\eta$ ($\eta \in (0,1)$), $\beta_n := \beta^n$ ($\beta \in [0,1)$), $\delta_n := \delta^n$ ($\delta \in [0,1)$), and a monotone decreasing sequence $(\gamma_n)_{n\in\mathbb{N}}$, we have $\kappa_{n+1} \leq \kappa_n$ ($n\in\mathbb{N}$) and $\limsup_{n\to + \infty} \beta_n < 1$. We also have
\begin{align*} 
&\lim_{n \to + \infty} \frac{1}{n \alpha_n} = \lim_{n \to + \infty} \frac{1}{n^{1-\eta}} 
= 0,\\
&\frac{1}{n} \sum_{k=1}^n \alpha_k
\leq
\frac{1}{n} \left( 1 + \int_1^n \frac{\mathrm{d}t}{t^\eta} \right)
\leq \frac{1}{(1-\eta)n^{\eta}}.
\end{align*}
Furthermore, $\sum_{k=1}^n \beta_k \leq \sum_{k=1}^{+\infty} \beta_k = \beta/(1-\beta)$ and $\sum_{k=1}^n \delta_k \leq \sum_{k=1}^{+\infty} \delta_k = \delta/(1-\delta)$. Therefore, Theorem \ref{thm:3} gives us the convergence rates in Theorem \ref{thm:2}. This completes the proof.
\end{proof}

\begin{proof}
[Proof of Proposition \ref{prop:2}] It is sufficient to show that any accumulation point of $(\tilde{\bm{x}}_n)_{n\in\mathbb{N}}$ belongs to $X^\star$ almost surely. Theorem 6 and the proof of Proposition 5 imply that $\lim_{n \to + \infty} \mathbb{E} [ f ( \tilde{\bm{x}}_n ) - f^\star] = 0$. Let $\hat{\bm{x}} \in \mathbb{R}^d$ be an arbitrary accumulation point of $(\tilde{\bm{x}}_n)_{n\in\mathbb{N}} \subset \mathbb{R}^d$. Then, there exists $(\tilde{\bm{x}}_{n_i})_{i\in\mathbb{N}} \subset (\tilde{\bm{x}}_n)_{n\in\mathbb{N}}$ such that $(\tilde{\bm{x}}_{n_i})_{i\in\mathbb{N}}$ converges almost surely to $\hat{\bm{x}}$. The continuity of $f$ implies that $\mathbb{E} \left[ f ( \hat{\bm{x}}) - f^\star \right]= 0$, and hence, $\hat{\bm{x}} \in X^\star$. The remaining assertions in Proposition \ref{prop:2} follow from the proof of Proposition \ref{prop:1}.
\end{proof}

\subsection{Algorithms for image classification}
\label{appen:5}
\subsubsection{ResNet-18 on CIFAR-100 dataset}
The training of ResNet-18 on the CIFAR-100 dataset is based on the results in \url{https://github.com/weiaicunzai/pytorch-cifar100}.

\textbf{Constant learning rates:} A learning rate with a cosine annealing is used \citep{sgdr}. The initial constant learning rates used in the experiments are determined using a grid search ($\alpha \in \{10^{-3}, 5 \times 10^{-3}, 10^{-2}, 5 \times 10^{-2}, 10^{-1} \}$). 

\begin{itemize}
\item SGD \citep{robb1951}: $\alpha = 5 \times 10^{-2}$ 
\item Momentum\footnote{\url{https://github.com/kuangliu/pytorch-cifar}} \citep{polyak1964,nes1983}: $\alpha = 10^{-1}$, $\beta_n = 0.9$, and the weight decay is $5 \times 10^{-4}$ 
\item RMSprop \citep{rmsprop}: $\alpha = 10^{-2}$ and the parameter $\alpha$ is $0.9$ 
\item Adagrad \citep{adagrad}: $\alpha = 10^{-2}$
\item AdamW \citep{loshchilov2018decoupled}: 
$\alpha = 10^{-3}$ and the weight decay is $10^{-2}$ 
\item Adam \citep{adam}: Algorithm \ref{algo:1} with \eqref{adam}, $\alpha = 10^{-2}$, $\zeta = \beta_n = 0.9$, $\theta = 0.999$, and $\gamma_n = \delta_n = 0$
\item AMSGrad \citep{reddi2018}: Algorithm \ref{algo:1} with \eqref{amsg}, $\alpha = 10^{-3}$, $\zeta = 0$, $\beta_n = 0.9$, $\theta = 0.999$, and $\gamma_n = \delta_n = 0$
\item SCGAdam-C: Algorithm \ref{algo:1} with \eqref{adam}, $\alpha = 10^{-3}$, $\zeta = \beta_n = 0.9$, $\theta = 0.999$, $\gamma_n = 10^{-1}$, and $\delta_n = 10^{-3}$
\item SCGAMSG-C: Algorithm \ref{algo:1} with \eqref{amsg}, $\alpha = 10^{-3}$, $\zeta = 0$, $\theta = 0.999$, $\beta_n = 0.9$, $\gamma_n = 10^{-1}$, and $\delta_n = 10^{-3}$
\end{itemize}

\textbf{Diminishing learning rates:} All optimizers used $\alpha_n = 1/\sqrt{n}$.
\begin{itemize}
\item SGD \citep{robb1951}
\item Momentum \citep{polyak1964,nes1983}: $\beta_n = 1/2^n$ 
\item RMSprop \citep{rmsprop}: the parameter $\alpha$ is $1/2^n$ 
\item Adagrad \citep{adagrad}: 
\item AdamW \citep{loshchilov2018decoupled}: the weight decay is $10^{-2}$ and $\beta_n = 1/2^n$
\item Adam \citep{adam}: Algorithm \ref{algo:1} with \eqref{adam}, $\zeta = 0.9$, $\theta = 0.999$, $\beta_n = 1/2^n$, and $\gamma_n = \delta_n = 0$ 
\item AMSGrad \citep{reddi2018}: Algorithm \ref{algo:1} with \eqref{amsg}, $\zeta = 0$, $\theta = 0.999$, $\beta_n = 1/2^n$, and $\gamma_n = \delta_n = 0$ 
\item SCGAdam-D: Algorithm \ref{algo:1} with \eqref{adam}, $\zeta = 0.9$, $\theta = 0.999$, and $\beta_n = \gamma_n = \delta_n = 1 / 2^n$
\item SCGAMSG-D: Algorithm \ref{algo:1} with \eqref{adam}, $\zeta = 0$, $\theta = 0.999$, and $\beta_n = \gamma_n = \delta_n = 1 / 2^n$
\end{itemize}

\subsubsection{ResNet-18 on CIFAR-10 dataset}
The training of ResNet-18 on the CIFAR-10 dataset is based on the results in \url{https://github.com/kuangliu/pytorch-cifar}.

\textbf{Constant learning rates:} A learning rate with a cosine annealing is used \citep{sgdr}. The initial constant learning rates used in the experiments are determined using a grid search ($\alpha \in \{10^{-3}, 5 \times 10^{-3}, 10^{-2}, 5 \times 10^{-2}, 10^{-1} \}$). 

\begin{itemize}
\item SGD \citep{robb1951}: $\alpha = 5 \times 10^{-2}$ 
\item Momentum\footnote{\url{https://github.com/kuangliu/pytorch-cifar}} \citep{polyak1964,nes1983}: $\alpha = 10^{-1}$, $\beta_n = 0.9$, and the weight decay is $5 \times 10^{-4}$
\item RMSprop \citep{rmsprop}: $\alpha = 10^{-2}$ and the parameter $\alpha$ is $0.9$ 
\item Adagrad \citep{adagrad}: $\alpha = 10^{-2}$
\item AdamW \citep{loshchilov2018decoupled}: $\alpha = 10^{-3}$ and
the weight decay is $10^{-2}$ 
\item Adam \citep{adam}: Algorithm \ref{algo:1} with \eqref{adam}, $\alpha = 10^{-2}$, $\zeta = \beta_n = 0.9$, $\theta = 0.999$, and $\gamma_n = \delta_n = 0$
\item AMSGrad \citep{reddi2018}: Algorithm \ref{algo:1} with \eqref{amsg}, $\alpha = 10^{-3}$, $\zeta = 0$, $\beta_n = 0.9$, $\theta = 0.999$, and $\gamma_n = \delta_n = 0$
\item SCGAdam-C: Algorithm \ref{algo:1} with \eqref{adam}, $\alpha = 10^{-3}$, $\zeta = \beta_n = 0.9$, $\theta = 0.999$, $\gamma_n = 10^{-1}$, and $\delta_n = 10^{-2}$
\item SCGAMSG-C: Algorithm \ref{algo:1} with \eqref{amsg}, $\alpha = 10^{-3}$, $\zeta = 0$, $\theta = 0.999$, $\beta_n = 0.9$, $\gamma_n = 10^{-1}$, and $\delta_n = 10^{-2}$
\end{itemize}

\textbf{Diminishing learning rates:} All optimizers used $\alpha_n = 1/\sqrt{n}$.
\begin{itemize}
\item SGD \citep{robb1951}
\item Momentum \citep{polyak1964,nes1983}: $\beta_n = 1/2^n$ 
\item RMSprop \citep{rmsprop}: the parameter $\alpha$ is $1/2^n$ 
\item Adagrad \citep{adagrad}: 
\item AdamW \citep{loshchilov2018decoupled}: the weight decay is $10^{-2}$ and $\beta_n = 1/2^n$
\item Adam \citep{adam}: Algorithm \ref{algo:1} with \eqref{adam}, $\zeta = 0.9$, $\theta = 0.999$, $\beta_n = 1/2^n$, and $\gamma_n = \delta_n = 0$ 
\item AMSGrad \citep{reddi2018}: Algorithm \ref{algo:1} with \eqref{amsg}, $\zeta = 0$, $\theta = 0.999$, $\beta_n = 1/2^n$, and $\gamma_n = \delta_n = 0$ 
\item SCGAdam-D: Algorithm \ref{algo:1} with \eqref{adam}, $\zeta = 0.9$, $\theta = 0.999$, and $\beta_n = \gamma_n = \delta_n = 1 / 2^n$
\item SCGAMSG-D: Algorithm \ref{algo:1} with \eqref{adam}, $\zeta = 0$, $\theta = 0.999$, and $\beta_n = \gamma_n = \delta_n = 1 / 2^n$
\end{itemize}

\subsection{Algorithms for text classification}
\label{appen:6}
\subsubsection{LSTM on IMDb dataset}
\textbf{Constant learning rates:} All optimizers used $\alpha = 10^{-3}$. The existing optimizers together with their default values are in \texttt{torch.optim}\footnote{\url{https://pytorch.org/docs/stable/optim.html}}.
\begin{itemize}
\item SGD \citep{robb1951}
\item Momentum \citep{polyak1964,nes1983}: $\beta_n = 0.9$ 
\item RMSprop \citep{rmsprop}: the parameter $\alpha$ is $0.99$ 
\item Adagrad \citep{adagrad}
\item AdamW \citep{loshchilov2018decoupled}: the weight decay is $10^{-2}$ 
\item Adam \citep{adam}: Algorithm \ref{algo:1} with \eqref{adam}, $\zeta = \beta_n = 0.9$, $\theta = 0.999$, and $\gamma_n = \delta_n = 0$
\item AMSGrad \citep{reddi2018}: Algorithm \ref{algo:1} with \eqref{amsg}, $\zeta = 0$, $\beta_n = 0.9$, $\theta = 0.999$, and $\gamma_n = \delta_n = 0$
\item SCGAdam-C: Algorithm \ref{algo:1} with \eqref{adam}, $\zeta = \beta_n = 0.9$, $\theta = 0.999$, $\gamma_n = 1$, and $\delta_n = 10^{-2}$
\item SCGAMSG-C: Algorithm \ref{algo:1} with \eqref{amsg}, $\zeta = 0$, $\theta = 0.999$, $\beta_n = 0.9$, $\gamma_n = 1$, and $\delta_n = 10^{-3}$
\end{itemize}

\textbf{Diminishing learning rates:} All optimizers used $\alpha_n = 1/\sqrt{n}$. The existing optimizers together with their default values are in \texttt{torch.optim}.

\begin{itemize}
\item SGD \citep{robb1951}
\item Momentum \citep{polyak1964,nes1983}: $\beta_n = 1/2^n$ 
\item RMSprop \citep{rmsprop}: the parameter is $1/2^n$ 
\item Adagrad \citep{adagrad}
\item AdamW \citep{loshchilov2018decoupled}: the weight decay is $10^{-2}$ and $\beta_n = 1/2^n$
\item Adam \citep{adam}: Algorithm \ref{algo:1} with \eqref{adam}, $\zeta = 0.9$, $\theta = 0.999$, $\beta_n = 1/2^n$, and $\gamma_n = \delta_n = 0$ 
\item AMSGrad \citep{reddi2018}: Algorithm \ref{algo:1} with \eqref{amsg}, $\zeta = 0$, $\theta = 0.999$, $\beta_n = 1/2^n$, and $\gamma_n = \delta_n = 0$ 
\item SCGAdam-D: Algorithm \ref{algo:1} with \eqref{adam}, $\zeta = 0.9$, $\theta = 0.999$, and $\beta_n = \gamma_n = \delta_n = 1 / 2^n$
\item SCGAMSG-D: Algorithm \ref{algo:1} with \eqref{adam}, $\zeta = 0$, $\theta = 0.999$, and $\beta_n = \gamma_n = \delta_n = 1 / 2^n$
\end{itemize}

\subsection{Algorithms for image generation}
The combinations of learning rates used in the experiments in Section \ref{subsec:4.4} were determined using a grid search. Figures \ref{fig:18}-\ref{fig:30} shows the results of the grid search.

\label{a:hyper}
\begin{table*}[htbp]
\begin{center}
\caption{Hyperparameters of the optimizer used in the training DCGAN in Section \ref{subsec:4.4}.}
\label{table:hyper}
\begin{tabular}{|ll|l|c|c|c|c|}
\hline

&optimizer
& $\alpha^D$ 
& $\alpha^G$
& $\beta_1^G = \beta_1^D$ 
& $\beta_2^G = \beta_2^D$ 
& $\beta_1$ and $\beta_2$'s reference\\ \hline \hline

&Adam
& $0.0003$
& $0.0001$
& $0.5$
& $0.999$ 
& \citep{radford2016unsupervised} \\

&AdaBelief
& $0.00003$
& $0.0003$
& $0.5$
& $0.999$ 
& \citep{adab} \\

&RMSProp
& $0.00003$
& $0.0001$
& $0$
& $0.99$ 
& \\

&SCGAdam
& $0.0001$
& $0.0003$
& $0.5$
& $0.999$ 
& \\

&SCGAMSGrad
& $0.00005$
& $0.0005$
& $0.5$
& $0.999$ 
& \\
\hline

\end{tabular}
\end{center}
\end{table*}

\begin{table*}[htbp]
\begin{center}
\caption{Hyperparameters of the optimizer used in the training SNGAN in Section \ref{subsec:4.4}.}
\label{table:hyper-sn}
\begin{tabular}{|ll|l|c|c|c|}
\hline

&optimizer
& $\alpha^D$ 
& $\alpha^G$
& $\beta_1^G = \beta_1^D$ 
& $\beta_2^G = \beta_2^D$ 
\\ \hline \hline

&Adam
& $0.0001$
& $0.0001$
& $0.5$
& $0.999$ \\

&AdaBelief
& $0.0001$
& $0.0001$
& $0.5$
& $0.999$ \\

&RMSProp
& $0.00005$
& $0.0001$
& $0$
& $0.99$ \\

&SCGAdam
& $0.0001$
& $0.0003$
& $0.5$
& $0.999$ \\

&SCGAMSGrad
& $0.0001$
& $0.0003$
& $0.5$
& $0.999$ \\
\hline

\end{tabular}
\end{center}
\end{table*}

\begin{table*}[htbp]
\begin{center}
\caption{Hyperparameters of the optimizer used in the training WGAN-GP in Section \ref{subsec:4.4}.}
\label{table:hyper-w}
\begin{tabular}{|ll|l|c|c|c|c|}
\hline

&optimizer
& $\alpha^D$ 
& $\alpha^G$
& $\beta_1^G = \beta_1^D$ 
& $\beta_2^G = \beta_2^D$ 
& $\beta_1$ and $\beta_2$'s reference\\ \hline \hline

&Adam
& $0.0003$
& $0.0001$
& $0.5$
& $0.999$ 
& \citep{NIPS2017_892c3b1c} \\

&AdaBelief
& $0.00005$
& $0.0005$
& $0.5$
& $0.999$ 
& \citep{adab} \\

&RMSProp
& $0.0003$
& $0.0005$
& $0$
& $0.99$ 
& \\

&SCGAdam
& $0.0003$
& $0.0001$
& $0.5$
& $0.999$ 
& \\

&SCGAMSGrad
& $0.0003$
& $0.0001$
& $0.5$
& $0.999$ 
& \\
\hline

\end{tabular}
\end{center}
\end{table*}

\begin{figure}[htbp]
\begin{tabular}{ccc}
\begin{minipage}[t]{0.33\linewidth}
\centering
\includegraphics[width=1\textwidth]{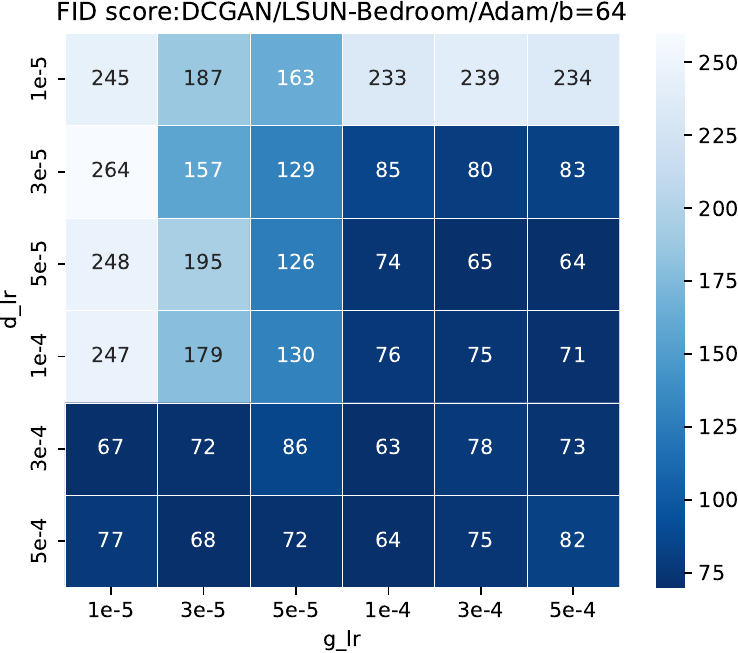}
\end{minipage} 
\begin{minipage}[t]{0.33\linewidth}
\centering
\includegraphics[width=1\textwidth]{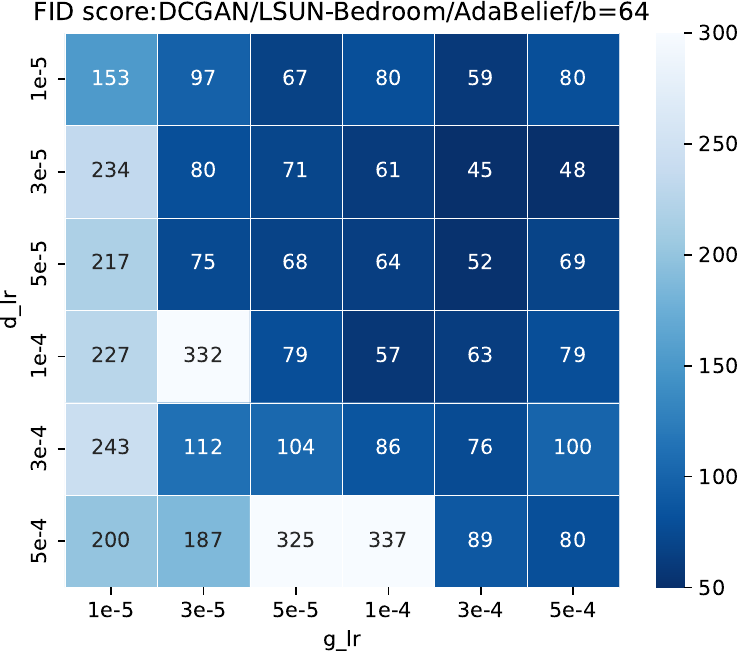}
\end{minipage} 
\begin{minipage}[t]{0.33\linewidth}
\centering
\includegraphics[width=1\textwidth]{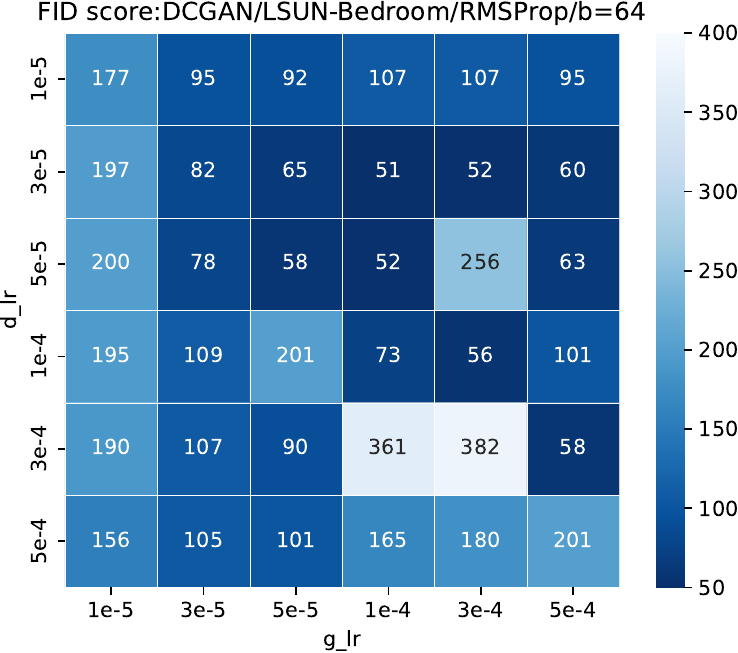}
\end{minipage} \\
\begin{minipage}[t]{0.33\linewidth}
\centering
\includegraphics[width=1\textwidth]{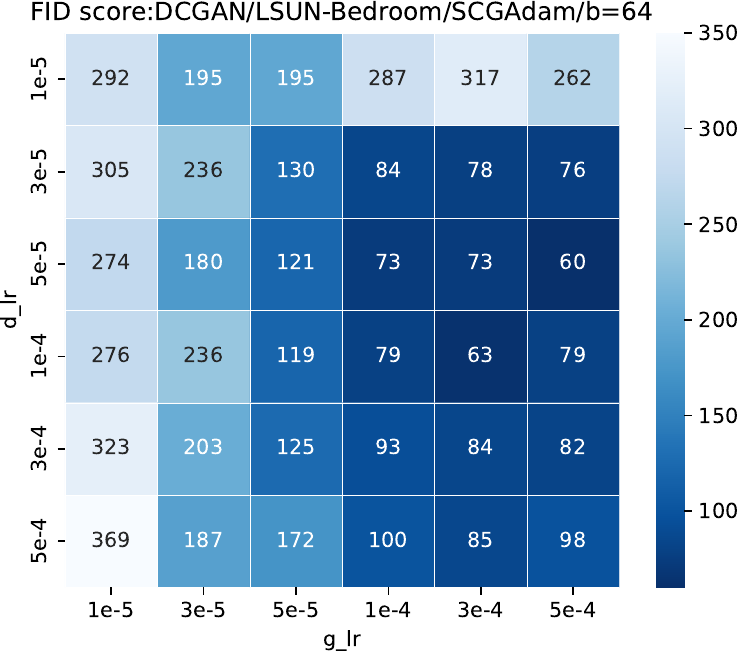}
\end{minipage} 
\begin{minipage}[t]{0.33\linewidth}
\centering
\includegraphics[width=1\textwidth]{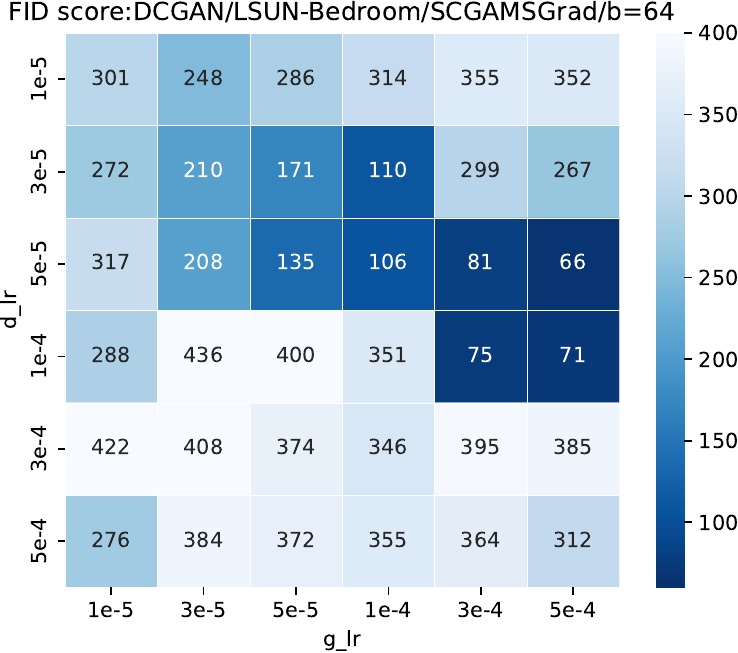}
\end{minipage} 
\end{tabular}
\caption{Analysis of the relationship between combination of learning rates and FID score in training DCGAN on the LSUN-Bedroom dataset: discriminator learning rate $\alpha^D$ on the vertical axis and generator learning rate $\alpha^G$ on the horizontal axis. The heatmap colors denote the FID scores; the darker the blue, the lower the FID, meaning that the training of the generator succeeded.}
\label{fig:18}
\end{figure}

\begin{figure}[htbp]
\begin{tabular}{ccc}
\begin{minipage}[t]{0.33\linewidth}
\centering
\includegraphics[width=1\textwidth]{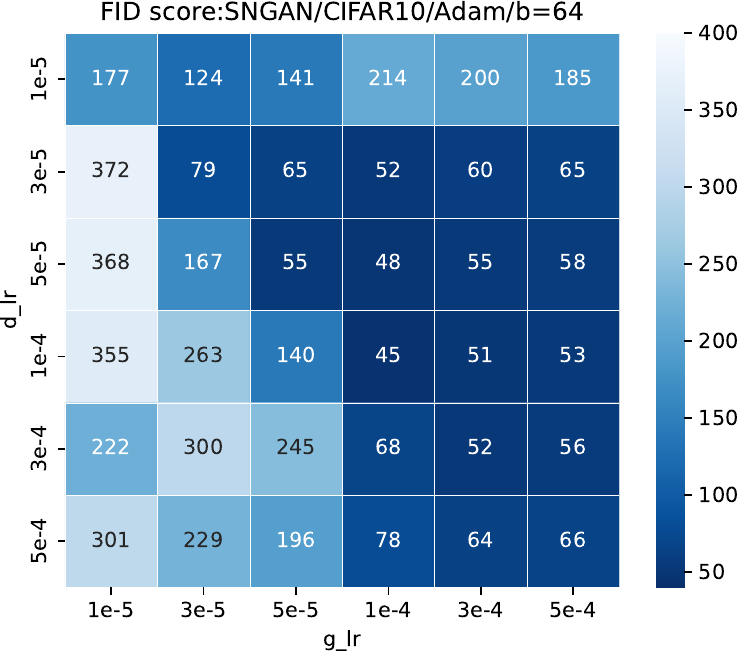}
\end{minipage} 
\begin{minipage}[t]{0.33\linewidth}
\centering
\includegraphics[width=1\textwidth]{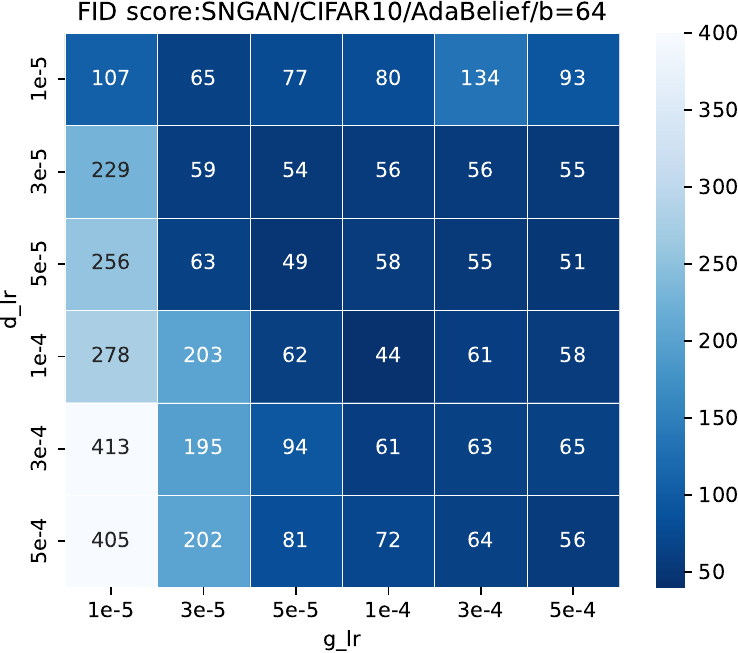}
\end{minipage} 
\begin{minipage}[t]{0.33\linewidth}
\centering
\includegraphics[width=1\textwidth]{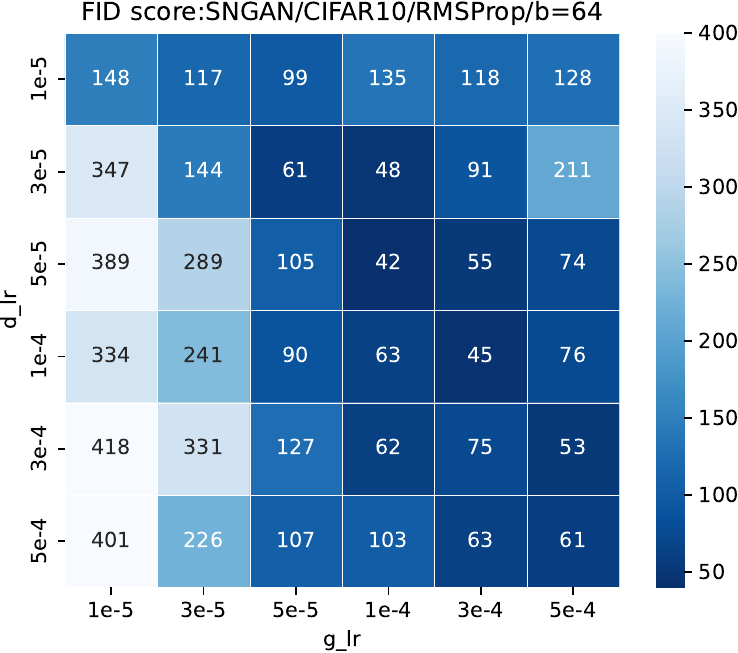}
\end{minipage} \\
\begin{minipage}[t]{0.33\linewidth}
\centering
\includegraphics[width=1\textwidth]{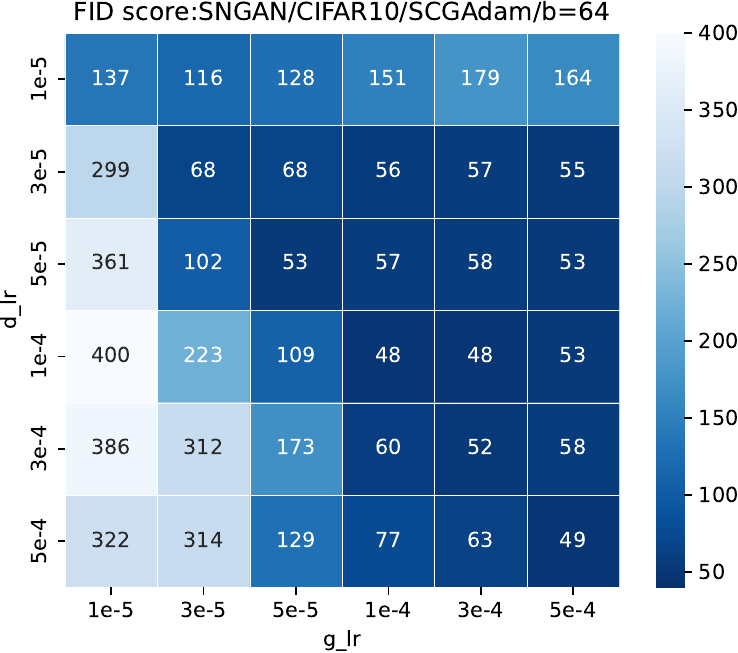}
\end{minipage} 
\begin{minipage}[t]{0.33\linewidth}
\centering
\includegraphics[width=1\textwidth]{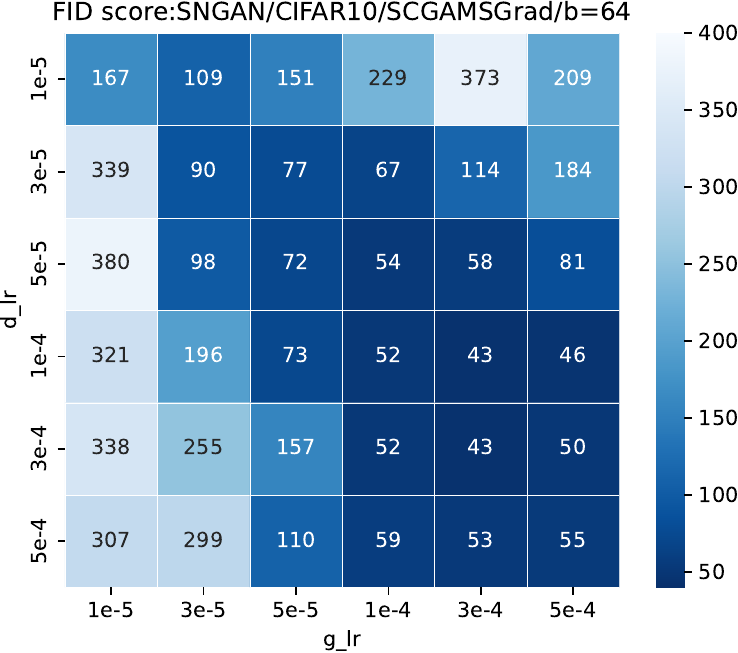}
\end{minipage} 
\end{tabular}
\caption{Analysis of the relationship between combination of learning rates and FID score in training SNGAN on the CIFAR10 dataset: discriminator learning rate $\alpha^D$ on the vertical axis and generator learning rate $\alpha^G$ on the horizontal axis. The heatmap colors denote the FID scores; the darker the blue, the lower the FID, meaning that the training of the generator succeeded.}
\label{fig:29}
\end{figure}

\begin{figure}[htbp]
\begin{tabular}{ccc}
\begin{minipage}[t]{0.33\linewidth}
\centering
\includegraphics[width=1\textwidth]{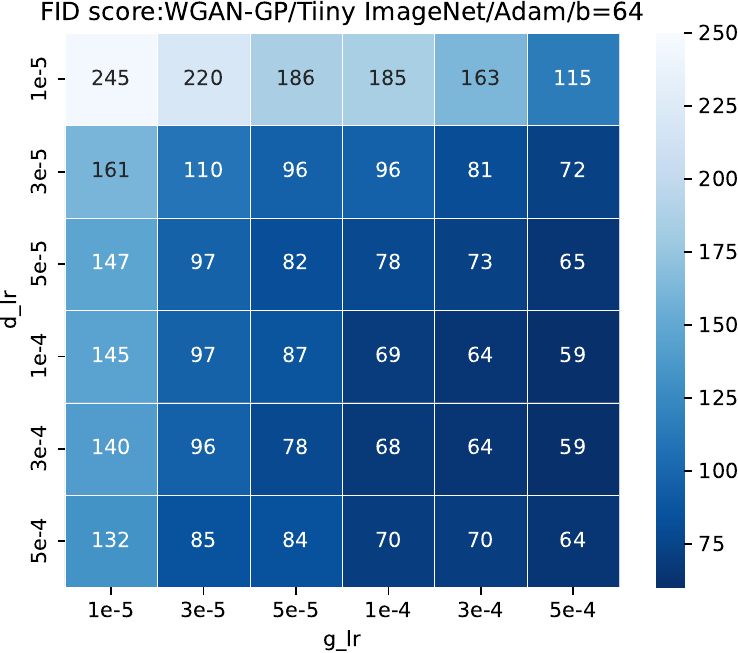}
\end{minipage} 
\begin{minipage}[t]{0.33\linewidth}
\centering
\includegraphics[width=1\textwidth]{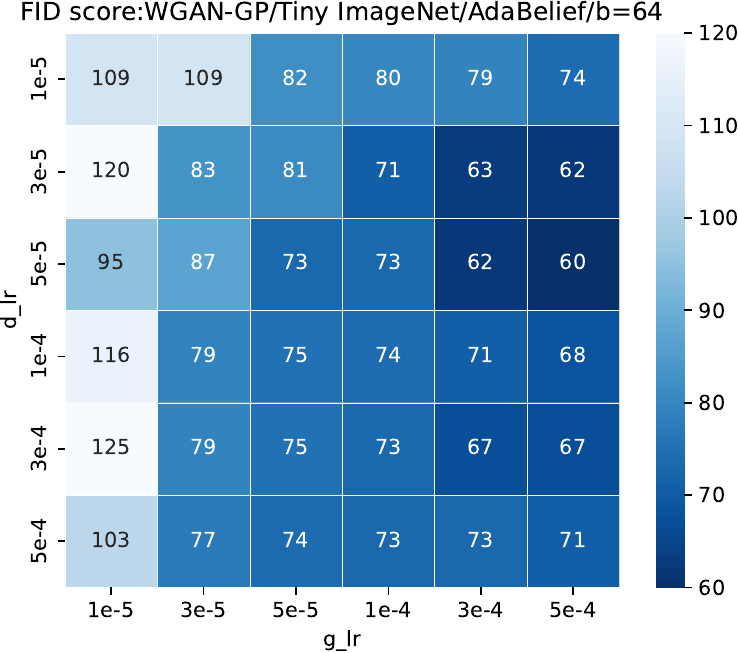}
\end{minipage} 
\begin{minipage}[t]{0.33\linewidth}
\centering
\includegraphics[width=1\textwidth]{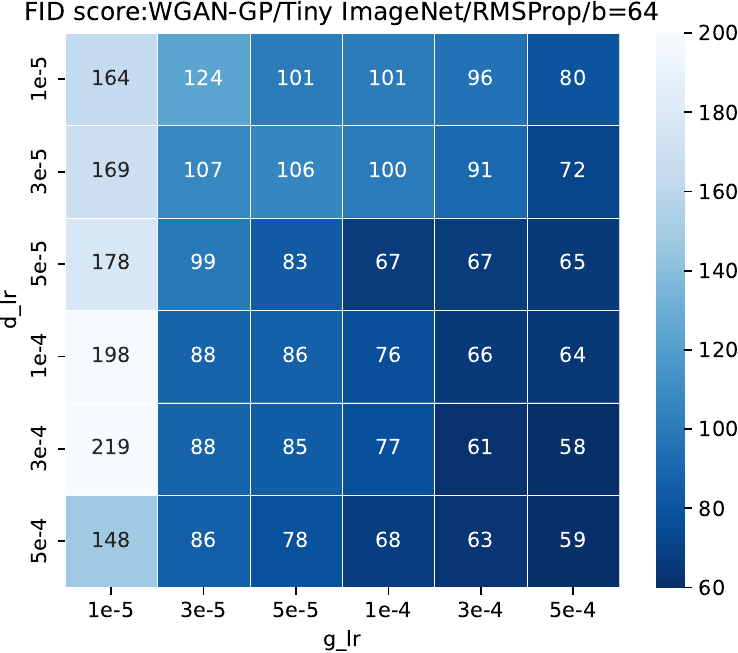}
\end{minipage} \\
\begin{minipage}[t]{0.33\linewidth}
\centering
\includegraphics[width=1\textwidth]{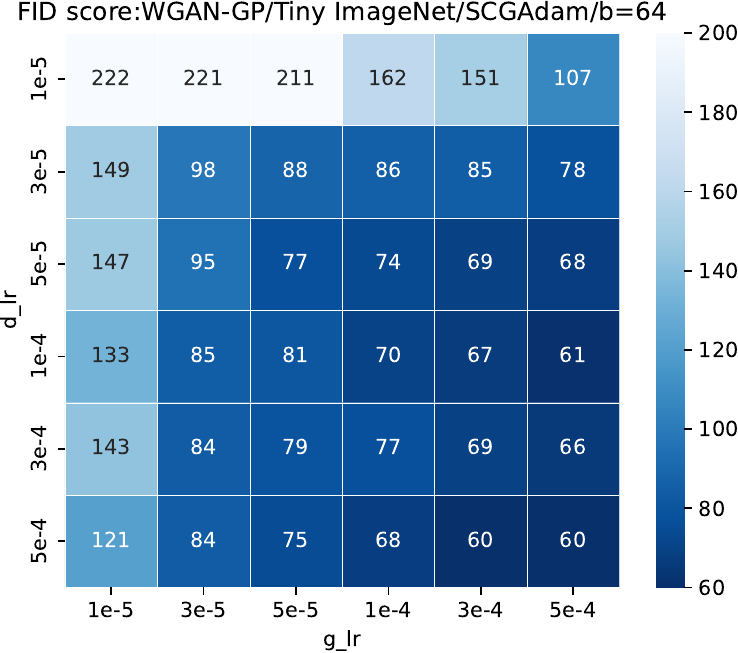}
\end{minipage} 
\begin{minipage}[t]{0.33\linewidth}
\centering
\includegraphics[width=1\textwidth]{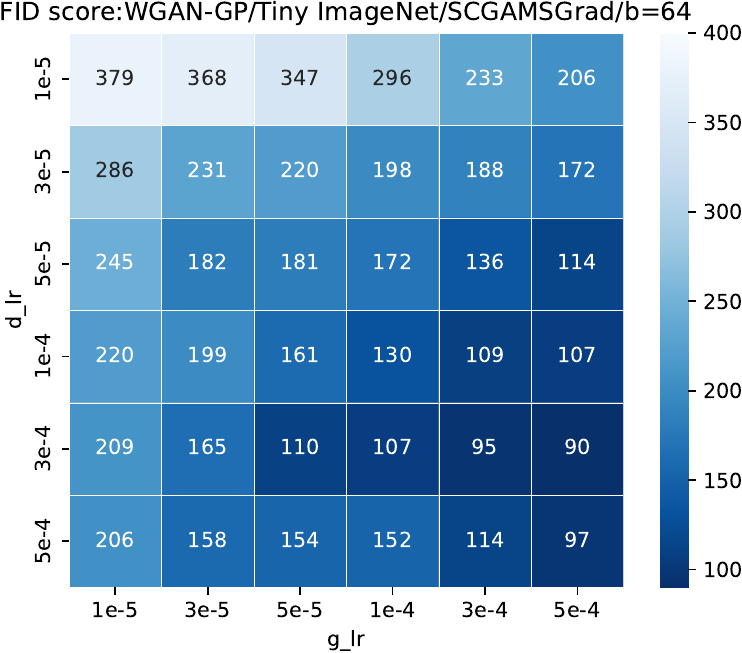}
\end{minipage} 
\end{tabular}
\caption{Analysis of the relationship between combination of learning rates and FID score in training WGAN-GP on the Tiny ImageNet dataset: discriminator learning rate $\alpha^D$ on the vertical axis and generator learning rate $\alpha^G$ on the horizontal axis. The heatmap colors denote the FID scores; the darker the blue, the lower the FID, meaning that the training of the generator succeeded.}
\label{fig:30}
\end{figure}

\vskip 0.2in
\bibliography{biblio}

\begin{thebibliography}{34}
\providecommand{\natexlab}[1]{#1}
\providecommand{\url}[1]{\texttt{#1}}
\expandafter\ifx\csname urlstyle\endcsname\relax
  \providecommand{\doi}[1]{doi: #1}\else
  \providecommand{\doi}{doi: \begingroup \urlstyle{rm}\Url}\fi

\bibitem[Andrei(2007)]{nec2007}
Neculai Andrei.
\newblock Scaled conjugate gradient algorithms for unconstrained optimization.
\newblock \emph{Computational Optimization and Applications}, 38\penalty0
  (3):\penalty0 401--416, 2007.

\bibitem[Chen et~al.(2020)Chen, Zheng, Kontar, and Raskutti]{chen2020}
Hao Chen, Lili Zheng, Raed~Al Kontar, and Garvesh Raskutti.
\newblock Stochastic gradient descent in correlated settings: {A} study on
  {G}aussian processes.
\newblock In \emph{Advances in Neural Information Processing Systems},
  volume~33, pages 1--12, 2020.

\bibitem[Chen et~al.(2019)Chen, Liu, Sun, and Hong]{chen2019}
Xiangyi Chen, Sijia Liu, Ruoyu Sun, and Mingyi Hong.
\newblock On the convergence of a class of {A}dam-type algorithms for
  non-convex optimization.
\newblock In \emph{Proceedings of The International Conference on Learning
  Representations}, pages 1--30, 2019.

\bibitem[Cheng(2007)]{doi:10.1080/01630560701749524}
Wanyou Cheng.
\newblock A two-term {PRP}-based descent method.
\newblock \emph{Numerical Functional Analysis and Optimization}, 28\penalty0
  (11--12):\penalty0 1217--1230, 2007.

\bibitem[Duchi et~al.(2011)Duchi, Hazan, and Singer]{adagrad}
John Duchi, Elad Hazan, and Yoram Singer.
\newblock Adaptive subgradient methods for online learning and stochastic
  optimization.
\newblock \emph{Journal of Machine Learning Research}, 12:\penalty0 2121--2159,
  2011.

\bibitem[Fehrman et~al.(2020)Fehrman, Gess, and Jentzen]{feh2020}
Benjamin Fehrman, Benjamin Gess, and Arnulf Jentzen.
\newblock Convergence rates for the stochastic gradient descent method for
  non-convex objective functions.
\newblock \emph{Journal of Machine Learning Research}, 21\penalty0
  (136):\penalty0 1--48, 2020.

\bibitem[Goodfellow(2017)]{ian2016nips}
Ian~J. Goodfellow.
\newblock {NIPS} 2016 tutorial: Generative adversarial networks.
\newblock \url{http://arxiv.org/abs/1701.00160}, 2017.

\bibitem[Gulrajani et~al.(2017)Gulrajani, Ahmed, Arjovsky, Dumoulin, and
  Courville]{NIPS2017_892c3b1c}
Ishaan Gulrajani, Faruk Ahmed, Martin Arjovsky, Vincent Dumoulin, and Aaron~C
  Courville.
\newblock Improved training of {W}asserstein {GAN}s.
\newblock In \emph{Advances in Neural Information Processing Systems},
  volume~30, pages 5769--5779, 2017.

\bibitem[He et~al.(2015)He, Zhang, Ren, and Sun]{he2015}
Kaiming He, Xiangyu Zhang, Shaoqing Ren, and Jian Sun.
\newblock Deep residual learning for image recognition.
\newblock In \emph{Computer Vision and Pattern Recognition}, pages 770--778,
  2015.

\bibitem[Heusel et~al.(2017)Heusel, Ramsauer, Unterthiner, Nessler, and
  Hochreiter]{Heusel2017}
M.~Heusel, H.~Ramsauer, T.~Unterthiner, B.~Nessler, and S.~Hochreiter.
\newblock {GAN}s trained by a two time-scale update rule converge to a local
  {N}ash equilibrium.
\newblock In \emph{Advances in Neural Information Processing Systems},
  volume~30, pages 6629--6640, 2017.

\bibitem[Iiduka(2022)]{iiduka2021}
Hideaki Iiduka.
\newblock Appropriate learning rates of adaptive learning rate optimization
  algorithms for training deep neural networks.
\newblock \emph{IEEE Transactions on Cybernetics}, 2022.

\bibitem[Iiduka and Kobayashi(2020)]{electronics9111809}
Hideaki Iiduka and Yu~Kobayashi.
\newblock Training deep neural networks using conjugate gradient-like methods.
\newblock \emph{Electronics}, 9\penalty0 (11), 2020.

\bibitem[Kingma and Ba(2015)]{adam}
Diederik~P Kingma and Jimmy~Lei Ba.
\newblock Adam: A method for stochastic optimization.
\newblock In \emph{Proceedings of The International Conference on Learning
  Representations}, pages 1--15, 2015.

\bibitem[Loizou et~al.(2021)Loizou, Vaswani, Laradji, and
  Lacoste-Julien]{loizou2021}
Nicolas Loizou, Sharan Vaswani, Issam Laradji, and Simon Lacoste-Julien.
\newblock Stochastic polyak step-size for {SGD}: {A}n adaptive learning rate
  for fast convergence.
\newblock In \emph{Proceedings of the 24th International Conference on
  Artificial Intelligence and Statistics (AISTATS)}, volume 130, pages 1--11,
  2021.

\bibitem[Loshchilov and Hutter(2017)]{sgdr}
Ilya Loshchilov and Frank Hutter.
\newblock {SGDR}: Stochastic gradient descent with warm restarts.
\newblock In \emph{Proceedings of The International Conference on Learning
  Representations}, 2017.

\bibitem[Loshchilov and Hutter(2019)]{loshchilov2018decoupled}
Ilya Loshchilov and Frank Hutter.
\newblock Decoupled weight decay regularization.
\newblock In \emph{International Conference on Learning Representations}, 2019.

\bibitem[Maas et~al.(2011)Maas, Daly, Pham, Huang, Ng, and Potts]{imdb}
Andrew~L. Maas, Raymond~E. Daly, Peter~T. Pham, Dan Huang, Andrew~Y. Ng, and
  Christopher Potts.
\newblock Learning word vectors for sentiment analysis.
\newblock In \emph{Proceedings of the 49th Annual Meeting of the Association
  for Computational Linguistics: Human Language Technologies}, pages 142--150.
  Association for Computational Linguistics, 2011.

\bibitem[Miyato et~al.(2018)Miyato, Kataoka, Koyama, and
  Yoshida]{Miyato2018Spe}
Takeru Miyato, Toshiki Kataoka, Masanori Koyama, and Yuichi Yoshida.
\newblock Spectral normalization for generative adversarial networks.
\newblock In \emph{Proceedings of the 6th International Conference on Learning
  Representations}, 2018.

\bibitem[M{\o}ller(1993)]{MOLLER1993525}
Martin~Fodslette M{\o}ller.
\newblock A scaled conjugate gradient algorithm for fast supervised learning.
\newblock \emph{Neural Networks}, 6\penalty0 (4):\penalty0 525--533, 1993.

\bibitem[Nakamura et~al.(2013)Nakamura, Narushima, and
  Yabe]{1547-5816_2013_3_595}
Wataru Nakamura, Yasushi Narushima, and Hiroshi Yabe.
\newblock Nonlinear conjugate gradient methods with sufficient descent
  properties for unconstrained optimization.
\newblock \emph{Journal of Industrial \& Management Optimization}, 9\penalty0
  (3):\penalty0 595--619, 2013.

\bibitem[Narushima and Yabe(2014)]{naru2014}
Yasushi Narushima and Hiroshi Yabe.
\newblock A survey of sufficient descent conjugate gradient methods for
  unconstrained optimization.
\newblock \emph{SUT Journal of Mathematics}, 50\penalty0 (2):\penalty0
  167--203, 2014.

\bibitem[Nesterov(1983)]{nes1983}
Yurii~Evgen'evich Nesterov.
\newblock A method for solving the convex programming problem with convergence
  rate ${O}(1/k^2)$.
\newblock \emph{Doklady Akademii Nauk SSSR}, 269:\penalty0 543--547, 1983.

\bibitem[Nocedal and Wright(2006)]{noce}
Jorge Nocedal and Stephen~J. Wright.
\newblock \emph{Numerical Optimization}.
\newblock Springer Series in Operations Research and Financial Engineering.
  Springer, New York, 2nd edition, 2006.

\bibitem[Polyak(1964)]{polyak1964}
Boris~T. Polyak.
\newblock Some methods of speeding up the convergence of iteration methods.
\newblock \emph{USSR Computational Mathematics and Mathematical Physics},
  4:\penalty0 1--17, 1964.

\bibitem[Radford et~al.(2016)Radford, Metz, and
  Chintala]{radford2016unsupervised}
Alec Radford, Luke Metz, and Soumith Chintala.
\newblock Unsupervised representation learning with deep convolutional
  generative adversarial networks.
\newblock In \emph{Proceedings of the International Conference on Learning
  Representations}, 2016.

\bibitem[Reddi et~al.(2018)Reddi, Kale, and Kumar]{reddi2018}
Sashank~J. Reddi, Satyen Kale, and Sanjiv Kumar.
\newblock On the convergence of {A}dam and beyond.
\newblock In \emph{Proceedings of The International Conference on Learning
  Representations}, pages 1--23, 2018.

\bibitem[Robbins and Monro(1951)]{robb1951}
Herbert Robbins and Sutton Monro.
\newblock A stochastic approximation method.
\newblock \emph{The Annals of Mathematical Statistics}, 22:\penalty0 400--407,
  1951.

\bibitem[Salimans et~al.(2016)Salimans, Goodfellow, Zaremba, Cheung, Radford,
  and Chen]{Tim2016Imp}
Tim Salimans, Ian~J. Goodfellow, Wojciech Zaremba, Vicki Cheung, Alec Radford,
  and Xi~Chen.
\newblock Improved techniques for training {GAN}s.
\newblock In \emph{Proceedings of the 29th Advances in Neural Information
  Processing Systems}, pages 2226--2234, 2016.

\bibitem[Scaman and Malherbe(2020)]{sca2020}
Kevin Scaman and C\'edric Malherbe.
\newblock Robustness analysis of non-convex stochastic gradient descent using
  biased expectations.
\newblock In \emph{Advances in Neural Information Processing Systems},
  volume~33, pages 1--11, 2020.

\bibitem[Tieleman and Hinton(2012)]{rmsprop}
Tijmen Tieleman and Geoffrey Hinton.
\newblock {RMSP}rop: {D}ivide the gradient by a running average of its recent
  magnitude.
\newblock \emph{{COURSERA}: {N}eural networks for machine learning}, 4\penalty0
  (2):\penalty0 26--31, 2012.

\bibitem[Ya~Le(2015)]{Yang2015Tin}
Xuan~Yang Ya~Le.
\newblock Tiny imagenet visual recognition challenge.
\newblock \emph{CS 231N}, 7, 2015.

\bibitem[Yu et~al.(2015)Yu, Seff, Zhang, Song, Funkhouser, and Xiao]{lsun}
Fisher Yu, Ari Seff, Yinda Zhang, Shuran Song, Thomas Funkhouser, and Jianxiong
  Xiao.
\newblock {LSUN}: Construction of a large-scale image dataset using deep
  learning with humans in the loop.
\newblock \url{https://arxiv.org/abs/1506.03365}, 2015.

\bibitem[Zhang et~al.(2006)Zhang, Zhou, and Li]{Zhang:2006vd}
Li~Zhang, Weijun Zhou, and Donghui Li.
\newblock Global convergence of a modified {F}letcher--{R}eeves conjugate
  gradient method with {A}rmijo-type line search.
\newblock \emph{Numerische Mathematik}, 104\penalty0 (4):\penalty0 561--572,
  2006.

\bibitem[Zhuang et~al.(2020)Zhuang, Tang, Ding, Tatikonda, Dvornek,
  Papademetris, and Duncan]{adab}
Juntang Zhuang, Tommy Tang, Yifan Ding, Sekhar Tatikonda, Nicha Dvornek,
  Xenophon Papademetris, and James~S. Duncan.
\newblock Ada{B}elief optimizer: {A}dapting stepsizes by the belief in observed
  gradients.
\newblock In \emph{Advances in Neural Information Processing Systems},
  volume~33, pages 1--29, 2020.

\end{thebibliography}

\end{document}